%% file: main.tex
\documentclass[acmsmall]{acmart}

\setcopyright{acmcopyright}
\acmJournal{POMACS}
\acmYear{2022} \acmVolume{6} \acmNumber{3} 
\acmArticle{53} \acmMonth{12} \acmPrice{15.00}
\acmDOI{10.1145/3570614}

\usepackage{amsmath,amsthm,bbm,mathrsfs}
\usepackage{mathtools}
\mathtoolsset{showonlyrefs}
\newcommand{\E}{\mathbb{E}} 
\renewcommand{\P}{\mathbb{P}} 
\newcommand{\R}{\mathbb{R}}

\newcommand{\N}{\mathbb{N}}

\newcommand{\ind}{\mathbbm{1}}
\renewcommand{\epsilon}{\varepsilon}
\DeclareMathOperator*{\argmax}{arg\,max}
\DeclareMathOperator*{\argmin}{arg\,min}
\DeclarePairedDelimiter{\ceil}{\lceil}{\rceil}
\DeclarePairedDelimiter{\floor}{\lfloor}{\rfloor}
\newtheorem{ass}{Assumption}
\newtheorem{clm}{Claim}
\newtheorem{cor}{Corollary}
\newtheorem{defn}{Definition}
\newtheorem{lem}{Lemma}
\newtheorem{obs}{Observation}
\newtheorem{exam}{Example}

\newtheorem{rem}{Remark}
\newtheorem{thm}{Theorem}
\usepackage[linesnumbered,ruled,noend]{algorithm2e}
\makeatletter
\renewcommand{\@algocf@capt@plain}{above}
\makeatother
\usepackage[htt]{hyphenat}
\usepackage{enumitem}
\setitemize{itemsep=1pt,topsep=1pt}
\usepackage{caption,subcaption}

\begin{document}

\title{Robust Multi-Agent Bandits Over Undirected Graphs}

\author{Daniel Vial}
\email{dvial@utexas.edu}
\affiliation{%
  \institution{University of Texas at Austin}
  \city{Austin}
  \state{Texas}
  \country{USA}
}%
\author{Sanjay Shakkottai}
\email{sanjay.shakkottai@utexas.edu}
\affiliation{%
  \institution{University of Texas at Austin}
  \city{Austin}
  \state{Texas}
  \country{USA}
}%
\author{R.\ Srikant}
\email{rsrikant@illinois.edu}
\affiliation{%
  \institution{University of Illinois Urbana-Champaign}
  \city{Urbana-Champaign}
  \state{Illinois}
  \country{USA}
}

\begin{abstract}
We consider a multi-agent multi-armed bandit setting in which $n$ honest agents collaborate over a network to minimize regret but $m$ malicious agents can disrupt learning arbitrarily. Assuming the network is the complete graph, existing algorithms incur $O( (m + K/n) \log (T) / \Delta )$ regret in this setting, where $K$ is the number of arms and $\Delta$ is the arm gap. For $m \ll K$, this improves over the single-agent baseline regret of $O(K\log(T)/\Delta)$. 

In this work, we show the situation is murkier beyond the case of a complete graph. In particular, we prove that if the state-of-the-art algorithm is used on the undirected line graph, honest agents can suffer (nearly) linear regret until time is doubly exponential in $K$ and $n$. In light of this negative result, we propose a new algorithm for which the $i$-th agent has regret $O( ( d_{\text{mal}}(i) + K/n) \log(T)/\Delta)$ on any connected and undirected graph, where $d_{\text{mal}}(i)$ is the number of $i$'s neighbors who are malicious. Thus, we generalize existing regret bounds beyond the complete graph (where $d_{\text{mal}}(i) = m$), and show the effect of malicious agents is entirely local (in the sense that only the $d_{\text{mal}}(i)$ malicious agents directly connected to $i$ affect its long-term regret).
\end{abstract}



\received{August 2022}
\received[revised]{October 2022}
\received[accepted]{November 2022}

\maketitle

\input{intro}

\input{prelim}

\input{algorithm}

\input{existing}

\input{proposed}

\input{experiments}

\input{analysis}

\section{Conclusion} \label{secConc}

In this work, we showed that existing algorithms for multi-agent bandits with malicious agents fail to generalize beyond the complete graph. In light of this, we proposed a new blocking algorithm and showed it has low regret on any connected and undirected graph. This regret bound relied on the analysis of a novel process involving gossip and blocking. Our work leaves open several questions, such as whether our insights can be applied to multi-agent reinforcement learning.

\begin{acks}
This work was supported by NSF Grants CCF 22-07547, CCF 19-34986, CNS 21-06801, 2019844, 2112471, and 2107037; ONR Grant N00014-19-1-2566; the Machine Learning Lab (MLL) at UT Austin; and the Wireless Networking and Communications Group (WNCG) Industrial Affiliates Program.
\end{acks}

\bibliographystyle{ACM-Reference-Format}
\bibliography{references}

\appendix

\section{Notes on appendices}

The appendices are organized as follows. First, Appendix \ref{appExp} contains the additional numerical results that were mentioned in Section \ref{secExp}. Next, we prove Theorem \ref{thmExisting} in Appendix \ref{appProofExisting} and all results from Section \ref{secProposed} in Appendix \ref{appProofProposed}. We then provide a rigorous version of the proof sketch from Section \ref{secAnalysis} in Appendix \ref{appProofSpread}. Finally, Appendix \ref{appProofOther} contains some auxiliary results -- namely, Appendix \ref{appBasicRes} records some simple inequalities, Appendix \ref{appBanditRes} provides some bandit results that are essentially known but stated in forms convenient to us, and Appendices \ref{appEarlyCalc}-\ref{appIntCalc} contain some tedious calculations. 

For the analysis, we use $C_i$, $C_i'$, etc.\ to denote positive constants depending only on the algorithmic parameters $\alpha$, $\beta$, $\eta$, $\rho_1$, and $\rho_2$. Each is associated with a corresponding claim, e.g., $C_{\ref{clmLogRegNaive}}$ with Claim \ref{clmLogRegNaive}. Within the proofs, we use $C$, $C'$, etc.\ to denote constants whose values may change across proofs. Finally, $\ind$ denotes the indicator function, $\E_j$ and $\P_j$ are expectation and probability conditioned on all randomness before the $j$-th communication period, and $A^{-1}(t) = \min \{ j \in \N : t \leq A_j \}$ denotes the current phase at time $t \in \N$.

\input{otherExp}

\input{proofExisting}

\input{proofProposed}

\input{proofSpread}

\input{proofOther}

\end{document}

%% file: intro.tex
\section{Introduction}

Motivated by applications including distributed computing, social recommendation systems, and federated learning, a number of recent papers have studied multi-agent variants of the classical multi-armed bandit problem. Typically, these variants involve a large number of agents playing a bandit while communicating over a network. The goal is to devise communication protocols that allow the agents to efficiently amalgamate information, thereby learning the bandit's parameters more quickly than they could by running single-agent algorithms in isolation.

Among the many multi-agent variants considered in the literature (see Section \ref{secRelated}), we focus on a particular setting studied in the recent line of work \cite{sankararaman2019social,chawla2020gossiping,newton2021asymptotic,vial2021robust}. In these papers, $n$ agents play separate instances of the same $K$-armed bandit and are restricted to $o(T)$ pairwise and bit-limited communications per $T$ arm pulls. We recount two motivating applications from this prior work.

\begin{exam} \label{examCommerce}
 For an e-commerce site (e.g., Amazon), the agents model $n$ servers choosing one of $K$ products to show visitors to the site. The product selection problem can be viewed as a bandit -- products are arms, while purchases yield rewards -- and communication among the agents/servers is restricted by bandwidth.
\end{exam}

\begin{exam} \label{examRecommend}
For a social recommendation site (e.g., Yelp), the agents represent $n$ users choosing among $K$ items, such as restaurants. This is analogously modeled as a bandit, and communication is limited because agents/users are exposed to a small portion of all reviews.
\end{exam}

To contextualize our contributions, we next discuss this line of work in more detail.

\subsection{Fully cooperative multi-agent bandits}

The goal of \cite{sankararaman2019social,chawla2020gossiping,newton2021asymptotic} is to devise fully cooperative algorithms for which the cumulative regret $R_T^{(i)}$ of each agent $i$ is small (see \eqref{eqDefnRegret} for the formal definition of regret). All of these papers follow a similar approach, which roughly proceeds as follows (see Section \ref{secAlg} for details):
\begin{itemize}
\item The arms are partitioned into $n$ subsets of size $O(K/n)$, and each agent is assigned a distinct subset called a \textit{sticky set}, which they are responsible for exploring.
\item Occasionally ($o(T)$ times per $T$ arm pulls), each agent $i$ asks a random neighbor $i'$ for an arm recommendation; $i'$ responds with the arm they believe is best, which $i$ begins playing.
\end{itemize}

For these algorithms, the regret analysis essentially contains two steps:
\begin{itemize}
\item First, the authors show that the agent (say, $i^\star$) with the true best arm in its sticky set eventually identifies it as such. Thereafter, a \textit{gossip} process unfolds. Namely, $i^\star$ recommends the best arm to its neighbors, who recommend it to their neighbors, etc., spreading the best arm to all agents. The spreading time (and thus the regret before this time) is shown to be polynomial in $K$, $n$, and $1/\Delta$, where $\Delta$ is the gap in mean reward between the two best arms.
\item Once the best arm spreads, agents play only it and their sticky sets, so long-term, they effectively face $O(K/n)$-armed bandits instead of the full $K$-armed bandit. By classical bandit results (see, e.g., \cite{auer2002finite}), this implies $O( (K/n) \log(T) / \Delta )$ regret over horizon $T$.
\end{itemize}
Hence, summing up the two terms, \cite{sankararaman2019social,chawla2020gossiping,newton2021asymptotic} provide regret bounds of the form\footnote{More precisely, \cite{chawla2020gossiping,newton2021asymptotic} prove \eqref{eqIntroHonestRegret}, while the $K/n$ term balloons to $(K/n)+\log n$ in \cite{sankararaman2019social}.}
\begin{equation} \label{eqIntroHonestRegret}
R_T^{(i)} = O \left( \frac{K}{n} \frac{ \log T }{ \Delta } + \text{poly} \left( K , n , \frac{1}{\Delta} \right) \right) ,
\end{equation}
as compared to $O(K \log(T) / \Delta )$ regret for running a single-agent algorithm in isolation.

\subsection{Robust multi-agent bandits on the complete graph}

Despite these improved bounds, \cite{sankararaman2019social,chawla2020gossiping,newton2021asymptotic} require all agents to execute the prescribed algorithm, and in particular, to recommend best arm estimates to their neighbors. As pointed out in \cite{vial2021robust}, this may be unrealistic: in Example \ref{examRecommend}, review spam can be modeled as bad arm recommendations, while in Example \ref{examCommerce}, servers may fail entirely. Hence, \cite{vial2021robust} considers a more realistic setting where $n$ \textit{honest agents} recommend best arm estimates but $m$ \textit{malicious agents} recommend arbitrarily. For this setting, the authors propose a robust version of the algorithm from \cite{chawla2020gossiping} where honest agents \textit{block} suspected malicious agents. More specifically, \cite{vial2021robust} considers the following blocking rule:
\begin{itemize}
\item If agent $i'$ recommends arm $k$ to honest agent $i$, but arm $k$ subsequently performs poorly for $i$ -- in the sense that the upper confidence bound (UCB) algorithm does not select it sufficiently often -- then $i$ temporarily suspends communication with $i'$. 
\end{itemize}

As shown in \cite{vial2021robust}, this blocking scheme prevents each malicious agent from recommending more than $O(1)$ bad arms long-term, which (effectively) results in an $O(m + K/n )$-armed bandit ($O(m)$ malicious recommendations, plus the $O(K/n)$-sized sticky set). Under the assumption that honest and malicious agents are connected by the \textit{complete graph}, this allows \cite{vial2021robust} to prove
\begin{equation} \label{eqIntroMaliciousRegret}
R_T^{(i)} = O \left( \left( \frac{K}{n} + m \right) \frac{ \log T }{ \Delta } + \text{poly} \left( K , n , m , \frac{1}{\Delta} \right) \right) .
\end{equation}
In \cite{vial2021robust}, it is also shown that blocking is necessary: for any $n \in \N$, if even $m=1$ malicious agent is present, the algorithm from \cite{chawla2020gossiping} (which does not use blocking) incurs $\Omega(K \log(T) / \Delta )$ regret. Thus, one malicious agent negates the improvement over the single-agent baseline.

\subsection{Objective and challenges} \label{secChallenges}

Our main goal is to generalize the results of \cite{vial2021robust} from the complete graph to the case where the honest agent subgraph is only connected and undirected. This is nontrivial because \cite{vial2021robust} relies \textit{heavily} on the complete graph assumption. In particular, the analysis in \cite{vial2021robust} requires that $i^\star$ (the agent with the best arm in its sticky set) \textit{itself} recommends the best arm to each of the other honest agents. In other words, each honest agent $i \neq i^\star$ relies on $i^\star$ to inform them of the best arm, which means $i^\star$ must be a neighbor of $i$. Thus, to extend \eqref{eqIntroMaliciousRegret} beyond complete graphs, we need to show a gossip process unfolds (like in the fully cooperative case): $i^\star$ recommends the best arm to its neighbors, who recommend it to their neighbors, etc., spreading it through the network.

The challenge is that, while blocking is necessary to prevent $\Omega(K \log(T) / \Delta )$ regret, it also causes honest agents to \textit{accidentally block each other}. Indeed, due to the aforementioned blocking rule and the noisy rewards, they will block each other until they collect enough samples to reliably identify good arms. From a network perspective, accidental blocking means that edges in the subgraph of honest agents temporarily fail. Consequently, it is not clear if the best arm spreads to all honest agents, or if (for example) this subgraph eventually becomes disconnected, preventing the spread and causing the agents who do not receive the best arm to suffer $\Theta(T)$ regret.

Analytically, accidental blocking means we must deal with a gossip process over a \textit{dynamic graph}. This process is extremely complicated, because the graph dynamics are driven by the bandit algorithms, which in turn affect the future evolution of the graph. Put differently, blocking causes the randomness of the communication protocol and that of the bandit algorithms to become \textit{interdependent}. We note this does not occur for the original non-blocking algorithm, where the two sources of randomness can be cleverly decoupled and separately analyzed -- see \cite[Proposition 4]{chawla2020gossiping}. Thus, in contrast to existing work, we need to analyze the interdependent processes directly.

\subsection{Our contributions} \label{secContributions}

{\bf Failure of the existing blocking rule:} In Section \ref{secExisting}, we show that the algorithm from \cite{vial2021robust} fails to achieve a regret bound of the form \eqref{eqIntroMaliciousRegret} for connected and undirected graphs in general. Toward this end, we define a natural ``bad instance'' in which $n=K$, the honest agent subgraph is an undirected line (thus connected), and all honest agents share a malicious neighbor. For this instance, we propose a malicious strategy that causes honest agents to repeatedly block one another, which results in the best arm spreading extremely slowly. More specifically, we show that if honest agents run the algorithm from \cite{vial2021robust}, then the best arm does not reach honest agent $n$ (the one at the end of the line) until time is \textit{doubly exponential} in $n=K$. Note \cite{vial2021robust} shows the best arm spreads polynomially fast for the complete graph, so we demonstrate a doubly exponential slowdown for complete versus line graphs. This is rather surprising, because for classical rumor processes that do not involve bandits or blocking (see, e.g., \cite{pittel1987spreading}), the slowdown is only exponential (i.e., $\Theta(\log n)$ rumor spreading time on the complete graph versus $\Theta(n)$ on the line graph). As a consequence of the slow spread, we show the algorithm from \cite{vial2021robust} suffers regret
\begin{equation} \label{eqIntroFailure}
R_T^{(n)} = \Omega \left( \min \left\{ \log(T) + \exp \left( \exp \left( \frac{n}{3} \right) \right) , \frac{T}{\log^7 T} \right\} \right) ,
\end{equation}
i.e., it incurs (nearly) linear regret until time $\Omega ( \exp ( \exp ( n / 3 ) ) )$ and thereafter incurs logarithmic regret but with a huge additive term (see Theorem \ref{thmExisting}).

\vspace{3pt} \noindent {\bf Refined blocking rule:} In light of this negative result, we propose a refined blocking rule in Section \ref{secProposed}. Roughly, our rule is as follows: agent $i$ blocks $i'$ for recommending arm $k$ if
\begin{itemize}
\item arm $k$ performs poorly, i.e., it is not chosen sufficiently often by UCB,
\item \textit{and} agent $i$ has not changed its own best arm estimate recently.
\end{itemize}
The second criterion is the key distinction from \cite{vial2021robust}. Intuitively, it says that agents should \textit{not} block for seemingly-poor recommendations until they become confident that their \textit{own} best arm estimates have settled on truly good arms. This idea is the main new algorithmic insight of the paper. It is directly motivated by the negative result of Section \ref{secExisting}; see Remark \ref{remMotivate}.

\vspace{3pt} \noindent {\bf Gossip despite blocking:} Analytically, our main contribution is to show that, with our refined blocking rule, the best arm quickly spreads to all honest agents. The proof is quite involved; we provide an outline in Section \ref{secAnalysis}. At a very high level, the idea is to show that honest agents using our blocking rule eventually stop blocking each other. Thereafter, we can couple the arm spreading process with a much more tractable noisy rumor process that involves neither bandits nor blocking (see Definition \ref{defnRumor}), and that is guaranteed to spread the best arm in polynomial time.

\vspace{3pt} \noindent {\bf Regret upper bound:} Combining our novel gossip analysis with some existing regret minimization techniques, we show in Section \ref{secProposed} that our refined algorithm enjoys the regret bound
\begin{equation} \label{eqIntroOurs}
R_T^{(i)} = O \left( \left( \frac{K}{n} + d_{\text{mal}}(i) \right) \frac{ \log T }{ \Delta } + \text{poly} \left( K , n , m , \frac{1}{\Delta} \right) \right) ,
\end{equation}
where $d_{\text{mal}}(i)$ is the number of malicious neighbors of $i$ (see Theorem \ref{thmProposed}). Thus, our result generalizes \eqref{eqIntroMaliciousRegret} from the complete graph (where $d_{\text{mal}}(i) = m$) to connected and undirected graphs. Moreover, note the leading $\log T$ term in \eqref{eqIntroOurs} is \textit{entirely local} -- only the $d_{\text{mal}}(i)$ malicious agents directly connected to $i$ affect its long-term regret. For example, in the sparse regime $d_{\text{mal}}(i) = O(1)$, our $\log T$ term matches the one in \eqref{eqIntroHonestRegret} up to constants, which (we recall) \cite{chawla2020gossiping,newton2021asymptotic} proved in the case where there are no malicious agents anywhere in the network. In fact, for honest agents $i$ with $d_{\text{mal}}(i) = 0$, we can prove that the $\log T$ term in our regret bound matches the corresponding term from \cite{chawla2020gossiping}, \textit{including} constants (see Corollary \ref{corNoMal}). In other words, we show that for large horizons $T$, the effects of malicious agents do not propagate beyond one-step neighbors. Furthermore, we note that the additive term in \eqref{eqIntroOurs} is polynomial in all parameters, whereas for the existing algorithm it can be doubly exponential in $K$ and $n$, as shown in \eqref{eqIntroFailure} and discussed above.

\vspace{3pt} \noindent {\bf Numerical results:} In Section \ref{secExp}, we replicate the experiments from \cite{vial2021robust} and extend them from the complete graph to $G(n+m,p)$ random graphs. Among other findings, we show that for $p=1/2$ and $p=1/4$, respectively, the algorithm from \cite{vial2021robust} can perform \textit{worse} than the non-blocking algorithm from \cite{chawla2020gossiping} and the single-agent baseline, respectively. In other words, the existing blocking rule becomes a \textit{liability} as $p$ decreases from the extreme case $p=1$ considered in \cite{vial2021robust}. In contrast, we show that our refined rule has lower regret than \cite{chawla2020gossiping} across the range of $p$ tested. Additionally, it outperforms \cite{vial2021robust} on average for all but the largest $p$ and has much lower variance for smaller $p$.

\vspace{3pt} \noindent {\bf Summary:} Ultimately, the high-level messages of this paper are twofold:
\begin{itemize}
\item In multi-agent bandits with malicious agents, we can devise algorithms that simultaneously (1) learn useful information and spread it through the network via gossip, and (2) learn who is malicious and block them to mitigate the harm they cause. Moreover, this harm is local in the sense that it only affects one-hop neighbors.

\item However, blocking must be done carefully; algorithms designed for the complete graph may spread information extremely slowly on general graphs. In particular, the slowdown can be \textit{doubly} exponential, much worse than the exponential slowdown of simple rumor processes.
\end{itemize}

\subsection{Other related work} \label{secRelated}

In addition to the paper \cite{vial2021robust} discussed above, several others have considered multi-agent bandits where some of the agents are uncooperative. In \cite{awerbuch2008competitive}, the honest agents face a non-stochastic (i.e., adversarial) bandit \cite{auer1995gambling} and communicate at every time step, in contrast to the stochastic bandit and limited communication of our work. The authors of \cite{mitra2021exploiting} consider the objective of best arm identification \cite{audibert2010best} instead of cumulative regret. Most of their paper involves a different communication model where the agents/clients collaborate via a central server; Section 6 studies a ``peer-to-peer'' model which is closer to ours but requires additional assumptions on the number of malicious neighbors. A different line of work considers the case where an adversary can corrupt the observed rewards (see, e.g., \cite{bogunovic2020corruption,bogunovic2021stochastic,garcelon2020adversarial,gupta2019better,jun2018adversarial,kapoor2019corruption,liu2019data,liu2021cooperative,lykouris2018stochastic}, and the references therein), which is distinct from the role that malicious agents play in our setting.

For the fully cooperative case, there are several papers with communication models that differ from the aforementioned \cite{chawla2020gossiping,newton2021asymptotic,sankararaman2019social}. For example, agents in \cite{buccapatnam2015information,chakraborty2017coordinated} broadcast information instead of exchanging pairwise arm recommendations, communication in \cite{kolla2018collaborative,lalitha2021bayesian,martinez2019decentralized} is more frequent, the number of transmissions in \cite{madhushani2021call} depends on $\Delta^{-1}$ so could be large, and agents in \cite{landgren2016distributed} exchange arm mean estimates instead of (bit-limited) arm indices.

More broadly, other papers have studied fully cooperative variants of different bandit problems. These include minimizing simple instead of cumulative regret (e.g., \cite{hillel2013distributed,szorenyi2013gossip}), minimizing the total regret across agents rather than ensuring all have low regret (e.g., \cite{dubey2020cooperative,wang2020optimal}), contextual instead of multi-armed bandits (e.g., \cite{chawla2020multi,dubey2020kernel,dubey2020differentially,korda2016distributed,tekin2015distributed}), adversarial rather than stochastic bandits (e.g., \cite{bar2019individual,cesa2016delay,kanade2012distributed}), and bandits that vary across agents (e.g., \cite{bistritz2018distributed,shahrampour2017multi,zhu2021federated}). Another long line of work features collision models where rewards are lower if multiple agents simultaneously pull the same arm (e.g., \cite{anandkumar2011distributed,avner2014concurrent,boursier2019sic,dakdouk2021collaborative,kalathil2014decentralized,liu2010distributed,liu2020competing,mansour2018competing,rosenski2016multi}), unlike our model. Along these lines, other reward structures have been studied, such as reward being a function of the agents' joint action (e.g., \cite{bargiacchi2018learning,bistritz2020cooperative,kao2021decentralized}).

\subsection{Organization}

The rest of the paper is structured as follows. We begin in Section \ref{secPrelim} with definitions. In Section \ref{secAlg}, we introduce the algorithm from \cite{vial2021robust}. Sections \ref{secExisting} and \ref{secProposed} discuss the existing and proposed blocking rules. Section \ref{secExp} contains experiments. We discuss our analysis in Section \ref{secAnalysis} and close in Section \ref{secConc}.

%% file: prelim.tex
\section{Preliminaries} \label{secPrelim}

\noindent {\bf Communication network:} Let $G = ( [n+m] , E )$ be an undirected graph with vertices $[ n+m ] = \{ 1 , \ldots , n+m \}$. We call $[n]$ the \textit{honest agents} and assume they execute the forthcoming algorithm. The remaining agents are termed \textit{malicious}; their behavior will be specified shortly. For instance, honest and malicious agents represent functioning and failed servers in Example \ref{examCommerce}. The edge set $E$ encodes which agents are allowed to communicate, e.g., if $(i,i') \in E$, the $i$-th and $i'$-th servers can communicate in the forthcoming algorithm.

Denote by $E_{\text{hon}} = \{ (i,i') \in E : i, i' \in [n] \}$ the edges between honest agents and $G_{\text{hon}} = ( [n] , E_{\text{hon}} )$ the honest agent subgraph. For each $i \in [n]$, we let $N(i) = \{ i' \in [n+m] : (i,i') \in E \}$ denote its neighbors, $N_{\text{hon}}(i) = N(i) \cap [n]$ its honest neighbors, and $N_{\text{mal}}(i) = N(i) \setminus [n]$ its malicious neighbors. We write $d(i) = |N(i)|$, $d_{\text{hon}}(i) = |N_{\text{hon}}(i)|$, and $d_{\text{mal}}(i) = |N_{\text{mal}}(i)|$ for the associated degrees, and $\bar{d} = \max_{i \in [n]} d(i)$, $\bar{d}_{\text{hon}} = \max_{i \in [n]} d_{\text{hon}}(i)$, and $\bar{d}_{\text{mal}} = \max_{i \in [n]} d_{\text{mal}}(i)$ for the maximal degrees. We make the following assumption, which generalizes the complete graph case of \cite{vial2021robust}.
\begin{ass} \label{assGraph}
The honest agent subgraph $G_{\text{hon}}$ is connected, i.e., for any $i,i' \in [n]$, there exists $l \in \N$ and $i_0 , i_1 , \ldots , i_l \in [n]$ such that $i_0 = i$, $(i_{j-1},i_j) \in E_{\text{hon}}\ \forall\ j \in [l]$, and $i_l = i'$.
\end{ass}

\noindent {\bf Multi-armed bandit:}  We consider the standard stochastic multi-armed bandit \cite[Chapter 4]{lattimore2020bandit}. Denote by $K \in \N$ the number of arms and $[K]$ the set of arms. For each $k \in [K]$, we let $\nu_k$ be a probability distribution over $\R$ and $\{ X_{i,t}(k) \}_{i \in [n] , t \in \N}$ an i.i.d.\ sequence of \textit{rewards} sampled from $\nu_k$. The interpretation is that, if the $i$-th honest agent chooses the $k$-th arm at time $t$, it earns reward $X_{i,t}(k)$. The objective (to be formalized shortly) is reward maximization. In Example \ref{examRecommend}, for instance, $[K]$ represents the set of restaurants in a city, and the reward $X_{i,t}(k)$ quantifies how much person $i$ enjoys restaurant $k$ if they dine there on day $t$.

For each arm $k \in [K]$, we let $\mu_k = \E [ X_{i,t}(k) ]$ denote the corresponding expected reward. Without loss of generality, we assume the arms are labeled such that $\mu_1 \geq \cdots \geq \mu_K$. We additionally assume the following, which generalizes the $\nu_k = \text{Bernoulli}(\mu_k)$ and $\mu_1 > \mu_2$ setting of \cite{vial2021robust}. Notice that under this assumption, the \textit{arm gap} $\Delta_k \triangleq \mu_1 - \mu_k$ is strictly positive.
\begin{ass} \label{assReward}
Rewards are $[0,1]$-valued, i.e., for each $k \in [K]$, $\nu_k$ is a distribution over $[0,1]$. Furthermore, the best arm is unique, i.e., $\mu_1 > \mu_2$.
\end{ass}

\noindent {\bf Objective:} For each $i \in [n]$ and $t \in \N$, let $I_t^{(i)} \in [K]$ denote the arm chosen by honest agent $i$ at time $t$. Our goal is to minimize the \textit{regret} $R_T^{(i)}$, which is the expected additive loss in cumulative reward for agent $i$'s sequence of arm pulls $\{ I_t^{(i)} \}_{t=1}^T$ compared to the optimal policy that always chooses the best arm $1$. More precisely, we define regret as follows: 
\begin{equation} \label{eqDefnRegret}
R_T^{(i)} \triangleq \sum_{t=1}^T \E \left[ X_{i,t}(1) - X_{i,t}(I_t^{(i)}) \right] = \sum_{t=1}^T \E \left[ \mu_1 - \mu_{I_t^{(i)}} \right] = \sum_{t=1}^T \E \left[ \Delta_{I_t^{(i)}} \right] .
\end{equation}

%% file: algorithm.tex
\section{Algorithm} \label{secAlg}

We next discuss the algorithm from \cite{vial2021robust} (Algorithm \ref{algGeneral} below), which modifies the one from \cite{chawla2020gossiping} to include blocking. For ease of exposition, we begin by discussing the key algorithmic design principles from \cite{chawla2020gossiping} in Section \ref{secAlgExample}. We then define Algorithm \ref{algGeneral} formally in Section \ref{secAlgFormal}. Finally, we introduce and discuss one additional assumption in Section \ref{secAlgAssumption}.

\subsection{Key ideas of the non-blocking algorithm} \label{secAlgExample}

\begin{figure}[t]
\includegraphics[width=\textwidth]{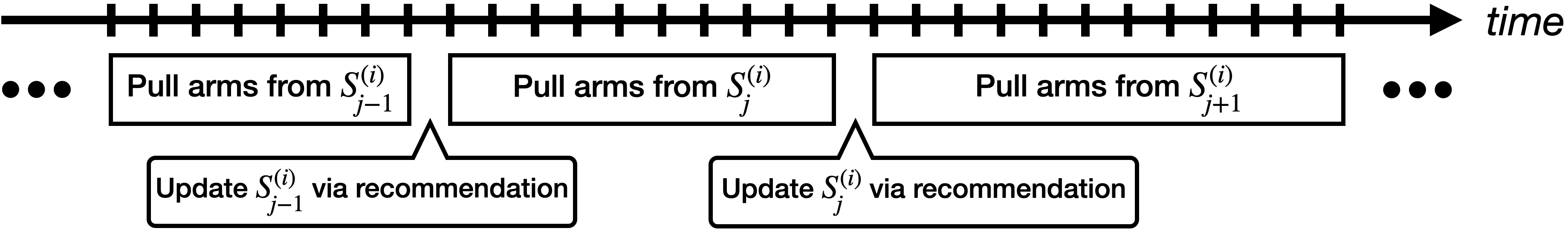}
\caption{Illustration of the phases in Algorithm \ref{algGeneral}; see beginning of Section \ref{secAlgExample} for details.} \label{figAlgorithm}
\Description{Illustration of the phases in Algorithm \ref{algGeneral}; see beginning of Section \ref{secAlgExample} for details.}
\end{figure}

\begin{figure}[t]
     \centering 
     \begin{subfigure}[b]{0.3\textwidth}
         \centering
         \includegraphics[width=0.9\textwidth]{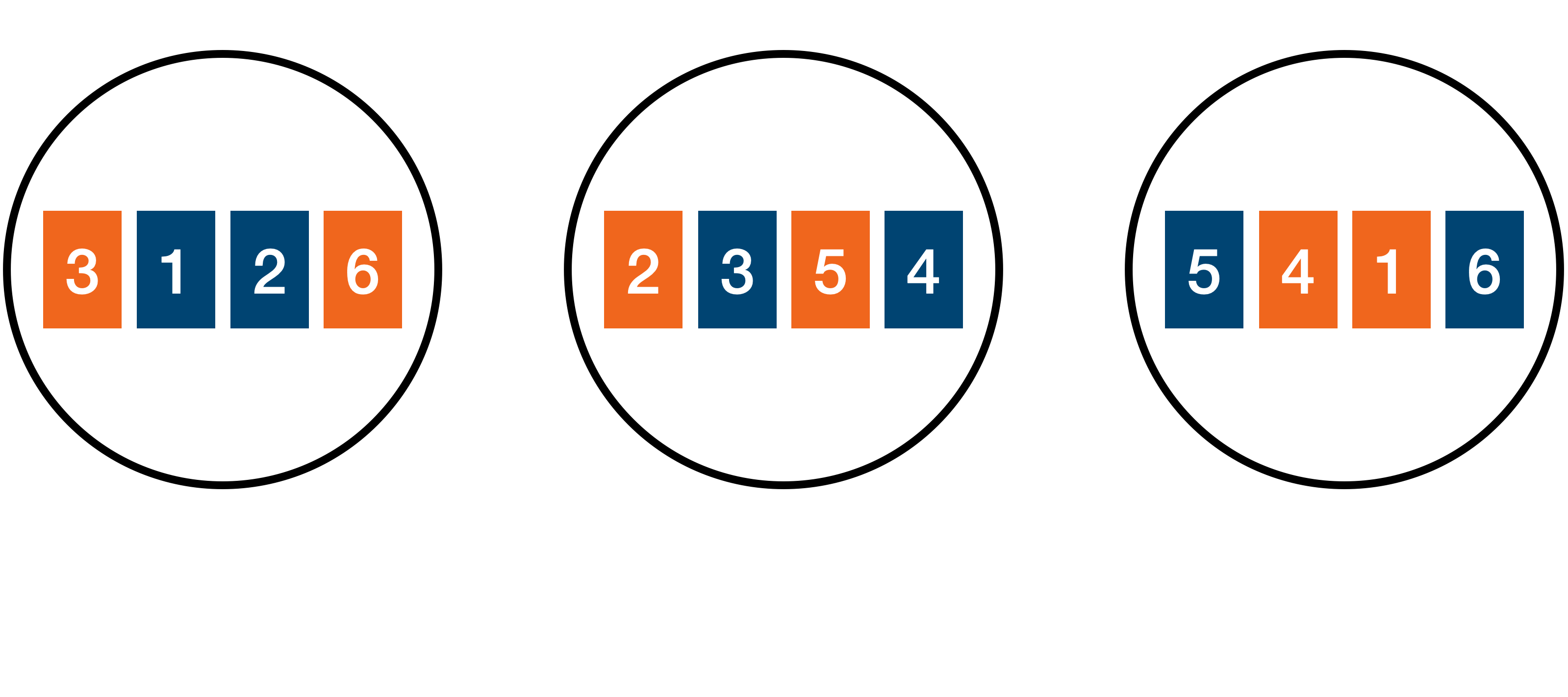}
         \caption{Initial active sets}
         \label{figExample1}
     \end{subfigure}%
	\hspace{0.04\textwidth}%
     \begin{subfigure}[b]{0.3\textwidth}
         \centering
         \includegraphics[width=0.9\textwidth]{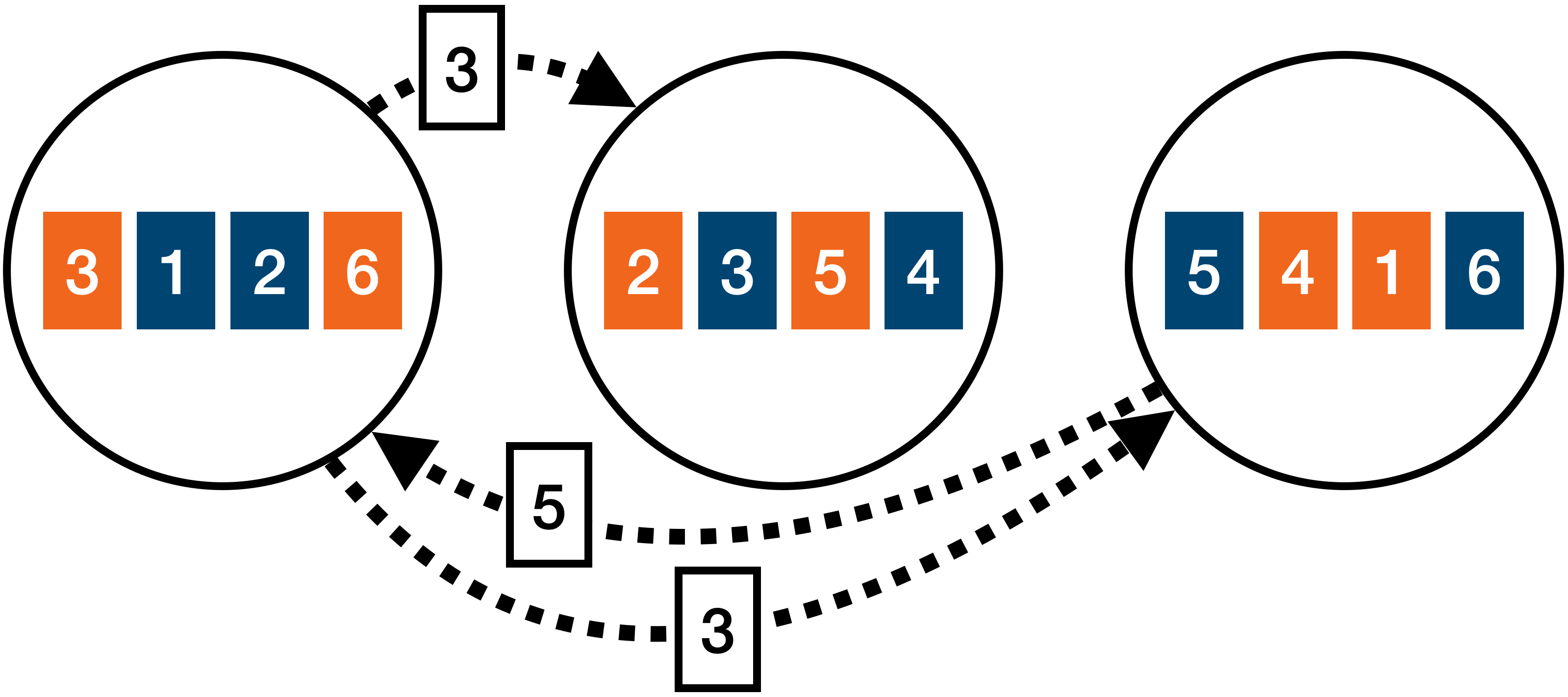}
         \caption{Recommendations}
         \label{figExample2}
     \end{subfigure}%
	\hspace{0.04\textwidth}%
     \begin{subfigure}[b]{0.3\textwidth}
         \centering
         \includegraphics[width=\textwidth]{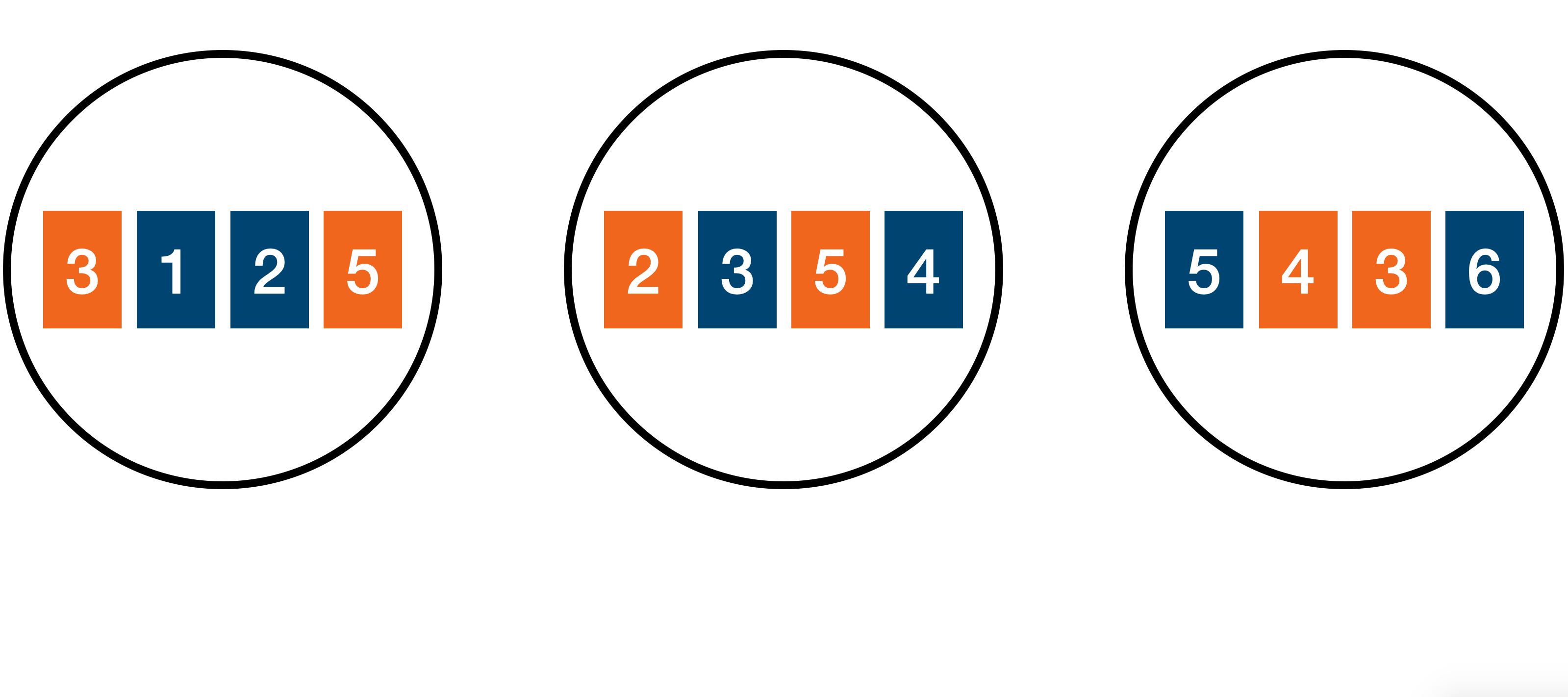}
         \caption{Updated active sets}
         \label{figExample3}
     \end{subfigure}
     \begin{subfigure}[b]{0.3\textwidth}
         \centering
         \includegraphics[width=\textwidth]{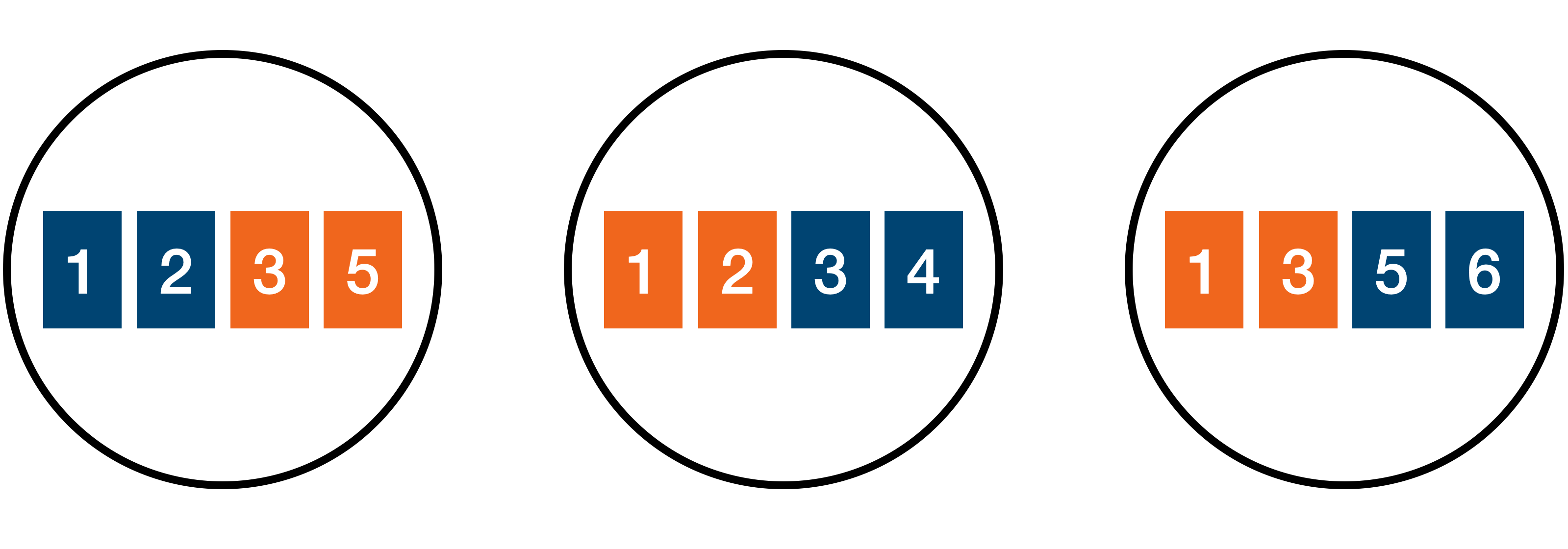}
         \caption{Later active sets}
         \label{figExample4}
     \end{subfigure}%
	\hspace{0.04\textwidth}%
     \begin{subfigure}[b]{0.3\textwidth}
         \centering
         \includegraphics[width=\textwidth]{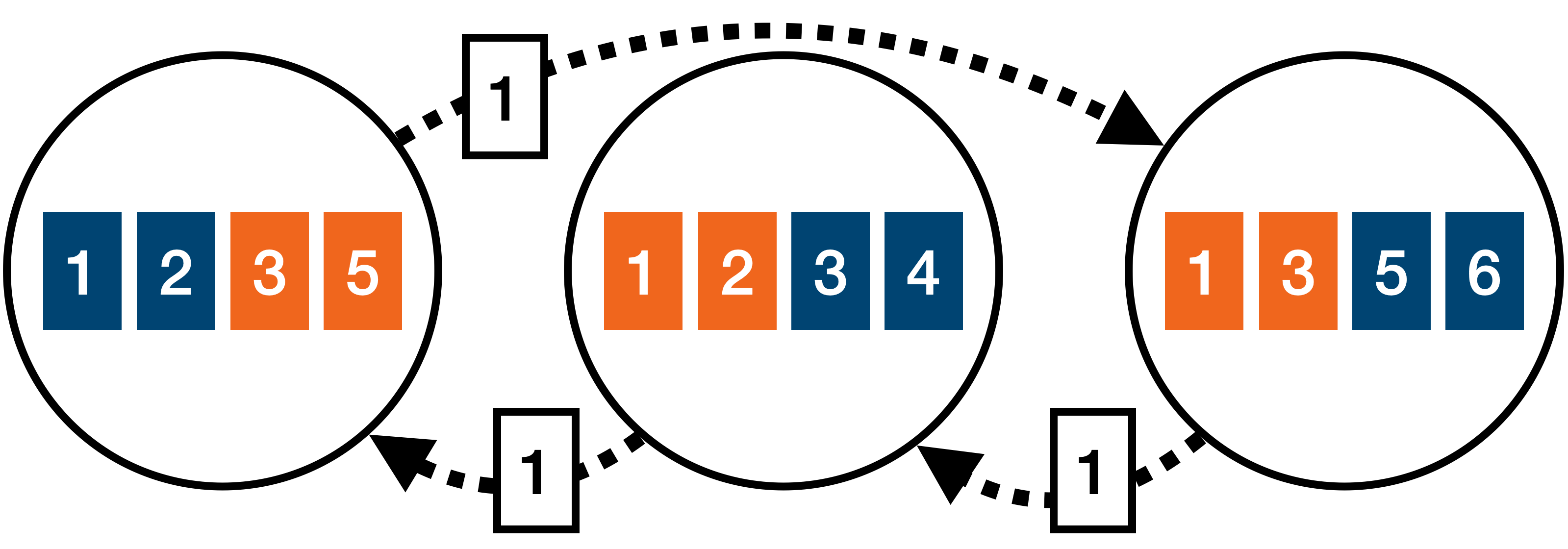}
         \caption{Recommendations}
         \label{figExample5}
     \end{subfigure}%
	\hspace{0.04\textwidth}%
     \begin{subfigure}[b]{0.3\textwidth}
         \centering
         \includegraphics[width=\textwidth]{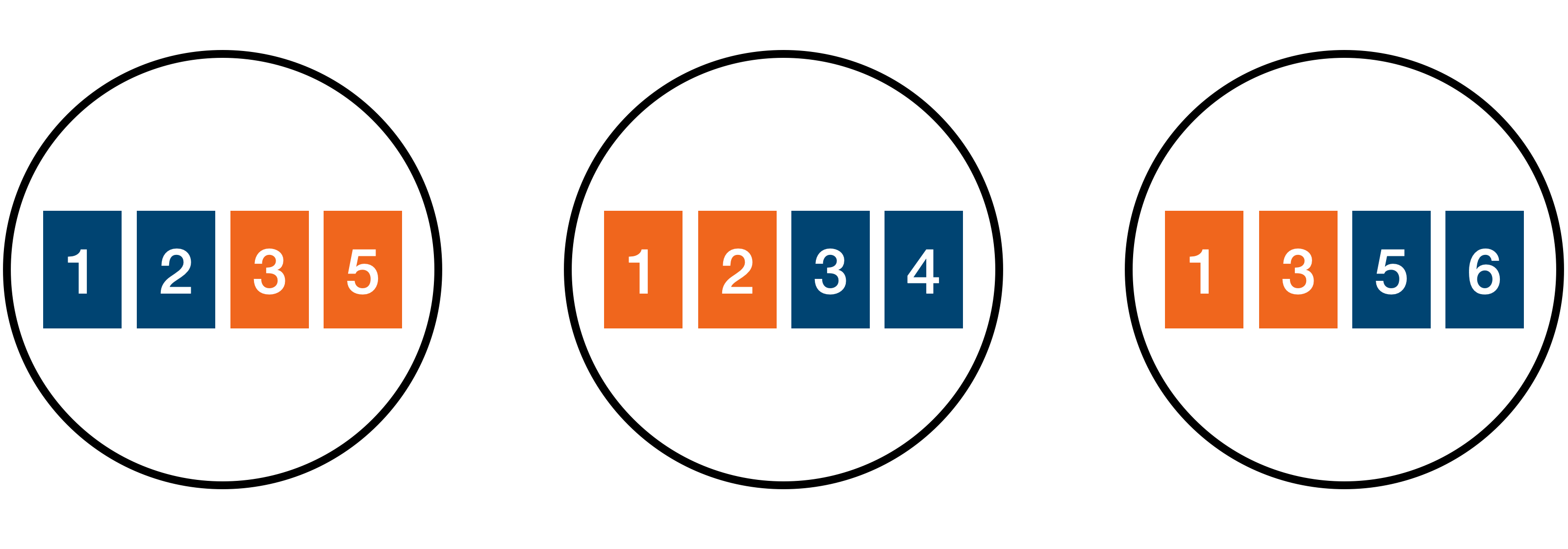}
         \caption{Updated active sets}
         \label{figExample6}
     \end{subfigure}
     \caption{Illustration of the active sets in Algorithm \ref{algGeneral}; see Example \ref{examAlg} for details.} \label{figExample}
	\Description{Illustration of the active sets in Algorithm \ref{algGeneral}; see Example \ref{examAlg} for details.}
\end{figure}

We assume $m=0$ this subsection and describe the non-blocking algorithm from \cite{chawla2020gossiping}.
\begin{itemize}
\item {\bf Phases:} In \cite{chawla2020gossiping}, the time steps $1,\ldots,T$ are grouped into \textit{phases}, whose role is twofold. First, within the $j$-th phase, the $i$-th honest agent only pulls arms belonging to a particular subset $S_j^{(i)} \subset [K]$. We call these \textit{active sets} and detail their construction next. Second, at the end of the $j$-th phase, the agents construct new active sets by exchanging \textit{arm recommendations} with neighbors, in a manner to be described shortly. See Figure \ref{figAlgorithm} for a pictorial description. Notice that the phase durations are increasing, which will be discussed in Section \ref{secAlgFormal}.

\item {\bf Active sets:} The active set $S_j^{(i)}$ will always contain a subset of arms $\hat{S}^{(i)} \subset [K]$ that does not vary with the phase $j$. Following \cite{chawla2020gossiping,vial2021robust}, we call $\hat{S}^{(i)}$ the \textit{sticky set} and its elements \textit{sticky arms}. The sticky sets ensure that each arm is explored by some agent, as will be seen in the forthcoming example. In addition, $S_j^{(i)}$ will contain two non-sticky arms that are dynamically updated across phases $j$ based on arm recommendations from neighbors.

\item {\bf Arm recommendations:} After the $j$-th phase, each agent $i$ contacts a random neighbor, who responds with whichever of their active arms performed ``best'' in the current phase. Upon receiving this recommendation, $i$ adds it to its active set and discards whichever currently-active non-sticky arm (i.e., whichever element of $S_j^{(i)} \setminus \hat{S}^{(i)}$) performed ``worse''. (We quantify ``best'' and ``worse'' in the formal discussion of Section \ref{secAlgFormal}.)
\end{itemize}

\begin{exam} \label{examAlg}
Each subfigure of Figure \ref{figExample} depicts $n=3$ honest agents as circles and their active sets as rectangles. The blue rectangles are sticky sets, the orange rectangles are non-sticky arms, and the arms are sorted by performance. For example, the left agent in Figure \ref{figExample1} has sticky set $\{1,2\}$ and active set $\{1,2,3,6\}$ and believes arm $3$ to be the best of these. Note the blue sticky sets partition $[K] = [6]$, so at each phase, each arm is active for some agent. This ensures the best arm is never permanently discarded during the arm recommendations discussed above. Figure \ref{figExample2} shows agents recommending the active arms they believe are best, and Figure \ref{figExample3} depicts the updated active sets. For instance, the left agent replaces its worse non-sticky arm $6$ with the recommendation $5$. Figure \ref{figExample4} shows a later phase where the best arm $1$ has spread to all agents, who have all identified it as such. Thereafter, all agents recommend $1$, so the active set remains fixed (Figures \ref{figExample5} and \ref{figExample6}). Hence, all agents eventually exploit the best arm while only exploring a subset of the suboptimal arms (three instead of five here).
\end{exam}

\subsection{Formal definition of the blocking algorithm} \label{secAlgFormal}

The algorithm in \cite{vial2021robust} supplements the one from Section \ref{secAlgExample} with a blocking procedure. Specifically, honest agents run the algorithm from \cite{chawla2020gossiping} while maintaining \textit{blocklists} of neighbors they are unwilling to communicate with. This approach is defined in Algorithm \ref{algGeneral} and detailed next.

\begin{algorithm} \caption{Multi-agent bandits with blocking (executed by $i \in [n]$)} \label{algGeneral}

{\bf Input:} UCB parameter $\alpha > 0$, phase parameter $\beta > 1$, sticky arms $\hat{S}^{(i)}$ with $|\hat{S}^{(i)}|=S \leq K-2$\label{algInput}

Initialize $A_j= \ceil{ j^{\beta} }, P_j^{(i)} = \emptyset\ \forall\ j \in \N$ (communication times and blocklists) \label{algInitCommGossip}

Set $j=1$ (current phase), let $\{ U_j^{(i)}, L_j^{(i)} \} \subset [K] \setminus \hat{S}^{(i)}$ be two distinct non-sticky arms and $S_j^{(i)} = \hat{S}^{(i)} \cup \{ U_j^{(i)}, L_j^{(i)} \}$ the initial active set \label{algInitActiveSet}

\For{$t \in \N$}{

Pull $I_t^{(i)} = \argmax_{k \in S_j^{(i)} }   \left( \hat{\mu}_k^{(i)}(t-1) + \sqrt{  \alpha \log(t) / T_k^{(i)}(t-1) } \right)$ (UCB over active set) \label{algPullArm}

\If{$t = A_j$ (if communication occurs at this time)}{

$B_j^{(i)} = \argmax_{k \in S_j^{(i)} } \left( T_k^{(i)}(A_j) - T_k^{(i)}(A_{j-1}) \right)$ (most played active arm in this phase) \label{algMostPlayed}

$\{ P_{j'}^{(i)} \}_{j'=j}^{\infty} \leftarrow \texttt{Update-Blocklist} ( \{ P_{j'}^{(i)} \}_{j'=j}^{\infty}  )$ (\cite{vial2021robust} uses Alg.\ \ref{algUpdateBlockOld}; we propose Alg.\ \ref{algUpdateBlockNew}) \label{algUpdateGossipDist}

$(H_j^{(i)} , R_j^{(i)}) = \texttt{Get-Recommendation}(i,j,P_j^{(i)})$ (see Alg.\ \ref{algGetArm}) \label{algGetRec}

\eIf{$R_j^{(i)} \notin S_j^{(i)}$ (if recommendation not already active)}{ \label{algUpdateActiveStart}

$U_{j+1}^{(i)} = \argmax_{k \in \{ U_j^{(i)} , L_j^{(i)} \} } \left( T_k^{(i)}(A_j) - T_k^{(i)}(A_{j-1})\right) $ (best non-sticky active arm) \label{algUjPlus1}

$L_{j+1}^{(i)} = R_j^{(i)}$ (replace worst non-sticky active arm with recommendation) \label{algLjPlus1}

$S_{j+1}^{(i)} = \hat{S}^{(i)} \cup \{ U_{j+1}^{(i)}, L_{j+1}^{(i)} \}$ (new active set is sticky set and two non-sticky arms) \label{algNewActive}

}{

$S_{j+1}^{(i)} = S_j^{(i)}$ (keep the same active set, since recommendation is already active) \label{algSameActive}

} \label{algUpdateActiveEnd}

$j \leftarrow j+1$ (increment phase) \label{algIncrPhase}

}

}

\end{algorithm}

\begin{algorithm} \caption{$(H_j^{(i)} , R_j^{(i)}) = \texttt{Get-Recommendation}(i,j,P_j^{(i)})$ (black box to $i \in [n]$)} \label{algGetArm}

{\bf Input:} Agent $i \in \{1,\ldots,n\}$, phase $j \in \N$, blocklist $P_j^{(i)}$

Sample $H_j^{(i)}$ from $N(i) \setminus P_j^{(i)}$ (non-blocked neighbors) uniformly at random \label{algSampleRecommender}

\eIf{$H_j^{(i)} \leq n$ (if the sampled agent is honest)}{

Set $R_j^{(i)} = B_j^{ ( H_j^{(i)} ) }$ (honest agents recommend most played arm from this phase) \label{algRecMostPlayed}}{

Choose $R_j^{(i)} \in [K]$ arbitrarily (malicious agents recommend arbitrary arms)

}

{\bf Output:} $(H_j^{(i)}, R_j^{(i)})$

\end{algorithm}

\vspace{3pt} \noindent {\bf Inputs (Line \ref{algInput}):} The first input is a standard UCB exploration parameter $\alpha > 0$, which will be discussed shortly. The input $\beta > 1$ controls the lengths of the phases; namely, the $j$-th phase encompasses times $1+A_{j-1} , \ldots , A_j$, where $A_j \triangleq  \ceil{j^\beta}$. Note the phase duration $A_j - A_{j-1} = O(j^{\beta-1})$ grows with $j$, as shown in Figure \ref{figAlgorithm}. The final input is an $S$-sized sticky set $\hat{S}^{(i)}$ ($S=2$ in Example \ref{examAlg}), which, as in \cite{vial2021robust}, we assume are provided to the agents (see Section \ref{secAlgAssumption} for more details).

\vspace{3pt} \noindent {\bf Initialization (Lines \ref{algInitCommGossip}-\ref{algInitActiveSet}):} To start, $i$ initializes the times $A_j$ at which the $j$-th phase ends, along with the blocklist $P_j^{(i)}$. Additionally, $i$ chooses two distinct (but otherwise arbitrary) non-sticky arms $U_1^{(i)}$ and $L_1^{(i)}$ and constructs the active set $S_1^{(i)} = \hat{S}^{(i)} \cup \{U_1^{(i)} , L_1^{(i)} \}$. Notice that the active set contains the sticky set and two arms that depend on the phase, as described in Section \ref{secAlgExample}.

\vspace{3pt} \noindent  {\bf UCB over the active set (Line \ref{algPullArm}):} As was also mentioned in Section \ref{secAlgExample}, $i$ only pulls arms from its current active set $S_j^{(i)}$. More specifically, at each time $t$ during phase $j$, $i$ chooses the active arm $I_t^{(i)} \in S_j^{(i)}$ that maximizes the UCB in Line \ref{algPullArm} (see \cite[Chapters 7-10]{lattimore2020bandit} for background). Here $T_k^{(i)}(t-1)$ and $\hat{\mu}_k^{(i)}(t-1)$ are the number of pulls of $k$ and the empirical mean of those pulls, i.e.,
\begin{equation*}
T_k^{(i)}(t-1) = \sum_{s \in [t-1]} \ind ( I_s^{(i)} = k ) , \quad \hat{\mu}_k^{(i)}(t-1) = \frac{1}{T_k^{(i)}(t-1)} \sum_{s \in [t-1] : I_s^{(i)} = k} X_{i,s}(k) ,
\end{equation*}
where $X_{i,s}(k) \sim \nu_k$ as in Section \ref{secPrelim} and $\ind$ is the indicator function.

\vspace{3pt} \noindent {\bf Best arm estimate (Line \ref{algMostPlayed}):} At the end of phase $j$ (i.e., when $t = A_j$), $i$ defines its best arm estimate $B_j^{(i)}$ as the active arm it played the most in phase $j$. The intuition is that, for large horizons, the arm chosen most by UCB is a good estimate of the true best arm \cite{bubeck2011pure}. Thus, because phase lengths are increasing (see Figure \ref{figAlgorithm}), $B_j^{(i)}$ will be a good estimate of the best active arm for large $j$.

\vspace{3pt} \noindent {\bf Blocklist update (Line \ref{algUpdateGossipDist}):} Next, $i$ calls the \texttt{Update-Blocklist} subroutine to update its blocklist $P_j^{(i)}$. The implementation of this subroutine is the key distinction between \cite{vial2021robust} and our work. We discuss the respective implementations in Sections \ref{secExisting} and \ref{secProposed}, respectively.

\vspace{3pt} \noindent {\bf Arm recommendations (Line \ref{algGetRec}):} Having updated $P_j^{(i)}$, $i$ requests an arm recommendation $R_j^{(i)}$ via Algorithm \ref{algGetArm}. Algorithm \ref{algGetArm} is a black box (i.e., $i$ provides the input and observes the output), which samples a random non-blocked neighbor $H_j^{(i)} \in N(i) \setminus P_j^{(i)}$. If $H_j^{(i)}$ is honest, it recommends its best arm estimate, while if malicious, it recommends an arbitrary arm.\footnote{Technically, malicious recommendations need to be measurable; see \cite[Section 3]{vial2021robust} for details.}

\vspace{3pt} \noindent {\bf Updating the active set (Lines \ref{algUpdateActiveStart}-\ref{algUpdateActiveEnd}):} Finally, $i$ updates its active set as in Section \ref{secAlgExample}. In particular, if the recommendation $R_j^{(i)}$ is not currently active\footnote{If the recommendation is currently active, the active set remains unchanged (see Line \ref{algSameActive} of Algorithm \ref{algGeneral}).}, $i$ defines $U_{j+1}^{(i)}$ to be the non-sticky arm that performed better in phase $j$, in the sense that UCB chose it more often (following the above intuition from \cite{bubeck2011pure}). The other non-sticky arm $L_{j+1}^{(i)}$ becomes the recommendation $R_j^{(i)}$, and the new active set becomes $S_{j+1}^{(i)} = \hat{S}^{(i)} \cup \{ U_{j+1}^{(i)} , L_{j+1}^{(i)} \}$ (the sticky set and two other arms, as above).

\subsection{Additional assumption} \label{secAlgAssumption}

Observe that Algorithm \ref{algGeneral} does not preclude the case where the best arm is not in any honest agent's sticky set, i.e., $1 \notin \cup_{i=1}^n \hat{S}^{(i)}$. In this case, the best arm may be permanently discarded, which causes linear regret even in the absence of malicious agents. For example, this would occur if $1$ was not a sticky arm for the left agent in Figure \ref{figExample} (since the right agent discards $1$ in Figure \ref{figExample3}). To prevent this situation, we will follow \cite{sankararaman2019social,chawla2020gossiping,newton2021asymptotic,vial2021robust} in assuming the following.

\begin{ass} \label{assSticky}
There exists $i^\star \in [n]$ with the best arm in its sticky set, i.e., $1 \in \hat{S}^{(i^\star)}$.
\end{ass}

\begin{rem} \label{remSticky}
As discussed in \cite[Appendix N]{chawla2020gossiping}, Assumption \ref{assSticky} holds with high probability if $S$ (the size of the sticky set input to Algorithm \ref{algGeneral}) is set to $\tilde{\Theta}(K/n)$ and each sticky set $\hat{S}^{(i)}$ is sampled uniformly at random from the $S$-sized subsets of $[K]$.
\end{rem}

\begin{rem} \label{remKnowledgeOfN}
The choice $S = \tilde{\Theta}(K/n)$ from Remark \ref{remSticky} requires the honest agents to know an order-accurate estimate of $n$, i.e., they need to know some $n' = \tilde{\Theta}(n)$ in order to set $S = \tilde{\Theta}(K/n')$ and ensure that $S = \tilde{\Theta}(K/n)$. As discussed in \cite[Remark 7]{vial2021robust}, this amounts to knowing order-accurate estimates of $n+m$ and $n/(n+m)$. The former quantity is the total number of agents, knowledge of which is rather benign and is also assumed in the fully-cooperative setting \cite{sankararaman2019social,chawla2020gossiping,newton2021asymptotic}. The latter requires the agents to know that, e.g., half of the others are honest, which is similar in spirit to the assumptions in related problems regarding social learning in the presence of adversarial agents (e.g., \cite{leblanc2013resilient}).
\end{rem}

\begin{rem} \label{remAltModel}
Alternatively, we can avoid Assumption \ref{assSticky} entirely by defining the set of the arms to be those initially known by the honest agents (i.e., their sticky sets), rather than sampling the sticky sets from a larger ``base set" as in Remark \ref{remSticky}. In this alternative model, the honest agents aim to identify and spread through the network whichever of the initially-known arms is best, similar to what happens on platforms like Yelp (see Example \ref{examRecommend}). In contrast, the Section \ref{secPrelim} model allows for the pathological case where the base set contains a better arm than any initially known to honest agents (e.g., where no honest Yelp user has ever dined at the best restaurant). Coping with these pathological cases either requires Assumption \ref{assSticky}, or another mode of exploration (i.e., exploration of base arms) that obfuscates the key point of our work (collaborative bandit exploration amidst adversaries). For these reasons, we prefer the alternative model, but to enable a cleaner comparison with prior work \cite{vial2021robust}, we restrict attention to the Section \ref{secPrelim} model (which generalizes that of \cite{vial2021robust}).
\end{rem}

%% file: existing.tex
\section{Existing blocking rule} \label{secExisting}

We can now define the blocking approach from \cite{vial2021robust}, which is provided in Algorithm \ref{algUpdateBlockOld}. In words, the rule is as follows: if the recommendation $R_{j-1}^{(i)}$ from phase $j-1$ is not $i$'s most played arm in the subsequent phase $j$, then the agent $H_{j-1}^{(i)}$ who recommended it is added to the blocklists $P_j^{(i)} , \ldots , P_{j^\eta}^{(i)}$, where $\eta > 1$ is a tuning parameter. By Algorithm \ref{algGetArm}, this means $i$ blocks (i.e., does not communicate with) $H_{j-1}^{(i)}$ until phase $j^\eta+1$ (at the earliest). Thus, agents block others whose recommendations perform poorly -- in the sense that UCB does not play them often -- and the blocking becomes more severe as the phase counter $j$ grows. See \cite[Remark 4]{vial2021robust} for further intuition.

\begin{algorithm} \caption{$\{ P_{j'}^{(i)} \}_{j'=j}^{\infty}= \texttt{Update-Blocklist}$ (executed by $i \in [n]$, existing rule from \cite{vial2021robust})} \label{algUpdateBlockOld}

\If{$j > 1$ and $B_j^{(i)} \neq R_{j-1}^{(i)}$ (if previous recommendation not most played)}{ 

$P_{j'}^{(i)} \leftarrow P_{j'}^{(i)} \cup \{ H_{j-1}^{(i)} \}\ \forall\ j' \in \{j,\ldots,\ceil{j^\eta}\}$ (block the recommender until phase $j^\eta$)

}
\end{algorithm}

In the remainder of this section, we define a bad instance (Section \ref{secBadInstance}) on which this blocking rule provably fails (Section \ref{secNegResult}). Our goal here is to demonstrate a single such instance in order to show this blocking rule must be refined. Therefore, we have opted for a concrete example, which includes some numerical constants (e.g., $13/15$ in \eqref{eqBadInstanceMu}, the $7$ in the $\log^7 T$ term in Theorem \ref{thmExisting}, etc.) that have no particular meaning. Nevertheless, the instance can be generalized; see Remark \ref{remOneBad}.

\subsection{Bad instance} \label{secBadInstance}

The network and bandit for the bad instance are as follows:
\begin{itemize}
\item There are an even number of honest agents (at least four) arranged in a line, increasing in index from left to right, and there is a malicious agent connected to each of the honest ones. Mathematically, we have $n \in \{4,6,8,\ldots\}$, $m=1$, and $E = \{ (i,i+1) \}_{i=1}^{n-1} \cup \{ (i,n+1) \}_{i=1}^n$.
\item There are $K=n$ arms that generate deterministic rewards (i.e., $\nu_k = \delta_{\mu_k}$) with
\begin{equation} \label{eqBadInstanceMu}
\mu_1 = 1 , \quad \mu_k= \frac{13}{15} + \sum_{h=1}^{(n/2)-k} 2^{-2^{h+1}}\ \forall\ k \in \{ 2 , \ldots , n/2 \} , \quad \mu_k = 0\ \forall\ k > n/2 .
\end{equation}
Intuitively, there are three sets of arms: the best arm, $(n/2)-1$ mediocre arms, and $n/2$ bad arms. We provide further intuition in the forthcoming proof sketch. For now, we highlight three key properties. First, the gap from mediocre to bad arms is constant, i.e., $\mu_{k_1} - \mu_{k_2} \geq 13/15$ when $k_1 \leq n/2 < k_2$. Second, the gaps between mediocre arms are doubly exponentially small, i.e., $\mu_k - \mu_{k+1} = 2^{ - 2^{ (n/2) - k + 1 } }$ for $k \in \{2,\ldots,(n/2)-1\}$. Third, the gap $\Delta_2$ from the best to the mediocre arms is at least $1/15$, as shown in Appendix \ref{appProofExisting}.
\end{itemize}

\begin{obs} \label{obsAvoidBlock}
Since rewards are deterministic, the most played arm $B_{j+1}^{(i)}$ in phase $j+1$ is a deterministic function of the number of plays of the active arms at the beginning of the phase, i.e., of the set $\{ T_k^{(i)}(A_j) \}_{k \in S_{j+1}^{(i)} }$. Hence, when the $j$-th recommendation is already active (i.e., when $R_j^{(i)} \in S_j^{(i)}$, which implies $S_{j+1}^{(i)} = S_j^{(i)}$ in Algorithm \ref{algGeneral}), $B_{j+1}^{(i)}$ is a function of $\{ T_k^{(i)}(A_j) \}_{k \in S_j^{(i)} }$, which is information available to the malicious agent at the $j$-th communication time $A_j$. Consequently, the malicious agent can always recommend some $R_j^{(i)} \in S_j^{(i)}$ such that $B_{j+1}^{(i)} = R_j^{(i)}$ to avoid being blocked by $i$.
\end{obs}

We make the following assumptions on Algorithms \ref{algGeneral} and \ref{algGetArm}:
\begin{itemize}
\item The parameters in Algorithm \ref{algGeneral} are $\alpha = 4$ and $\beta = 2$, while $\eta = 2$ in Algorithm \ref{algUpdateBlockOld}.
\item Sticky sets have size $S=1$ and for any $i \in \{1+n/2,\ldots,n\}$, $i$'s initial active set satisfies $\min S_1^{(i)} > n/2$. Thus, active sets contain three arms, and the right half of the honest agents are initially only aware of the bad arms, i.e., of those that provide no reward.
\end{itemize}

\begin{rem} \label{remOneBad}
Note that Assumptions \ref{assGraph}-\ref{assSticky} all hold for this instance, and the choices $\alpha = 4$ and $\beta = \eta = 2$ are used for the complete graph experiments in \cite{vial2021robust}. Additionally, the instance can be significantly generalized -- the key properties are that $K$ and $n$ have the same scaling, the gaps from mediocre arms to others are constant, the gaps among mediocre arms are doubly exponentially small, and a constant fraction of agents on the right initially only have bad arms active. 
\end{rem}

Finally, we define a particular malicious agent strategy. Let $J_1 = 2^8$ and inductively define $J_{l+1} = ( J_l + 2 )^2$ for each $l \in \N$. Then the malicious recommendations are as follows:
\begin{itemize}
\item If $j = J_l$ and $i \in \{ l + 1 + n/2  , l+ 2 + n/2 \}$ for some $l \in [(n/2)-1]$, set $R_j^{(i)} = 1 - l + n/2$.
\item Otherwise, let $R_j^{(i)} \in S_j^{(i)}$ be such that $B_{j+1}^{(i)} = R_j^{(i)}$ (see Observation \ref{obsAvoidBlock}).
\end{itemize}
Similar to the arm means, we will wait for the proof sketch to explain this strategy in more detail. For now, we only mention that the phases $J_l$ grow doubly exponentially, i.e.,
\begin{equation}\label{eqJlDoubExp}
J_{l+1} = ( J_l + 2 )^2 > J_l^2 > \cdots > J_1^{2^l}\ \forall\ l \in \N .
\end{equation}

\subsection{Negative result} \label{secNegResult}
 
We can now state the main result of this section. It shows that if the existing blocking rule from \cite{vial2021robust} is used on the above instance, then the honest agent $n$ at the end of the line suffers nearly linear regret $\tilde{\Omega}(T)$ until time $T$ exceeds a doubly exponential function of $n=K$.
\begin{thm} \label{thmExisting}
If we run Algorithm \ref{algGeneral} and use Algorithm \ref{algUpdateBlockOld} as the \texttt{Update-Blocklist} subroutine with the parameters and problem instance described in Section \ref{secExisting}, then
\begin{equation*}
R_T^{(n)} = \Omega \left( \min \left\{ \log(T) + \exp \left( \exp \left( n / 3 \right) \right) , T / \log^7 T \right\} \right) .
\end{equation*}
\end{thm}
\begin{proof}[Proof sketch]
We provide a complete proof in Appendix \ref{appProofExisting} but discuss the intuition here.
\begin{itemize}
\item First, suppose honest agent $1+n/2$ contacts the malicious agent $n+1$ at all phases $j \in [J_1-1]$ (this occurs with constant probability since $J_1$ is constant). Then the right half of honest agents (i.e., agents $1+n/2,\ldots,n$) only have bad arms (i.e., arms $1+n/2,\ldots,n$) in their active sets at phase $J_1$. This is because their initial active sets only contain such arms (by assumption), $n+1$ only recommends currently-active arms before $J_1$, and no arm recommendations flow from the left half of the graph to the right half (they need to first be sent from $n/2$ to $1+n/2$, but we are assuming the latter only contacts $n+1$ before $J_1$).

\item Now consider phase $J_1$. With constant probability, $1+n/2$ and $2+n/2$ both contact $n+1$, who recommends a currently active (thus bad) arm and the mediocre arm $n/2$, respectively. Then, again with constant probability, $2+n/2$ contacts $1+n/2$ at the next phase $J_1+1$; $1+n/2$ only has bad arms active and thus recommends a bad arm. Therefore, during phase $J_1+2$, agent $2+n/2$ has the mediocre arm $n/2$ and some bad recommendation from $1+n/2$ in its active set. The inverse gap squared between these arms is constant, thus less than the length of phase $J_1+2$ (for appropriate $J_1$), so by standard bandit arguments (basically, noiseless versions of best arm identification results from \cite{bubeck2011pure}), $n/2$ will be most played. Consequently, by the blocking rule in Algorithm \ref{algUpdateBlockOld}, $2+n/2$ blocks $1+n/2$ until phase $(J_1+2)^2 = J_2$.

\item We then use induction. For each $l \in [(n/2)-1]$ ($l=1$ in the previous bullet), suppose $l+1+n/2$ blocks $l+n/2$ between phases $J_l+2$ and $J_{l+1}$. Then during these phases, no arm recommendations flow past $l+n/2$, so agents $\geq l+1+n/2$ only play arms $\geq 1-l+n/2$. At phase $J_{l+1}$, the malicious agent recommends $k \geq 1-l+n/2$ and $-l+n/2$ to agents $l+1+n/2$ and $l+2+n/2$, respectively, and at the subsequent phase $J_{l+1}+1$, $l+1+n/2$ recommends $k' \geq l+1+n/2$ to $l+2+n/2$. Similar to the previous bullet, we then show $l+2+n/2$ plays arm $-l+n/2$ more than $k'$ during phase $J_{l+1}+2$ and thus blocks $l+1+n/2$ until $(J_{l+1}+2)^2 = J_{l+2}$, completing the inductive step. The proof that $-l+n/2$ is played more than $k'$ during phase $J_{l+1}+2$ again follows from noiseless best arm identification, although unlike the previous bullet, the relevant arm gap is no longer constant (both could be mediocre arms). However, we chose the mediocre arm means such that their inverse gap squared is at most doubly exponential in $l$, so by \eqref{eqJlDoubExp}, the length of phase $J_{l+1}$ dominates it.
\end{itemize}

In summary, we show that due to blocking amongst honest agents, $l+1+n/2$ does not receive arm $1-l+n/2$ until phase $J_l$, given that some constant probability events occur at each of the times $J_1 , \ldots , J_l$. This allows us to show that, with probability at least $\exp ( - \Omega(n) )$, agent $n$ does not receive the best arm until phase $J_{n/2} = \exp ( \exp ( \Omega(n) ) )$, and thus does not play the best arm until time $\exp ( \exp ( \Omega(n) ) )$ in expectation. Since $\Delta_2$ is constant, we can lower bound regret similarly.
%
\end{proof}

%% file: proposed.tex
\section{Proposed blocking rule} \label{secProposed}

To summarize the previous section, we showed that the existing blocking rule (Algorithm \ref{algUpdateBlockOld}) may result in honest agents blocking too aggressively, which causes the best arm to spread very slowly. In light of this, we propose a relaxed blocking criteria (see Algorithm \ref{algUpdateBlockNew}): at phase $j$, agent $i$ will block the agent $H_{j-1}^{(i)}$ who recommended arm $R_{j-1}^{(i)}$ at the previous phase $j-1$ if 
\begin{equation} \label{eqNewBlock}
T_{ R_{j-1}^{(i)}}^{(i)}(A_j) \leq \kappa_j  \quad \text{and} \quad B_j^{(i)} = B_{j-1}^{(i)} = \cdots = B_{ \floor{\theta_j} }^{(i)}  ,
\end{equation}
where $\kappa_j \leq A_j$ and $\theta_j \leq j$ are tuning parameters. Thus, $i$ blocks if \textit{both} of the following occur:
\begin{itemize}
\item The recommended arm $R_{j-1}^{(i)}$ performs poorly, in the sense that UCB has not chosen it sufficiently often (i.e., at least $\kappa_j$ times) by the end of phase $j$.
\item Agent $i$ has not changed its own best arm estimate since phase $\theta_j$. Intuitively, this can be viewed as a confidence criterion: if instead $i$ has recently changed its estimate, then $i$ is currently unsure which arm is best, so should not block for recommendations that appear suboptimal at first glance (i.e., those for which the first criterion in \eqref{eqNewBlock} may hold).
\end{itemize}

\begin{algorithm} \caption{$\{ P_{j'}^{(i)} \}_{j'=j}^{\infty}= \texttt{Update-Blocklist}$ (executed by $i \in [n]$, proposed rule)} \label{algUpdateBlockNew}

\lIf{$j > 1$ and \eqref{eqNewBlock} holds}{$P_{j'}^{(i)} \leftarrow P_{j'}^{(i)} \cup \{ H_{j-1}^{(i)} \}\ \forall\ j \in \{j,\ldots,\ceil{j^\eta}\}$ (block recommender)}

\end{algorithm}

\begin{rem} \label{remMotivate}
The first criterion in \eqref{eqNewBlock} is a natural relaxation of demanding the recommended arm be most played. The second is directly motivated by the negative result from Section \ref{secExisting}. In particular, recall from the Theorem \ref{thmExisting} proof sketch that $l+1+n/2$ blocked $l+n/2$ shortly after receiving a new mediocre arm from the malicious agent. Thus, blocking amongst honest agents was always precipitated by the blocking agent changing its best arm estimate. The second criterion in \eqref{eqNewBlock} aims to avoid this.
\end{rem}

\begin{rem} \label{remParameters}
Our proposed rule has two additional parameters compared to the existing one: $\kappa_j$ and $\theta_j$. For our theoretical results, these will be specified in Theorem \ref{thmProposed}; for experiments, they are discussed in Section \ref{secExp}. For now, we only mention that they should satisfy two properties. First, $\kappa_j$ should be $o(A_j)$, so that the first criterion in \eqref{eqNewBlock} dictates a sublinear number of plays. Second, $j - \theta_j$ should grow with $j$, since (as discussed above) the second criterion represents the confidence in the best arm estimate, which grows as the number of reward observations increases.
\end{rem}

In the remainder of this section, we introduce a further definition (Section \ref{secNoisyRumor}), provide a general regret bound under our blocking rule (Section \ref{secPosResult}), and discuss some special cases (Section \ref{secPosCor}).

\subsection{Noisy rumor process} \label{secNoisyRumor}

As discussed in Section \ref{secContributions}, we will show that under our proposed rule (1) honest agents eventually stop blocking each other, and (2) honest agents with the best arm active will eventually recommend it to others. Thereafter, we essentially reduce the arm spreading process to a much simpler rumor process in which each honest agent $i$ contacts a uniformly random neighbor $i'$ and, if $i'$ is an honest agent who knows the rumor (i.e., if the best arm is active for $i'$), then $i'$ informs $i$ of the rumor (i.e., $i'$ recommends the best arm to $i$). The only caveat is that we make \textit{no assumption} on the malicious agent arm recommendations, so we have \textit{no control} over whether or not they are blocked. In other words, the rumor process unfolds over a dynamic graph, where edges between honest and malicious agents may or may not be present, and we have no control over these dynamics. 

In light of this, we take a worst-case view and lower bound the arm spreading process with a noisy rumor process that unfolds on the (static) honest agent subgraph. More specifically, we consider the process $\{ \bar{\mathcal{I}}_j \}_{j=0}^\infty$ that tracks the honest agents informed of the rumor. Initially, only $i^\star$ (the agent from Assumption \ref{assSticky}) is informed (i.e., $\bar{\mathcal{I}}_0 = \{ i^\star \}$). Then at each phase $j \in \N$, each honest agent $i$ contacts a random honest neighbor $i'$. If $i'$ is informed (i.e., if $i' \in \bar{\mathcal{I}}_{j-1}$), then $i$ becomes informed as well (i.e., $i \in \bar{\mathcal{I}}_j$), subject to some $\text{Bernoulli}(\Upsilon)$ noise, where $\Upsilon \leq d_{\text{hon}}(i) / d(i)$. Hence, $i$ becomes informed with probability $| \bar{\mathcal{I}}_{j-1} \cap N_{\text{hon}}(i) | \Upsilon / d_{\text{hon}}(i) \leq  | \bar{\mathcal{I}}_{j-1} \cap N_{\text{hon}}(i) | / d(i)$. Note the right side of this inequality is in turn upper bounded by the probability with which they receive the best arm in the process of the previous paragraph.

More formally, we define the noisy rumor process as follows. The key quantity in Definition \ref{defnRumor} is $\bar{\tau}_{\text{spr}}$, the first phase all are informed. Analogous to \cite{chawla2020gossiping}, our most general result will be in terms of the expected time that this phase occurs, i.e., $\E [A_{ \bar{\tau}_{\text{spr}} }]$. Under Assumption \ref{assGraph}, the latter quantity is $\tilde{O} ( ( n \bar{d}_{\text{hon}} / \Upsilon )^\beta )$, which cannot be improved in general (see Appendix \ref{appCoarseBound}). However, Section \ref{secPosCor} provides sharper bounds for $\E [A_{ \bar{\tau}_{\text{spr}} }]$ in some special cases.

\begin{defn} \label{defnRumor}
Let $\Upsilon = \min_{i \in [n]} d_{\text{hon}}(i) / d(i)$. For each honest agent $i \in [n]$, let $\{ \bar{Y}_j^{(i)} \}_{j=1}^\infty$ be i.i.d.\ $\text{Bernoulli}(\Upsilon)$ random variables and $\{ \bar{H}_j^{(i)} \}_{j=1}^\infty$ i.i.d.\ random variables chosen uniformly at random from $N_{\text{hon}}(i)$. Inductively define $\{ \bar{\mathcal{I}}_j \}_{j=0}^\infty$ as follows: $\bar{\mathcal{I}}_0 = \{ i^\star \}$ (the agent from Assumption \ref{assSticky}) and
\begin{equation} \label{eqPullsSetsNoisy}
\bar{\mathcal{I}}_j  = \bar{\mathcal{I}}_{j -1} \cup \{ i \in [n] \setminus \bar{\mathcal{I}}_{j -1} : \bar{Y}_j^{(i)} = 1 , \bar{H}_j^{(i)} \in \bar{\mathcal{I}}_{j -1}  \}\ \forall\ j \in \N .
\end{equation}
Finally, let $\bar{\tau}_{\text{spr}} = \inf \{ j \in \N : \bar{\mathcal{I}}_j = [n] \}$.
\end{defn}

\subsection{Positive result} \label{secPosResult}

We can now present the main result of this section: a regret upper bound for the proposed blocking rule. We state it first and then unpack the statement in some ensuing remarks. The proof of this result (and all others in this section) is deferred to Appendix \ref{appProofProposed}.
\begin{thm} \label{thmProposed}
Let Assumptions \ref{assGraph}-\ref{assSticky} hold. Suppose we run Algorithm \ref{algGeneral} and use Algorithm \ref{algUpdateBlockNew} as the \texttt{Update-Blocklist} subroutine with $\theta_j = (j/3)^{\rho_1}$ and $\kappa_j = j^{\rho_2} / ( K^2 S )$ in \eqref{eqNewBlock}. Also assume
\begin{equation} \label{eqParamRequire}
\beta > 1 , \quad \eta > 1 , \quad 0 < \rho_1 \leq \frac{1}{\eta}  , \quad \alpha > \frac{3}{2} + \frac{1}{2\beta} + \frac{1}{2 \rho_1^2} , \quad  \frac{1}{2\alpha-3} < \rho_2 < \rho_1(\beta-1) .
\end{equation}
Then for any honest agent $i \in [n]$ and horizon $T \in \N$, we have
\begin{equation} \label{eqRegLogT}
R_T^{(i)} \leq 4 \alpha \log(T) \min \left\{ \frac{2 \eta-1}{\eta-1} \sum_{k=2}^{d_{\text{mal}}(i)+3} \frac{1}{\Delta_k} + \sum_{k=d_{\text{mal}}(i)+4}^{S+d_{\text{mal}}(i)+4} \frac{1}{\Delta_k} , \sum_{k=2}^K \frac{1}{\Delta_k} \right\} +  2 \E [ A_{ 2 \bar{\tau}_{\text{spr}}} ] + C_\star ,
\end{equation}
where $\Delta_k = 1$ by convention if $k > K$. Here $C_\star$ is a term independent of $T$ satisfying
\begin{equation}\label{eqRegAdditive}
C_\star = \tilde{O} \left( \max \left\{  d_{\text{mal}}(i) / \Delta_2  , ( K / \Delta_2 )^2 , S^{\beta/( \rho_1^2(\beta-1)) }  , ( S / \Delta_2^2 )^{\beta/(\beta-1)} , \bar{d}^{\beta/\rho_1} , n K^2 S \right\}  \right) ,
\end{equation}
where $\tilde{O}(\cdot)$ hides dependencies on $\alpha$, $\beta$, $\eta$, $\rho_1$, and $\rho_2$ and log dependencies on $K$, $n$, $m$, and $\Delta_2^{-1}$.
\end{thm}

\begin{rem}
The theorem shows that our algorithm's regret scales as $( d_{\text{mal}}(i) + S ) \log(T) / \Delta$, plus an additive term $2 \E [ A_{ 2 \bar{\tau}_{\text{spr}}} ] + C_\star$ that is independent of $T$ and polynomial in all other parameters. When $S = O(K/n)$ (see Remark \ref{remSticky}), the first term is $O ( ( d_{\text{mal}}(i) + K/n ) \log(T) / \Delta )$, as stated in Section \ref{secContributions}. Also, when $d_{\text{mal}}(i)$ is large, we recover the $O (K \log(T) / \Delta )$ single-agent bound (including the constant $4\alpha$), i.e., if there are many malicious agents, honest ones fare no worse than the single-agent case.
\end{rem}

\begin{rem} \label{remThmParam}
In addition to Assumptions \ref{assGraph}-\ref{assSticky}, the theorem requires the algorithmic parameters to satisfy \eqref{eqParamRequire}. For example, we can choose $\beta = \eta = 2$, $\rho_1 = 1/2$, $\alpha = 4$, and $\rho_2 = 1/3$. More generally, we view these five parameters as small numerical constants and hide them in the $\tilde{O}(\cdot)$ notation.
\end{rem}

\begin{rem} \label{remThmSimpler}
The bound in Theorem \ref{thmProposed} can be simplified under additional assumptions. For instance, in Example \ref{examRecommend}, it is reasonable to assume $K = \Theta(n)$ (i.e., the number of restaurants is proportional to the population) and $\bar{d} = O(1)$ (i.e., the degrees are constant, as in sparse social networks). Under these assumptions, the choice $S = O(K/n) = O(1)$ from Remark \ref{remSticky}, and the parameters from Remark \ref{remParameters}, the theorem's regret bound can be further upper bounded by
\begin{equation*}
R_T^{(i)} \leq \sum_{k=2}^{O(1)} \frac{48 \log T}{ \Delta_k }  + 2 \E [ A_{2 \bar{\tau}_{\text{spr}}} ] + \tilde{O} ( \max \{ (K / \Delta_2)^2 ,  \Delta_2^{-4} , n K^2 \} ) .
\end{equation*}
\end{rem}

\begin{rem} \label{remBadInstanceOurs}
Note the parameters from Remark \ref{remThmParam} were also used for the bad instance of Section \ref{secExisting}. There, we had $\Delta_k > 1/15$, $S =d_{\text{mal}}(i) = 1$, and $\E [ A_{ 2 \bar{\tau}_{\text{spr}}} ] = \tilde{O}( n^\beta )$, so our regret is $O(\log T )$ plus a polynomial additive term that is much smaller than the doubly exponential term in Section \ref{secExisting}.
\end{rem}

\begin{proof}[Proof sketch]
Let $\tau_{\text{spr}} = \inf \{ j \in \N : B_{j'}^{(i')} = 1\ \forall\ i' \in [n] , j' \geq j \}$ denote the first phase where the best arm is most played for all honest agents at all phases thereafter. Before this phase (i.e., before time $A_{\tau_{\text{spr}}}$) we simply upper bound regret by $\E [ A_{\tau_{\text{spr}}} ]$. The main novelty of our analysis is bounding $\E [ A_{\tau_{\text{spr}}} ]$ in terms of $C_\star$ and $\E [ A_{ \bar{\tau}_{\text{spr}}} ]$. We devote Section \ref{secAnalysis} to discussing this proof.

After phase $\tau_{\text{spr}}$, the best arm is active by definition, so $i$ incurs logarithmic in $T$ regret. We let
\begin{equation}
\tau_{\text{blk}}^{(i)} = \inf \{ j \in \N : H_{j'-1}^{(i)} \in P_{j'}^{(i)} \setminus P_{j'-1}^{(i)}\ \forall\  j' \geq j\ s.t.\ R_{j'-1}^{(i)} \neq 1 \}  
\end{equation}
be the earliest phase such that $i$ blocks for all suboptimal recommendations thereafter. We then split the phases after $\tau_{\text{spr}}$ into two groups: those before $\tilde{\tau}^{(i)} \triangleq \tau_{\text{spr}} \vee \tau_{\text{blk}}^{(i)} \vee T^{1/(\beta K)}$ and those after.

For the phases after $\tau_{\text{spr}}$ but before $\tilde{\tau}^{(i)}$, we consider three cases:
\begin{itemize}
\item $\tilde{\tau}^{(i)} = \tau_{\text{spr}}$: In this case, there are no such phases, so there is nothing to prove.

\item $\tilde{\tau}^{(i)} = T^{1/(\beta K)}$: Here we have an effective horizon $A_{\tilde{\tau}^{(i)}} = ( \tilde{\tau}^{(i)} )^\beta = T^{1/K}$, so similar to \cite{vial2021robust}, we exploit the fact that the best arm is active and modify existing UCB analysis to bound regret by $O( K \log ( T^{1/K} ) / \Delta ) = O (\log(T) / \Delta )$, which is dominated by \eqref{eqRegLogT} (in an order sense).

\item $\tilde{\tau}^{(i)} = \tau_{\text{blk}}^{(i)}$: Here we are considering phases $j$ where the best arm is most played by $i$ (since $j \geq \tau_{\text{spr}}$) but $i$ does not block suboptimal recommendations (since $j \leq \tau_{\text{blk}}^{(i)}$). Note that no such phases arise for the existing blocking rule, so here the proof diverges from \cite{vial2021robust}, and most of Appendix \ref{appProofThmProposed} is dedicated to this case. Roughly speaking, the analogous argument of the previous case yields the regret bound $O( (K/\Delta) \E [\log \tilde{\tau}^{(i)} ]$, and we prove this term is also $O( \log(T) / \Delta )$ by deriving a tail bound for $\tilde{\tau}^{(i)}$. The tail amounts to showing that, once the best arm is active, $i$ can identify suboptimal arms as such, within the phase. This in turn follows from best arm identification results and the growing phase lengths.
\end{itemize}

After phase $\tilde{\tau}^{(i)}$, the best arm is most played for all honest agents (since $\tilde{\tau}^{(i)} \geq \tau_{\text{spr}}$), so they only recommend this arm. Thus, $i$ only plays the best arm, its $S$ sticky arms, and any malicious recommendations. Consequently, to bound regret by $O( ( S + d_{\text{mal}}(i) ) \log(T) / \Delta )$ as in \eqref{eqRegLogT}, we need to show each malicious neighbor $i'$ only recommends $O(1)$ suboptimal arms. It is easy to see that $i'$ can only recommend $O(\log K)$ such arms: if $i'$ recommends a bad arm at phase $\tilde{\tau}^{(i)}$, they will be blocked until phase $T^{ \eta / ( \beta K ) }$ (since $\tilde{\tau}^{(i)} \geq \tau_{\text{blk}}^{(i)} \vee T^{1/(\beta K)}$), then until phase $( T^{ \eta / ( \beta K ) } )^\eta = T^{\eta^2 / (\beta K)}$, etc. Thus, the $(\log_\eta K)$-th bad recommendation occurs at phase $T^{ \eta^{\log_\eta K} / ( \beta K ) } = T^{1/\beta}$, which is time $T$ by definition $A_j = j^\beta$. Finally, an argument from \cite{vial2021robust} sharpens this $O(\log K)$ term to $O(1)$.
\end{proof}

\subsection{Special cases} \label{secPosCor}

We next discuss some special cases of our regret bound. First, as in \cite{chawla2020gossiping}, we can prove an explicit bound assuming the honest agent subgraph $G_{\text{hon}}$ is $d$-regular, i.e., $d_{\text{hon}}(i) = d\ \forall\ i \in [n]$.

\begin{cor} \label{corRegular}
Let the assumptions of Theorem \ref{thmProposed} hold and further assume $G_{\text{hon}}$ is $d$-regular with $d \geq 2$. Let $\phi$ denote the conductance of $G_{\text{hon}}$. Then for any honest agent $i \in [n]$ and horizon $T \in \N$,
\begin{align}
R_T^{(i)} & \leq 4 \alpha \log(T) \min \left\{ \frac{2 \eta-1}{\eta-1} \sum_{k=2}^{d_{\text{mal}}(i)+3} \frac{1}{\Delta_k} + \sum_{k=d_{\text{mal}}(i)+4}^{S+d_{\text{mal}}(i)+4} \frac{1}{\Delta_k} , \sum_{k=2}^K \frac{1}{\Delta_k} \right\}  \label{eqRegLogTcor}  \\
& \quad + \tilde{O} \left(  \max \left\{  d_{\text{mal}}(i) / \Delta_2  , ( K / \Delta_2 )^2 , S^{\beta/( \rho_1^2(\beta-1)) }  , ( S / \Delta_2^2 )^{\beta/(\beta-1)} , \bar{d}^{\beta/\rho_1} , n K^2 S , ( \phi \Upsilon )^{-\beta} \right\} \right)  .
\end{align}
\end{cor}

\begin{rem}
This corollary includes the complete graph case studied in \cite{vial2021robust}, where $d_{\text{mal}}(i) = m$, $\phi = \Theta(1)$, and $\Upsilon = \Theta( n/(n+m) )$. In this case, the term \eqref{eqRegLogTcor} matches the corresponding term from \cite{vial2021robust} exactly, i.e., for large $T$, Corollary \ref{corRegular} is a strict generalization. Our additive term scales as
\begin{equation*}
\max \left\{ ( m / \Delta_2 ) , ( K / \Delta_2 )^2 ,S^{\beta/( \rho_1^2(\beta-1)) }  , ( S / \Delta_2^2 )^{\beta/(\beta-1)} , (n+m)^{\beta/\rho_1} , n K^2 S  , ( (n+m) / n )^{\beta} \right\}
\end{equation*}
whereas the additive term from \cite{vial2021robust} scales as $\max \{ ( m / \Delta_2 ) , ( K / \Delta_2 ) , ( S / \Delta_2^2 )^{2 \beta \eta / (\beta-1)} , (n+m)^\beta , n K^2 S \}$. Notice our dependence on the arm gap is $\Delta_2^{-2\beta/(\beta-1)}$, which matches the fully cooperative case \cite{chawla2020gossiping}, whereas the dependence is $\Delta_2^{-2 \beta \eta / (\beta-1) }$ in \cite{vial2021robust}, which is potentially much larger.
\end{rem}

\begin{rem} \label{remCorSimpler}
In the setting of Remark \ref{remThmSimpler}, the corollary's regret bound becomes
\begin{equation*}
R_T^{(i)} \leq \sum_{k=2}^{O(1)} \frac{48 \log T}{ \Delta_k }  + \tilde{O} ( \max \{ (K / \Delta_2)^2 ,  \Delta_2^{-4} , n K^2 , ( \phi \Upsilon )^{-\beta} \} ) .
\end{equation*}
The key difference is the dependence on conductance $\tilde{O} ( \phi^{-\beta} )$, which matches the result from \cite{chawla2020multi}.
\end{rem}

\begin{proof}[Proof sketch]
In light of Theorem \ref{thmProposed}, we only need to show $\E [A_{ \bar{\tau}_{\text{spr}}} ] = \tilde{O} ( ( \phi \Upsilon )^{-\beta} )$. To do so, we let $\underline{\mathcal{I}}_j$ denote the noiseless version of $\bar{\mathcal{I}}_j$ (defined in the same way but with $\Upsilon = 1$) and $\underline{\tau}_{\text{spr}} = \inf \{ j : \underline{\mathcal{I}}_j = [n] \}$. We then construct a coupling between $\bar{\mathcal{I}}_j$ and $\underline{\mathcal{I}}_j$, which ensures that with high probability, $\bar{\tau}_{\text{spr}} \leq j \log (j) / \Upsilon$ whenever $\underline{\tau}_{\text{spr}} \leq j$. Finally, using this coupling and a tail bound for $\underline{\tau}_{\text{spr}}$ from \cite{chawla2020gossiping} (which draws upon the analysis of \cite{chierichetti2010almost}), we derive a tail bound for $\bar{\tau}_{\text{spr}}$. This allows us to show $\E [A_{  \bar{\tau}_{\text{spr}} }] = O ( ( ( \log n )^2 \log( \log(n) / \phi ) / ( \phi \Upsilon ) )^\beta ) = \tilde{O}( ( \phi \Upsilon )^{-\beta} )$, as desired.\footnote{When $\Upsilon=1$, \cite{chawla2020gossiping} shows $\E [A_{\underline{\tau}_{\text{spr}}}] = \E[ A_{\bar{\tau}_{\text{spr}}}] = O ( (\log(n) / \phi)^\beta )$, so our bound generalizes theirs up to $\log$ terms.}
\end{proof}

Finally, we can sharpen the above results for honest agents without malicious neighbors.
\begin{cor} \label{corNoMal}
For $i \in [n]$ with $d_{\text{mal}}(i) = 0$, the terms \eqref{eqRegLogT} and \eqref{eqRegLogTcor} from Theorem \ref{thmProposed} and Corollary \ref{corRegular}, respectively, can (under their respective assumptions) be improved to $4 \alpha \log(T) \sum_{k=2}^{S+2} \Delta_k^{-1}$.
\end{cor}

\begin{rem}
The improved term in Corollary \ref{corNoMal} matches the $\log T$ term from \cite{chawla2020gossiping}, including constants. Thus, the corollary shows that for large $T$, agents who are not directly connected to malicious agents are unaffected by their presence elsewhere in the graph.
\end{rem}

\begin{proof}[Proof sketch]
Recall from the Theorem \ref{thmProposed} proof sketch that  the $\log T$ term arises from regret after phase $\tau_{\text{spr}}$. At any such phase, the best arm is most played for all honest agents (by definition), so when $d_{\text{mal}}(i) = 0$, $i$'s neighbors only recommend this arm. Therefore, $i$'s active sets after $\tau_{\text{spr}}$ are \textit{fixed}; they contain the best arm and $S+1$ suboptimal ones. Thus, $i$ only plays $S+1$ suboptimal arms long-term, so in the worst case incurs the standard UCB regret $4 \alpha \log(T) \sum_{k=2}^{S+2} \Delta_k^{-1}$.
\end{proof}

%% file: experiments.tex
\section{Numerical results} \label{secExp}

Thus far, we have shown the proposed blocking rule adapts to general graphs more gracefully than the existing one, at least in theory. We now illustrate this finding empirically.

\vspace{3pt} \noindent {\bf Experimental setup:} We follow \cite[Section 6]{vial2021robust} except we extend those experiments to $G(n+m,p)$ graphs, i.e., each edge is present with probability $p$. For each $p \in \{1,1/2,1/4\}$ and each of two malicious strategies (to be defined shortly), we conduct $100$ trials of the following:
\begin{itemize}
\item Set $n = 25$ and $m=10$ and generate $G$ as a $G(n+m,p)$ random graph, resampling if necessary until the honest agent subgraph $G_{\text{hon}}$ is connected (see Assumption \ref{assGraph}).
\item Set $K=100$, $\mu_1 = 0.95$, and $\mu_2 = 0.85$, then sample the remaining arm means $\{ \mu_k \}_{k=3}^K$ uniformly from $[0,0.85]$ (so $\Delta_2 = 0.1$). For each $k \in [K]$, set $\nu_k = \text{Bernoulli}(\mu_k)$.
\item Set $S = K/n$ and sample the sticky sets $\{ \hat{S}^{(i)} \}_{i=1}^n$ uniformly from the $S$-sized subsets of $[K]$, resampling if necessary until $1 \in \cup_{i=1}^n \hat{S}^{(i)}$ (see Assumption \ref{assSticky}).
\item Run Algorithm \ref{algGeneral} with the existing (Algorithm \ref{algUpdateBlockOld}) and proposed (Algorithm \ref{algUpdateBlockNew}) blocking rules, along with two baselines: a no communication scheme, where agents ignore the network and run UCB in isolation, and the algorithm from \cite{chawla2020gossiping}, where they do not block. 
\end{itemize}

\vspace{3pt} \noindent {\bf Algorithmic parameters:} We set $\alpha = 4$ and $\beta = \eta = 2$ as in Remarks \ref{remOneBad} and \ref{remThmParam}. For the parameters in the proposed blocking rule, we choose $\kappa_j = j^{1.5}$, and $\theta_j = j - \log j$. While these are different from the parameters specified in our theoretical results (which we found are too conservative in practice), they do satisfy the key properties discussed in Remark \ref{remParameters}.

\vspace{3pt} \noindent {\bf Malicious strategies:} Like \cite{vial2021robust}, we use strategies we call the \textit{naive} and \textit{smart} strategies (they are called \textit{uniform} and \textit{omniscient} in \cite{vial2021robust}). The naive strategy simply recommends a uniformly random suboptimal arm. The smart strategy recommends $R_j^{(i)} = \argmin_{k \in \{2,\ldots,K\} \setminus S_j^{(i)}} T_k^{(i)}(A_j)$, i.e., the least played, inactive, suboptimal arm. Intuitively, this is a more devious strategy which forces $i$ to play $R_j^{(i)}$ often in the next phase (to drive down its upper confidence bound). Consequently, $i$ may play it most and discard a better arm in favor of it (see Lines \ref{algUjPlus1}-\ref{algNewActive} of Algorithm \ref{algGeneral}).

\vspace{3pt} \noindent {\bf Results:} In Figure \ref{figExpSynthetic}, we plot the average and standard deviation (across trials) of the per-agent regret $\sum_{i=1}^n R_T^{(i)} / n$. For the naive strategy, the existing blocking rule eventually becomes worse than the no blocking baseline as $p$ decreases. More strikingly, it even becomes worse than the no communication baseline for the smart strategy. In other words, honest agents would have been better off \textit{ignoring the network} and simply running UCB on their own. As in Section \ref{secExisting}, this is because accidental blocking causes the best arm to spread very slowly. Additionally, the standard deviation becomes much higher than all other algorithms, suggesting that regret is significantly worse in some trials. In contrast, the proposed blocking rule \textit{improves} as $p$ decreases, because it is mild enough to spread the best arm at all values of $p$, and for smaller $p$, honest agents have fewer malicious neighbors (on average). We also observe that the proposed rule outperforms both baselines uniformly across $p$. Additionally, it improves over the existing rule more dramatically for the smart strategy, i.e., when the honest agents face a more sophisticated adversary. Finally, it is worth acknowledging the existing rule is better when $p=1$ -- although not in a statistically significant sense for the smart strategy -- because it does spread the best arm quickly on the complete graph (as shown in \cite{vial2021robust}), and thereafter more aggressively blocks malicious agents.

\begin{figure}[t]
\centering
\includegraphics[width=\textwidth]{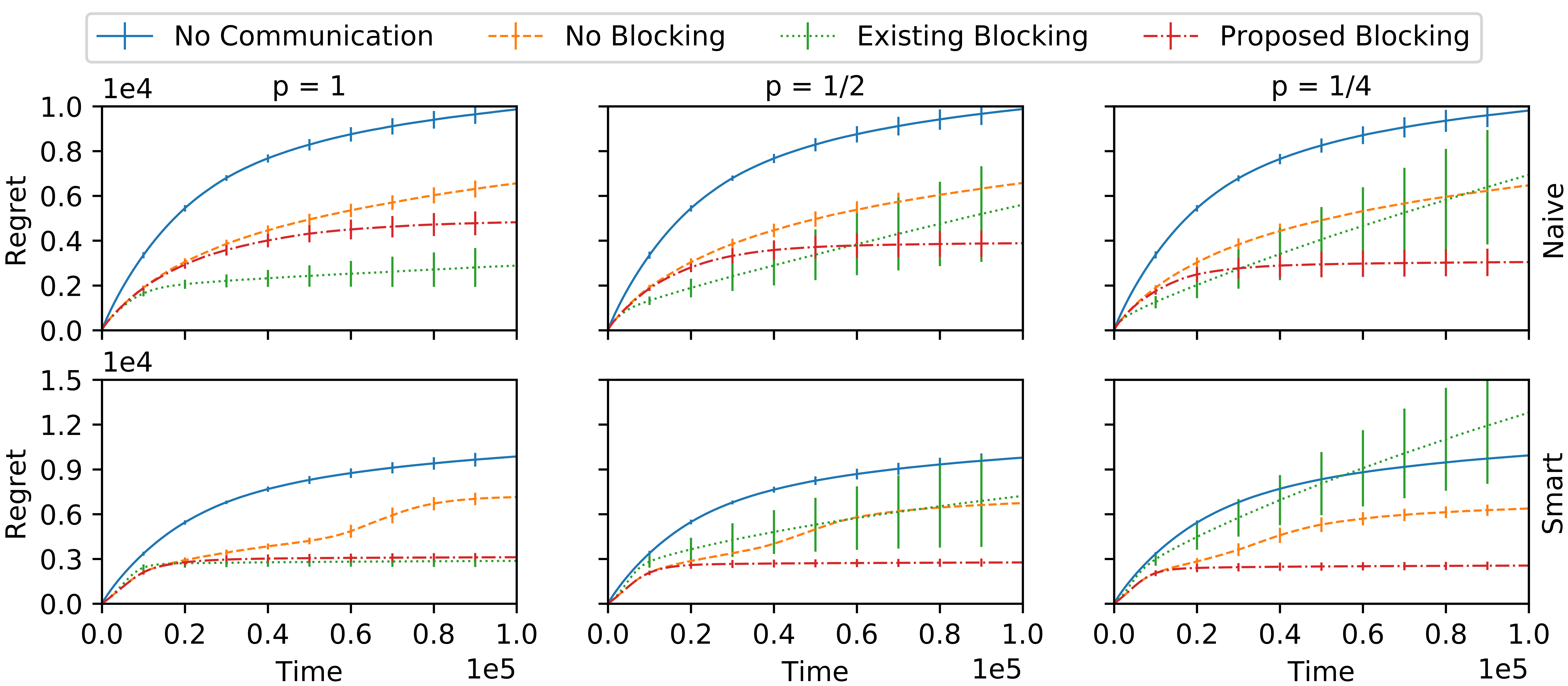}
\caption{Empirical results for synthetic data. Rows of subfigures correspond to the malicious strategy, while columns correspond to the edge probability $p$ for the $G(n+m,p)$ random graph.} \label{figExpSynthetic}
\Description{Empirical results for synthetic data. Rows of subfigures correspond to the malicious strategy, while columns correspond to the edge probability $p$ for the $G(n+m,p)$ random graph.} 
\end{figure}

\vspace{3pt} \noindent {\bf Other results:} As in \cite{vial2021robust}, we reran the simulations using arm means derived from the MovieLens dataset \cite{harper2015movielens}. We also experimented with new variants of the smart and naive strategies, where the malicious agents follow these strategies if the best arm is active (in hopes of forcing honest agents to discard it) and recommend the \textit{second} best arm otherwise. Intuitively, these variants differ in that malicious agents recommend good arms (i.e., the second best) more frequently, while still never revealing the best arm (the only one that leads to logarithmic regret). For all experiments, the key message -- that the proposed blocking rule adapts to varying graph structures more gracefully than the existing one -- is consistent. See Appendix \ref{appExp} for details.

%% file: analysis.tex
\section{Gossip despite blocking} \label{secAnalysis}

As discussed above, the main analytical contribution of this work is proving that the best arm spreads in a gossip fashion, despite accidental blocking. In this (technical) section, we provide a detailed sketch of this proof. We begin with a high-level outline. The key is to show that honest agents eventually stop blocking each other. This argument (roughly) proceeds as follows:
\begin{itemize}
\item {\bf Step 1:} First, we show that honest agents learn the arm statistics in a certain sense. More specifically, we provide a tail bound for a random phase $\tau_{\text{arm}}$ such that for all phases $j \geq \tau_{\text{arm}}$ (1) each honest agent's most played arm in phase $j$ is close to its true best active arm and (2) any active arm close to the true best one is played at least $\kappa_j$ times by the end of phase $j$.
\item {\bf Step 2:} Next, we show that honest agents communicate with their neighbors frequently. In particular, we establish a tail bound for another random phase $\tau_{\text{com}}$ such that for any $j \geq \tau_{\text{com}}$, each honest agent contacts all of its honest neighbors at least once between $\theta_j$ and $j$.
\item {\bf Step 3:} Finally, we use the above to show that eventually, no blocking occurs amongst honest agents. The basic idea is as follows. Consider a phase $j$, an honest agent $i$, and a neighbor $i'$ of $i$. Then if $i$ has had the same best arm estimate $k$ since phase $\theta_j$ -- i.e., if the second blocking criterion in \eqref{eqNewBlock} holds -- $i'$ would have contacted $i$ at some phase $j' \in \{\theta_j,\ldots,j\}$ (by step 2) and received arm $k$. Between phases $j'$ and $j$, the most played arm for $i'$ cannot get significantly worse (by step 1). Thus, if $i$ asks $i'$ for a recommendation at $j$, $i'$ will respond with an arm whose mean is close to $\mu_k$, which $i$ will play at least $\kappa_j$ times (by step 1). Hence, the first criterion in \eqref{eqNewBlock} fails, i.e., the two cannot simultaneously hold.
\end{itemize}
In the next three sub-sections, we discuss these three steps. Then in Section \ref{secCoupling}, we describe how, once accidental blocking stops, the arm spreading process can be coupled to the noisy rumor process from Definition \ref{defnRumor}. Finally, in Section \ref{secSpread}, we discuss how to combine all of these steps to bound the term $\E [A_{\tau_{\text{spr}}}]$ from the Theorem \ref{thmProposed} proof sketch.

\subsection{Learning the arm statistics} \label{secLearnArm}

Recall we assume $\mu_1 \geq \cdots \geq \mu_K$, so for any $W \subset [K]$, $\min W$ is the best arm in $W$, i.e., $\mu_{\min W} = \max_{w \in W} \mu_w$. Therefore, for any $\delta \in (0,1)$, $G_\delta(W) \triangleq \{ w \in W : \mu_w \geq \mu_{ \min W } - \delta \}$ is the subset of arms at least $\delta$-close to the best one. For each honest agent $i \in [n]$ and phase $j \in \N$, define
\begin{gather}
\Xi_{j,1}^{(i)}   = \left\{ B_j^{(i)} \notin G_{\delta_{j,1}}( S_j^{(i)}) \right\} , \quad \Xi_{j,2}^{(i)} = \left\{ \min_{ w \in G_{\delta_{j,2}}( S_j^{(i)}) } T_w^{(i)} (A_j) \leq \kappa_j \right\}  , \quad \Xi_j^{(i)} = \Xi_{j,1}^{(i)} \cup \Xi_{j,2}^{(i)} .
\end{gather} 
where $\delta_{j,1} , \delta_{j,2} \in (0,1)$ will be chosen shortly. Finally, define the random phase
\begin{equation}\label{eqTauArmDefn}
\tau_{\text{arm}} = \inf \{ j \in \N : \ind ( \Xi_{j'}^{(i)} ) = 0\ \forall\ i \in [n] , j' \in \{j,j+1,\ldots\} \} .
\end{equation}
In words, $\tau_{\text{arm}}$ is the earliest phase such that, at all phases $j$ thereafter, (1) the most played arms are $\delta_{j,1}$-close to best active arms and (2) all arms $\delta_{j,2}$-close to the best are played at least $\kappa_j$ times.

As discussed above, Step 1 involves a tail bound for $\tau_{\text{arm}}$. The analysis is based on \cite[Theorem 2]{bubeck2011pure}, which includes a tail bound showing that the most played arm is $\delta$-close to the best, provided that $1 / \delta^2$ samples have been collected from each of the $\delta$-far arms. In our case, phase $j$ lasts $A_j - A_{j-1} = \Theta ( j^{\beta-1} )$ time steps, so each of $S+2$ active arms is played $\Theta ( j^{\beta-1} / S)$ times on average. Hence, we can show the most played arm within the phase is $\delta_{j,1}$-close to the best if we choose $\delta_{j,1} = \Theta( \sqrt{ S / j^{\beta-1}} )$, which allows us to bound $\P(\Xi_{j,1}^{(i)})$. Analogously, we choose $\delta_{j,2} = \Theta ( 1 / \sqrt{ \kappa_j } )$ and show that $\delta_{j,2}$-close arms must be played $1 / \delta_{j,2}^2  = \tilde{\Theta}(\kappa_j)$ times before they are distinguished as such, which allows us to bound $\P(\Xi_{j,2}^{(i)})$. Taken together, we can prove a tail bound for $\tau_{\text{arm}}$ (Lemma \ref{lemLearnArm} in Appendix \ref{appLearnArm}) with these choices $\delta_{j,1} = \Theta( \sqrt{ S / j^{\beta-1}} )$ and $\delta_{j,2}  = \tilde{\Theta}(1 / \sqrt{\kappa_j})$.

\subsection{Communicating frequently} \label{secCommFreq}

Next, for any $i,i' \in [n]$ such that $(i,i') \in E_{\text{hon}}$, let $\Xi_j^{(i \rightarrow i')} = \cap_{j'=\floor{\theta_j} }^{j-2} \{ H_{j'}^{(i')} \neq i \}$ denote the event that $i$ did not send a recommendation to $i'$ between phases $\floor{\theta_j}$ and $j-2$. Also define
\begin{equation}\label{eqTauComDefn}
\tau_{\text{com}} = \inf \{ j \in \N : \ind ( \cup_{ i \rightarrow i' \in E_{\text{hon}} } \Xi_{j'}^{(i \rightarrow i')}  ) = 0\ \forall\ j' \in \{ j , j+1 , \ldots \} \} .
\end{equation}
Here we abuse notation slightly; the union is over all (undirected) edges in $E_{\text{hon}}$ but viewed as pairs of directed edges. Hence, at all phases $j \geq \tau_{\text{com}}$, each honest agent $i'$ receives a recommendation from each of its honest neighbors $i$ at some phase $j'$ between $\theta_j$ and $j-2$.

Step 2 involves the tail bound for $\tau_{\text{com}}$ that was mentioned above (see Lemma \ref{lemTauCom} in Appendix \ref{appCommFreq}). The proof amounts to bounding the probability of $\Xi_j^{(i \rightarrow i')}$. Recall this event says $i'$ did not contact $i$ for a recommendation at any phase $j' \in \{ {\theta_j} , \ldots , j-2 \}$. Clearly, this means $i'$ did not \textit{block} $i$ at any such phase. Hence, in the worst case, $i'$ blocked $i$ just before $\theta_j$, in which case $i$ was un-blocked at $\theta_j^\eta = (j/3)^{\rho_1 \eta} \leq j/3$, where the inequality holds by assumption in Theorem \ref{thmProposed}. Hence, $\Xi_j^{(i \rightarrow i')}$ implies $i'$ was not blocking $i$ between phases $j/3$ and $j-2$, so each of the $\Theta(j)$ neighbors that $i'$ contacted in these phases was sampled uniformly from a set containing $i$, yet $i$ was never sampled. The probability of this decays exponentially in $j$, which yields an exponential tail for $\tau_{\text{com}}$.

\subsection{Avoiding accidental blocking} \label{secNoBlock}

Next, we show honest agents eventually stop blocking each other. Toward this end, we first note
\begin{equation} \label{eqMuDecayOneMain}
\forall\ i \in [n], \quad \forall\ j \geq \tau_{\text{arm}} , \quad \mu_{\min S_j^{(i)}}  \leq \mu_{ B_j^{(i)} } + \delta_{j,1} \leq \mu_{\min S_{j+1}^{(i)}} + \delta_{j,1}
\end{equation}
where the first inequality uses the definition of $\tau_{\text{arm}}$, and the second holds because $\min S_{j+1}^{(i)} \in \argmax_{k \in S_{j+1}^{(i)}} \mu_k$ and $B_j^{(i)} \in S_{j+1}^{(i)}$ in Algorithm \ref{algGeneral} (see Claim \ref{clmMuDecayOne} in Appendix \ref{appNoBlock} for details). In words, \eqref{eqMuDecayOneMain} says the best active arm can decay by at most $\delta_{j,1}$ at phase $j$. Applying iteratively and since there are $K$ arms total, we then show (see Claim \ref{clmMuDecayMulti} in Appendix \ref{appNoBlock})
\begin{equation}
\forall\ i \in [n], \quad \forall\ j' \geq j \geq \tau_{\text{arm}} , \quad \mu_{\min S_{j'}^{(i)}} \geq \mu_{ \min S_{j}^{(i)} } - (K-1) \sup_{j'' \in \{ j, \ldots, j'\} } \delta_{j'',1} .
\end{equation}
Combining the previous two inequalities, we conclude (see Corollary \ref{corMuDecay} in Appendix \ref{appNoBlock})
\begin{equation} \label{eqCorMuDecay}
\forall\ i \in [n], \quad \forall\ j' \geq j \geq \tau_{\text{arm}} , \quad \mu_{B_{j'}^{(i)}} \geq \mu_{ \min S_{j}^{(i)} } - K \sup_{j'' \in \{ j, \ldots, j'\} } \delta_{j'',1} .
\end{equation}

Now the key part of Step 3 is to use \eqref{eqCorMuDecay} to show (see Claim \ref{clmNoEnterIfComm} in Appendix \ref{appNoBlock})
\begin{equation} \label{eqNoNewBlocking}
\forall\ j \in \N\ s.t.\ \theta_j \geq \tau_{\text{arm}} , \quad \forall\ i,i' \in [n]\ s.t.\ i \in \{ H_{j'}^{(i')} \}_{j' = \theta_j}^{j-2} , \quad i' \notin P_j^{(i)} \setminus P_{j-1}^{(i)} .
\end{equation}
In words, this result says that if $j$ is sufficiently large, and if $i$ has sent a recommendation to $i'$ since phase $\theta_j$, then $i$ will not block $i'$ at phase $j$. The proof is by contraction: if instead $i$ blocks $i'$ at $j$, then by Algorithm \ref{algUpdateBlockNew}, $i$ has not changed its best arm estimate $k$ since phase $\theta_j$, so it would have recommended $k$ to $i'$ at some phase $j' \geq \theta_j$. Therefore, $\mu_{ \min S_{j'}^{(i')} } \geq \mu_k$. Additionally, since $j' \geq \theta_j = \Omega ( j^{\rho_1} )$, we know that for any $j'' \geq j'$, the choice of $\delta_{j'',1}$ in Section \ref{secLearnArm} guarantees that
\begin{equation*}
K \delta_{j'',1} \leq O ( K \delta_{j',1}  ) = \tilde{O} \Big( \sqrt{K^2 S / (j')^{\beta-1} }\Big) \leq \tilde{O} \Big( \sqrt{K^2 S / j^{\rho_1(\beta-1)} } \Big) .
\end{equation*}
Combining these observations and using \eqref{eqCorMuDecay} (with $j'$ and $j$ replaced by $j-1$ and $j'$), we then show
\begin{equation} \label{eqNewAccBlock}
\mu_{B_{j-1}^{(i')}} \geq \mu_{ \min S_{j'}^{(i')} } - K \sup_{j'' \geq j'} \delta_{j'',1} \geq \mu_k - \tilde{O} \Big( \sqrt{K^2 S / j^{\rho_1(\beta-1)} } \Big) .
\end{equation}
On the other hand, $i$ blocking $i'$ at phase $j$ means $i$ plays the recommended arm $B_{j-1}^{(i')}$ fewer than $\kappa_j$ times by the end of phase $j$. Since $j \geq \theta_j \geq \tau_{\text{arm}}$, this implies (by definition of $\tau_{\text{arm}}$) that $\mu_k > \mu_{B_{j-1}^{(i')}} + \delta_{j,2}$, where $\delta_{j,2} = \tilde{\Theta}( 1 / \sqrt{\kappa_j}  )$ as in Section \ref{secLearnArm}. Combined with \eqref{eqNewAccBlock} and the choice $\kappa_j = j^{\rho_2} / ( K^2 S )$ from Theorem \ref{thmProposed}, we conclude $j^{\rho_1(\beta-1)} \leq \tilde{O}(j^{\rho_2})$. This contradicts the assumption $\rho_2 < \rho_1(\beta-1)$ in Theorem \ref{thmProposed}, which completes the proof of \eqref{eqNoNewBlocking}.

Finally, we use \eqref{eqNoNewBlocking} to show honest agents eventually stop blocking each other entirely, i.e.,
\begin{equation} \label{eqNoBlocking}
\forall\ j \in \N\ s.t.\ \theta_j \geq \tau_{\text{com}} , \theta_{\theta_j} \geq \tau_{\text{arm}}, \quad P_j^{(i)} \cap [n] = \emptyset\ \forall\ i \in [n] 
\end{equation}
(see Lemma \ref{lemNoBlock} in Appendix \ref{appNoBlock}). Intuitively, \eqref{eqNoBlocking} holds because after new blocking stops \eqref{eqNoNewBlocking}, old blocking will eventually ``wear off''. The proof is again by contradiction: if $i$ is blocking some honest $i'$ at phase $j$, the blocking must have started at some $j' \geq j^{1/\eta}$ (else, it ends by $(j')^\eta < j$). Thus, by assumption $j^{1/\eta} \geq j^{\rho_1} \geq { \theta_j } \geq \tau_{\text{com}}$, $i$ blocked $i'$ at phase $j' \geq \tau_{\text{com}}$. But by definition of $\tau_{\text{com}}$, $i'$ would have contacted $i$ at some phase $j'' \in \{ \theta_{ j'} , \ldots , j' \}$. Applying \eqref{eqNoNewBlocking} (at phase $j'$; note that by the above inequalities, $\theta_{j'} \geq \theta_{j^{1/\eta}} \geq \theta_{\theta_j} \geq \tau_{\text{arm}}$, as required by \eqref{eqNoNewBlocking}), we obtain a contradiction.

\subsection{Coupling with noisy rumor process} \label{secCoupling}

To begin, we define an equivalent way to sample $H_j^{(i)}$ in Algorithm \ref{algGetArm}.\footnote{Claim \ref{clmCoupling} in Appendix \ref{appCoupling} verifies this equivalency (the proof is a straightforward application of the law of total probability).} This equivalent method will allow us to couple the arm spreading and noisy rumor processes through a set of primitive random variables. In particular, for each honest agent $i \in [n]$, let $\{ \upsilon_j^{(i)} \}_{j=1}^\infty$ and $\{ \bar{H}_j^{(i)} \}_{j=1}^\infty$ be i.i.d.\ sequences drawn uniformly from $[0,1]$ and $N_{\text{hon}}(i)$. Then choose $H_j^{(i)}$ according to two cases:
\begin{itemize}
\item If $P_j^{(i)} \cap [n] = \emptyset$, let $Y_j^{(i)} = \ind ( \upsilon_j^{(i)} \leq d_{\text{hon}}(i) / | N(i) \setminus P_j^{(i)} | )$ and consider two sub-cases. First, if $Y_j^{(i)} = 1$, set $H_j^{(i)} = \bar{H}_j^{(i)}$. Second, if $Y_j^{(i)} = 0$, sample $H_j^{(i)}$ from $N_{\text{mal}}(i) \setminus P_j^{(i)}$ uniformly.
\item If $P_j^{(i)} \cap [n] \neq \emptyset$, sample $H_j^{(i)}$ from $N(i) \setminus P_j^{(i)}$ uniformly.
\end{itemize}

Next, we observe that since $\delta_{j,1} \rightarrow 0$ as $j \rightarrow \infty$ by the choice of $\delta_{j,1}$ in Section \ref{secLearnArm} and $\Delta_2 > 0$ by Assumption \ref{assReward}, we have $\delta_{j,1} < \Delta_2$ for large enough $j$. Paired with the definition of $\tau_{\text{arm}}$, this allows us to show that for all large $j$ and $i \in [n]$ with $1 \in S_j^{(i)}$ (i.e., with the best arm active), $B_j^{(i)} = 1$ (i.e., the best arm is played most). See Claim \ref{clmInformedRecBest} in Appendix \ref{appCoupling} for the formal statement.

Finally, we observe that by \eqref{eqNoBlocking}, only the first case of the above sampling strategy occurs for large $j$. Moreover, in this case, $Y_j^{(i)}$ is Bernoulli with parameter
\begin{equation*}
d_{\text{hon}}(i) / | N(i) \setminus P_j^{(i)} | \geq d_{\text{hon}}(i) / | N(i) | \triangleq d_{\text{hon}}(i) / d(i) \geq \Upsilon ,
\end{equation*}
where the second inequality holds by Definition \ref{defnRumor}. Hence, the probability that $Y_j^{(i)} = 1$, and thus the probability that $i$ contacts the random honest neighbor $\bar{H}_j^{(i)}$ in the above sampling strategy, dominates the probability that $i$ contacts $\bar{H}_j^{(i)}$ in the noisy rumor process of Definition \ref{defnRumor}. Additionally, by the previous paragraph, agents with the best arm active will recommend it (for large enough $j$). Taken together, we can show that the probability of receiving the best arm in the arm spreading process dominates the probability of being informed of the rumor in the noisy rumor process. This allows us to prove a tail bound for $\tau_{\text{spr}}$ in terms of a tail bound for the random phase $\bar{\tau}_{\text{spr}}$ from Definition \ref{defnRumor}, on the event that the tails of $\tau_{\text{arm}}$ and $\tau_{\text{com}}$ are sufficiently small (in the sense of \eqref{eqNoBlocking}; see Lemma \ref{lemTXfail} in Appendix \ref{appCoupling} for details).

\subsection{Spreading the best arm} \label{secSpread}

In summary, we prove tail bounds for $\tau_{\text{arm}}$ and $\tau_{\text{com}}$ (Sections \ref{secLearnArm} and \ref{secCommFreq}) and show the tails of $\tau_{\text{spr}}$ are controlled by those of  $\bar{\tau}_{\text{spr}}$, provided the tails of $\tau_{\text{arm}}$ and $\tau_{\text{com}}$ are not too heavy (Sections \ref{secNoBlock} and \ref{secCoupling}). Combining and summing tails allows us to bound $\E [ A_{ {\tau}_{\text{spr}}} ]$ in terms of $C_\star$ (which accounts for the tails of $\tau_{\text{arm}}$ and $\tau_{\text{com}}$) and $\E [ A_{ \bar{\tau}_{\text{spr}}} ]$ (which accounts for the tail of $\bar{\tau}_{\text{spr}}$), as mentioned in the Theorem \ref{thmProposed} proof sketch. See Theorem \ref{thmTail} and Corollary \ref{corEarly} in Appendix \ref{appSpread} for details.

%% file: otherExp.tex
\section{Additional experiments} \label{appExp}

\begin{figure}
\centering
\includegraphics[width=\textwidth]{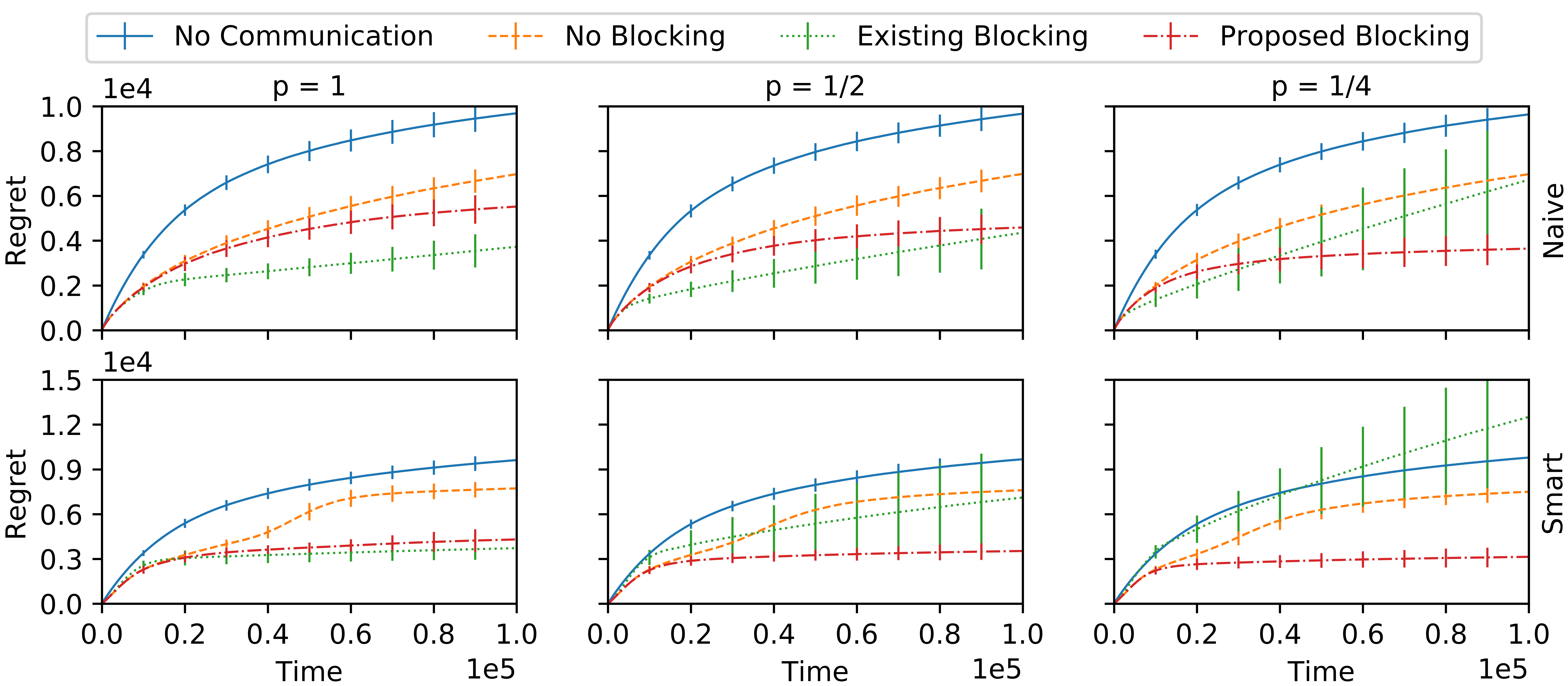}
\caption{Empirical results for real data. Rows of subfigures correspond to the malicious strategy, while columns correspond to the edge probability $p$ for the $G(n+m,p)$ random graph.} \label{figExpReal}
\Description{Empirical results for real data. Rows of subfigures correspond to the malicious strategy, while columns correspond to the edge probability $p$ for the $G(n+m,p)$ random graph.} 
\end{figure}

\begin{figure}
\centering
\includegraphics[width=0.86\textwidth]{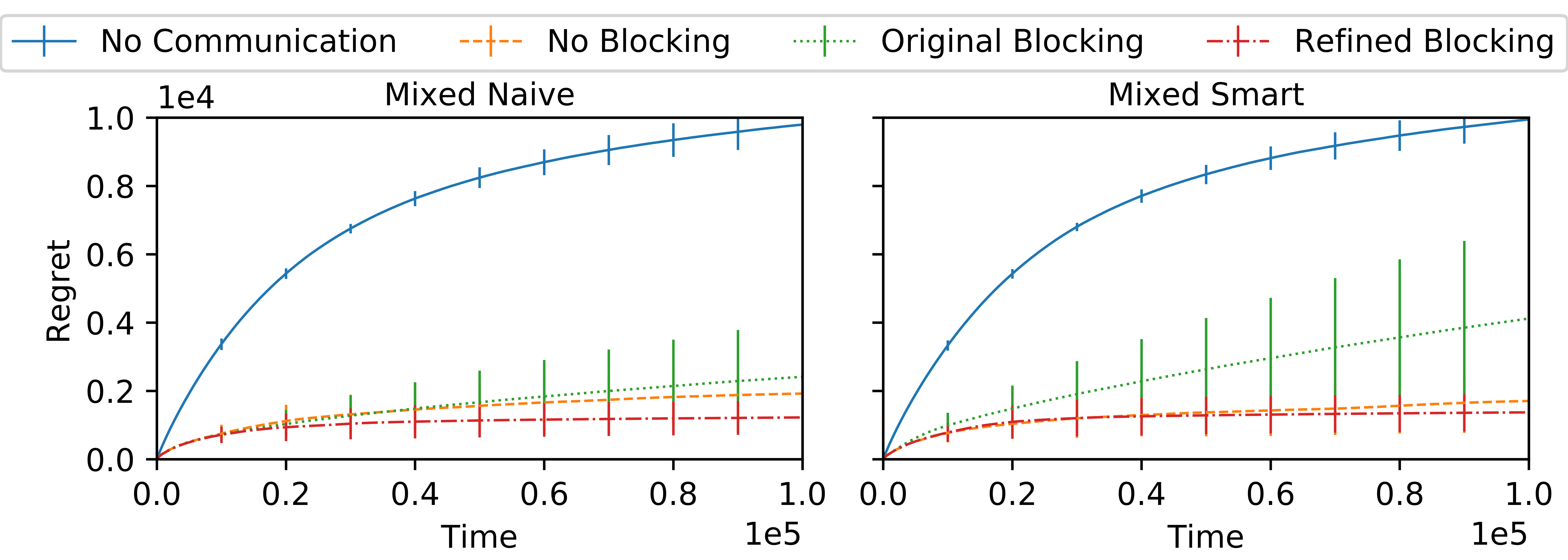}
\caption{Empirical results for synthetic data with $p=1/2$ and mixed malicious strategies.} \label{figExpMixed}
\Description{Empirical results for synthetic data with $p=1/2$ and mixed malicious strategies.}
\end{figure}

As mentioned in Section \ref{secExp}, we also considered arm means derived from real data. The setup was the same as for the synthetic means, except for two changes (as in \cite{vial2021robust}): we choose $m=15$ instead of $m=10$, and we sample $\{ \mu_k \}_{k=1}^K$ uniformly from a set of arm means derived from MovieLens \cite{harper2015movielens} user film ratings via matrix completion; see \cite[Section 6.2]{vial2021robust} for details. The results (Figure \ref{figExpReal}) are qualitatively similar to the synthetic case. 

Finally, we repeated the synthetic data experiments from Section \ref{secExp} with the intermediate $G(n+m,p)$ graph parameter $p=1/2$ and two new malicious strategies called \textit{mixed naive} and \textit{mixed smart}. As discussed in Section \ref{secExp}, these approaches use a ``mixed report'' where the malicious agents more frequently recommend good arms -- namely, the second best when the best is inactive and the naive or smart recommendation otherwise. Results are shown in Figure \ref{figExpMixed}. They again reinforce the key message that the proposed rule adapts more gracefully to networks beyond the complete graph -- in this case, our blocking rule has less than half of the regret of the existing one at the horizon $T$. Additionally, we observe that the no blocking algorithm from \cite{chawla2020gossiping} has much lower regret in Figure \ref{figExpMixed} than Figure \ref{figExpSynthetic}, though still higher than our proposed blocking algorithm. This suggests that our algorithm remains superior even for ``nicer'' malicious strategies under which blocking is less necessary (in the sense that \cite{chawla2020gossiping} has lower regret in Figure \ref{figExpMixed} than Figure \ref{figExpSynthetic}).

%% file: proofExisting.tex
\section{Proof of Theorem \ref{thmExisting}} \label{appProofExisting}

We first observe that since $h \leq 2^{h-1}\ \forall\ h \in \N$, we can lower bound the arm gap as follows:
\begin{equation*}
\Delta_2 = \mu_1 - \mu_2 = 1 - \frac{13}{15} - \sum_{h=1}^{(n/2)-2} \left( \frac{1}{16} \right)^{2^{h-1}} > 1 -  \frac{13}{15} - \sum_{h=1}^\infty \left( \frac{1}{16} \right)^h = \frac{1}{15} .
\end{equation*}

Next, we show that for phases $j \geq J_l$, agents aware of arms at least as good as $1-l+n/2$ will (1) play such arms most often in phase $j$ and (2) have such arms active thereafter. The proof is basically a noiseless version of a known bandit argument, specialized to the setting of Section \ref{secExisting}.
\begin{lem} \label{lemBestArmId}
Under the assumptions of Theorem \ref{thmExisting}, for any $l \in [n/2]$, $j \geq J_l$, and $i \in [n]$ such that $\min S_j^{(i)} \leq 1 - l + n/2$, we have $B_j^{(i)} \leq 1 - l + n/2$ and $\min S_{j'}^{(i)} \leq 1 - l + n/2\ \forall\ j' \geq j$.
\end{lem}
\begin{proof}
First, we prove by contradiction that $B_j^{(i)} \leq 1-l+n/2$: suppose instead that $B_j^{(i)} \geq 2 - l + n/2$. Let $k_1 = \min S_j^{(i)}$ and $k_2 = B_j^{(i)}$. Then $T_{k_2}^{(i)}(A_j) - T_{k_2}^{(i)}(A_{j-1}) \geq (A_j-A_{j-1})/3$; otherwise, since $|S_j^{(i)}|=S+2=3$ by assumption and $k_2$ is the most played arm in phase $j$, we obtain 
\begin{equation*}
\sum_{k \in S_j^{(i)}} ( T_k^{(i)}(A_j) - T_k^{(i)}(A_{j-1}) ) \leq 3 ( T_{k_2}^{(i)}(A_j) - T_{k_2}^{(i)}(A_{j-1}) ) < A_j - A_{j-1} ,
\end{equation*}
which is a contradiction. Furthermore, there clearly exists $t \in \{1+A_{j-1},\ldots,A_j\}$ such that
\begin{equation*}
T_{k_2}^{(i)}(t-1) - T_{k_2}^{(i)}(A_{j-1}) = T_{k_2}^{(i)}(A_j) - T_{k_2}^{(i)}(A_{j-1}) - 1 , \quad I_t^{(i)} = k_2 .
\end{equation*}
Combining these observations, and since $T_{k_2}^{(i)}(A_{j-1})\geq0$, we obtain that
\begin{equation*}
T_{k_2}^{(i)}(t-1) \geq T_{k_2}^{(i)}(A_j) - T_{k_2}^{(i)}(A_{j-1}) - 1 \geq \frac{A_j-A_{j-1}}{3} - 1 , \quad I_t^{(i)} = k_2 .
\end{equation*}
By the UCB policy and since $\alpha = 4$ by assumption, the previous expression implies
\begin{equation} \label{eqBestArmIdUcb}
\mu_{k_1} < \mu_{k_1} + \sqrt{ \frac{4  \log t}{ T_{k_1}^{(i)}(t-1) } } \leq \mu_{k_2} + \sqrt{ \frac{4  \log t}{ T_{k_2}^{(i)}(t-1) } } \leq \mu_{k_2} + \sqrt{ \frac{4  \log t}{ \frac{A_j-A_{j-1}}{3} - 1 } } .
\end{equation}
Since $A_j = j^2$ by the assumption $\beta = 2$ and $t \leq A_j$, we also know
\begin{equation} \label{eqBestArmIdH}
\frac{4  \log t}{ \frac{A_j-A_{j-1}}{3} - 1 } \leq \frac{ 8 \log j }{ \frac{2j-1}{3} - 1 } = \frac{12 \log j}{ j-2 } = h(j) ,
\end{equation}
where we define $h(j') = 12 \log(j') / (j'-2)\ \forall\ j' > 2$. Note this function decreases on $[3,\infty)$, since
\begin{equation*}
h'(j') = \frac{ 12 ( j' - 2 - j' \log j' ) }{ j' (j'-2)^2 } \leq \frac{ 12 ( j' - 2 - j' \log 3) }{ j' (j'-2)^2 } < \frac{-24}{j'(j'-2)^2} < 0\ \forall\ j' \geq 3, 
\end{equation*}
where the second inequality is $e < 3$. Thus, since $j \geq J_l \geq J_1 \geq 3$, we know $h(j) \leq h(J_l)$. Combined with \eqref{eqBestArmIdUcb} and \eqref{eqBestArmIdH}, we obtain $\mu_{k_1} < \mu_{k_2} + \sqrt{ h(J_l) }$. Finally, recall $k_2 \geq 2 - l + n/2$ and $k_1 \leq 1 - l + n/2$, so $\mu_{k_1} - \mu_{k_2} \geq \mu_{1-l+n/2} - \mu_{2-l+n/2} > 0$. Combined with $\mu_{k_1} < \mu_{k_2} + \sqrt{ h(J_l) }$, we conclude
\begin{equation} \label{eqBestArmIdCont}
( \mu_{1-l+n/2} - \mu_{2-l+n/2} )^2 < h(J_l) = 12 \log(J_l) / ( J_l - 2 ) .
\end{equation}
We now show that in each of three cases, \eqref{eqBestArmIdCont} yields a contradiction.
\begin{itemize}
\item $l=1$: By definition, the left side of \eqref{eqBestArmIdCont} is $( \mu_{n/2} - \mu_{1+n/2} )^2 = (13/15)^2$ and the right side is $12 \log (2^8) / ( 2^8 - 2 )$, and one can verify that $12 \log (2^8) / ( 2^8 - 2 ) \leq (13/15)^2$.

\item $1 < l < n/2$: Here $1-l+n/2 > 1$ and $2-l+n/2 < 1+n/2$, i.e., both arms are mediocre. By definition, we thus have $( \mu_{1-l+n/2} - \mu_{2-l+n/2} )^2 = 2^{ - 2^{l+1} }$, so to obtain a contradiction to \eqref{eqBestArmIdCont}, it suffices to show $h(J_l) \leq 2^{-2^{l+1}}$. We show by induction that this holds for all $l \in \{2,3,\ldots\}$.

For $l=2$, note $J_2 = ( J_1 + 2 )^2 > J_1^2 + 2 = 2^{16} + 2$ (so $J_2-2>2^{16}$) and $J_2 = J_1^2 + 4 J_1 + 4 = 2^{16} + 2^{10} + 2^2 < 2^{17}$ (so $\log J_2 < 17 \log 2 < 17$). Thus, $h(J_2) < 12 \cdot 17 / 2^{16} = 204 / 2^{16} < 2^{-8}$. 

Now assume $h(J_l) < 2^{-2^{l+1}}$ for some $l \geq 2$; we aim to show $h(J_{l+1}) < 2^{-2^{l+2}}$. Since $J_l \geq 2$, we have $J_{l+1} \leq ( 2 J_l )^2 \leq J_l^4$; we also know $J_{l+1} - 2 = J_l^2 + 4 J_l + 2 >  J_l ( J_l-2 )$. Thus, we obtain
\begin{equation}
h(J_{l+1}) < \frac{ 12 \log J_l^4 }{ J_l(J_l-2)} = \frac{4}{J_l} \cdot h(J_l) < \frac{4}{J_l} \cdot 2^{-2^{l+1}} < \frac{4}{ J_1^{2^{l-1}} } \cdot 2^{-2^{l+1}} = 2^{ 2 - 2^{l+2} - 2^{l+1} } ,
\end{equation}
where the inequalities follow from the previous paragraph, the inductive hypothesis, and \eqref{eqJlDoubExp} from Section \ref{secBadInstance}, respectively. Since $2 < 2^{l+1}$, this completes the proof.

\item $l=n/2$: Recall that in the previous case, we showed $h(J_l) < 2^{-2^{l+1}}$ for any $l \in \{2,3,\ldots\}$. Therefore, $h(J_{n/2}) \leq 2^{-2^{(n/2)+1}} \leq 2^{-8}$ by assumption $n \in \{4,6,\ldots\}$. Since $(\mu_{1-l+n/2} - \mu_{2-l+n/2} )^2 = \Delta_2^2 = (1/15)^2 > (1/16)^2 = 2^{-8}$ in this case, we obtain a contradiction to \eqref{eqBestArmIdCont}.
\end{itemize}

Thus, we have established the first part of the lemma ($B_j^{(i)} \leq 1-l+n/2$). To show $\min S_{j'}^{(i)} \leq 1-l+n/2$, we suppose instead that $\min S_{j'}^{(i)} > 1-l+n/2$ for some $j' \geq j$. Then $j^\dagger = \min \{ j' \geq j : \min S_{j'}^{(i)} > 1-l+n/2 \}$ is well-defined. If $j^\dagger = j$, then $\min S_j^{(i)} > 1-l+n/2$, which violates the assumption of the lemma, so we assume $j^\dagger > j$. In this case, we know $\min S_{j^\dagger-1}^{(i)} \leq 1-l+n/2$ (since $j^\dagger$ is minimal) and $B_{j^\dagger-1}^{(i)} > 1-l+n/2$ (else, because $B_{j^\dagger-1}^{(i)} \in S_{j^\dagger}^{(i)}$, we would have $\min S_{j^\dagger}^{(i)} \leq B_{j^\dagger-1}^{(i)} \leq 1-l+n/2$). But since $j^\dagger - 1 \geq J_l$ (by assumption $j^\dagger > j$ and $j \geq J_l$), this contradicts the first part of the lemma.
\end{proof}

Next, for each $l \in [n/2]$, we define the event
\begin{equation*}
\mathcal{E}_l = \{ l + n/2 \notin P_{J_l}^{(l+1+n/2)} \} \cap \cap_{j=1}^{J_l}  \cap_{i = l+n/2}^n \{  \min S_j^{(i)} > 1 - l + n/2 , n+1 \notin P_j^{(i)} \} ,
\end{equation*}
where $\{ n \notin P_{J_{n/2}}^{(n+1)} \} = \Omega$ by convention (so $\mathcal{E}_{n/2}$ is well-defined). Thus, in words $\mathcal{E}_l$ says (1) $l+1+n/2$ is not blocking $l+n/2$ at phase $J_l$, (2) no honest agent $i \geq l+n/2$ has ever been aware of arms as good as $1-l+n/2$ up to phase $J_l$, and (3) no such $i$ has ever blocked the malicious agent $n+1$. Point (2) will allow us to show that agent $n$ does not become aware of the best arm until phase $J_{n/2}$ (when $\mathcal{E}_{n/2}$ holds). The other events in the definition of $\{ \mathcal{E}_l \}_{l=1}^{n/2}$ will allow us to inductively lower bound their probabilities. The next lemma establishes the base of this inductive argument. 
\begin{lem} \label{lemBadBaseL}
Under the assumptions of Theorem \ref{thmExisting}, $\P(\mathcal{E}_1) \geq 3^{-2(J_1-1)}$.
\end{lem}
\begin{proof}
We first observe that at all phases $j \in [J_1-1]$, only the second case of the malicious strategy -- where the malicious agent recommends to avoid blocking -- arises, which implies $n+1 \notin P_j^{(i)}\ \forall\ j \in [J_1] , i \in [n]$. Therefore, it suffices to show $\P(\mathcal{E}_1') \geq 3^{-2(J_1-1)}$, where we define
\begin{equation*}
\mathcal{E}_1' = \{ 1 + n/2 \notin P_{J_1}^{(2+n/2)} \} \cap \cap_{j=1}^{J_1}  \cap_{i = 1+n/2}^n \{  \min S_j^{(i)} >  n/2 \} .
\end{equation*}
To do so, we will show $\mathcal{E}_1' \supset \mathcal{F} \triangleq \cap_{j=1}^{J_1-1} \cap_{i=1+n/2}^{2+n/2} \{ H_j^{(i)} = n+1 \}$ and $\P(\mathcal{F}) \geq 3^{-2(J_1-1)}$.

To show $\P(\mathcal{F}) \geq 3^{-2(J_1-1)}$, first note that by the law of total expectation, we have
\begin{equation*}
\P(\mathcal{F}) = \E [ \ind ( \cap_{j=1}^{J_1-2} \cap_{i=1+n/2}^{2+n/2} \{ H_j^{(i)} = n+1 \} ) \P_{J_1-1} ( \cap_{i=1+n/2}^{2+n/2} \{ H_{J_1-1}^{(i)} = n+1 \} ) ] .
\end{equation*}
Now when $\cap_{j=1}^{J_1-2} \cap_{i=1+n/2}^{2+n/2} \{ H_j^{(i)} = n+1 \}$ occurs, the malicious strategy implies $P_{J_1-1}^{(i)} = \emptyset$, so $H_{J_1-1}^{(i)}$ is sampled from a set of three agents which includes $n+1$, for each $i \in \{1+n/2,2+n/2\}$. Since this sampling is independent, we conclude that when $\cap_{j=1}^{J_1-2} \cap_{i=1+n/2}^{2+n/2} \{ H_j^{(i)} = n+1 \}$ occurs,
\begin{align*}
\P_{J_1-1} ( \cap_{i=1+n/2}^{2+n/2} \{ H_{J_1-1}^{(i)} = n+1 \} )  = 3^{-2} .
\end{align*}
Thus, combining the previous two expressions with the definition of $\mathcal{F}$ and iterating, we obtain
\begin{equation*}
\P ( \cap_{j=1}^{J_1-1} \cap_{i=1+n/2}^{2+n/2} \{ H_j^{(i)} = n+1 \} ) = \P ( \mathcal{F} ) = \P ( \cap_{j=1}^{J_1-2} \cap_{i=1+n/2}^{2+n/2} \{ H_j^{(i)} = n+1 \} ) / 3^2 = 3^{-2(J_1-1)} . 
\end{equation*}

To show $\mathcal{E}_1' \supset \mathcal{F}$, first observe that when $\mathcal{F}$ occurs, $2+n/2$ does not contact $1+n/2$ at any phase $j \in [J_1-1]$, so $1+n/2 \notin P_{J_1}^{(2+n/2)}$. Thus, it only remains to show that $\mathcal{F}$ implies $\min S_j^{(i)} > n/2$ for all $j \leq J_1$ and $i > n/2$. Suppose instead that $\mathcal{F}$ holds and $\min S_j^{(i)} \leq n/2$ for some such $j$ and $i$. Let $j^\dagger = \min \{ j \leq J_1 : \min S_j^{(i)} \leq n/2 \text{ for some } i > n/2 \}$ be the earliest $j$ it occurs; note $j^\dagger > 1$ by assumption that $\min S_1^{(i)} > n/2$ for $i > n/2$. Let $i^\dagger$ be some agent it occurs for, i.e., $i^\dagger > n/2$ is such that $k^\dagger \in S_{j^\dagger}^{(i^\dagger)}$ for some $k^\dagger \leq n/2$. Since $j^\dagger$ is the earliest such phase and $j^\dagger > 1$, we know $k^\dagger$ was not active for $i^\dagger$ at the previous phase $j^\dagger-1$, so it was recommended to $i^\dagger$ at this phase. By the malicious strategy and $j^\dagger-1 \leq J_1-1$, this implies that the agent $i^\ddagger$ who recommended $k^\dagger$ to $i^\dagger$ is honest, so $k^\dagger$ was active for $i^\ddagger$ at the previous phase, which implies $i^\ddagger \leq n/2$ (else, we contradict the minimality of $j^\dagger$). From the assumed line graph structure, we must have $i^\dagger = 1 + n/2$ and $i^\ddagger = n/2$, i.e., $1+n/2$ contacted $n/2$ at phase $j^\dagger - 1 \leq J_1-1$. But this contradicts the definition of $\mathcal{F}$, which stipulates that $1+n/2$ only contacts $n+1$ before $J_1$.
\end{proof}
 
To continue the inductive argument, we lower bound $\P(\mathcal{E}_{l+1})$ in terms of $\P(\mathcal{E}_l)$.
\begin{lem} \label{lemBadGenL}
Under the assumptions of Theorem \ref{thmExisting}, for any $l \in [(n/2-1]$, we have $\P(\mathcal{E}_{l+1}) \geq 3^{-4} \P(\mathcal{E}_l)$.
\end{lem}
\begin{proof}
The proof is somewhat lengthy and proceeds in four steps.

{\bf Step 1: Probabilistic arguments.} First, we define the events
\begin{equation*}
\mathcal{G}_1 = \cap_{i=l+n/2}^{l+2+n/2} \{  H_{J_l}^{(i)} = n+1 \} , \quad \mathcal{G}_2 = \{ H_{J_l+1}^{(l+1+n/2)} = l + n/2 \} , \quad  \mathcal{G} = \mathcal{G}_1 \cap \mathcal{G}_2 .
\end{equation*}
Then by the law of total expectation, we know that
\begin{equation} \label{eqBadGenLTotExp1}
\P (  \mathcal{E}_l \cap \mathcal{G} ) = \E [ \E_{J_l+1} [ \ind(\mathcal{E}_l \cap \mathcal{G}_1 ) \ind(\mathcal{G}_2)  ] ] = \E [ \ind(\mathcal{E}_l \cap \mathcal{G}_1) \P_{J_l+1} (\mathcal{G}_2)  ] .
\end{equation}
Now if $\mathcal{E}_l \cap \mathcal{G}_1$ occurs, then $l+n/2 \notin P_{J_l}^{(l+1+n/2)}$ (by $\mathcal{E}_l$) and $H_{J_l}^{(l+1+n/2)} = n+1$ (by $\mathcal{G}_1$); the latter implies $l+n/2 \notin P_{J_l+1}^{(l+1+n/2)} \setminus P_{J_l}^{(l+1+n/2)}$, so combined with the former, $l+n/2 \notin P_{J_l+1}^{(l+1+n/2)}$. Thus, $\mathcal{E}_l \cap \mathcal{G}_1$ implies that $H_{J_l+1}^{(l+1+n/2)}$ is sampled from a set of most three agents containing $l+n/2$, so $\P_{J_l+1}(\mathcal{G}_2) \geq 1/3$. Substituting into \eqref{eqBadGenLTotExp1}, and again using total expectation, we thus obtain
\begin{equation*}
\P (  \mathcal{E}_l \cap \mathcal{G} )  \geq \P(\mathcal{E}_l \cap \mathcal{G}_1)  / 3 = \E [ \E_{J_l} [ \ind (\mathcal{E}_l ) \ind ( \mathcal{G}_1) ] ] / 3 = \E [ \ind (\mathcal{E}_l  ) \P_{J_l} ( \mathcal{G}_1 ) ] / 3 .
\end{equation*}
Analogously, when $\mathcal{E}_l$ holds, $n+1 \notin P_{J_l}^{(i)}\ \forall\ i \in \{ l+n/2 , \ldots , l+2+n/2 \}$, which by similar logic gives $\P_{J_l}(\mathcal{G}_1) \geq 3^{-3}$. Therefore, combining the previous two inequalities, we have shown $\P (  \mathcal{E}_l \cap \mathcal{G} ) \geq 3^{-4} \P(\mathcal{E}_l)$. Consequently, it suffices to show that $\mathcal{E}_l \cap \mathcal{G} \subset \mathcal{E}_{l+1}$.

{\bf Step 2: Event decomposition.} For $l' \in \{l,l+1\}$, we decompose $\mathcal{E}_{l'} = \cap_{h=1}^4 \mathcal{H}_{l',h}$, where
\begin{gather}
\mathcal{H}_{l',1} = \{ l' + n/2 \notin P_{J_{l'}}^{(l'+1+n/2)} \} ,\ \mathcal{H}_{l',2} = \cap_{j=1}^{J_{l'-1}}  \cap_{i = l'+n/2}^n \{  \min S_j^{(i)} > 1 - l' + n/2 , n+1 \notin P_j^{(i)} \} , \\
\mathcal{H}_{l',3} = \cap_{j=J_{l'-1}+1}^{J_{l'}}  \cap_{i = l'+n/2}^n \{  \min S_j^{(i)} > 1 - l' + n/2 \} ,\ \mathcal{H}_{l',4} = \cap_{j=J_{l'-1}+1}^{J_{l'}}  \cap_{i = l'+n/2}^n \{ n+1 \notin P_j^{(i)} \} .
\end{gather}
As a simple consequence of these definitions, we note that
\begin{align}
\mathcal{H}_{l,2} \cap \mathcal{H}_{l,3} \cap \mathcal{H}_{l,4} & = \cap_{j=1}^{J_l}  \cap_{i = l+n/2}^n \{  \min S_j^{(i)} > 1 - l + n/2 , n+1 \notin P_j^{(i)} \} \\
& \subset \cap_{j=1}^{J_l}  \cap_{i = l+1+n/2}^n \{  \min S_j^{(i)} > 1 - (l+1) + n/2 , n+1 \notin P_j^{(i)} \} = \mathcal{H}_{l+1,2} .
\end{align}
Hence, to prove $\mathcal{E}_l \cap \mathcal{G} \subset \mathcal{E}_{l+1}$, it suffices to show $\mathcal{E}_l \cap \mathcal{G} \subset \mathcal{H}_{l+1,1} \cap \mathcal{H}_{l+1,3} \cap \mathcal{H}_{l+1,4}$. For the remainder of the proof, we thus assume $\mathcal{E}_l \cap \mathcal{G}$ holds and argue $\mathcal{H}_{l+1,1} \cap \mathcal{H}_{l+1,3} \cap \mathcal{H}_{l+1,4}$ holds.

{\bf Step 3: Some consequences.} We begin by deriving several consequences of $\mathcal{E}_l \cap \mathcal{G}$. First, note each $i \in \{ l + 1 + n/2 , l + 2 + n/2 \}$ contacts $n+1$ at phase $J_l$ (by $\mathcal{G}_1$), who recommends $1-l+n/2$ (by the malicious strategy). Since $\min S_{J_l}^{(i)} > 1 - l + n/2$ (by $\mathcal{H}_{l,3}$), this implies $1 - l + n/2 = \min S_{J_l+1}^{(i)}$, so $1-l+n/2$ is most played in phase $J_l+1$ (by Lemma \ref{lemBestArmId}). In summary, we have shown
\begin{equation}\label{eqBadGenLConRec}
H_{J_l}^{(i)} = n+1 , R_{J_l}^{(i)} = B_{J_l+1}^{(i)} = 1 - l + n/2\ \forall\ i \in \{ l + 1 + n/2 , l + 2 + n/2 \} .
\end{equation}
Second, as a consequence of the above and Lemma \ref{lemBestArmId}, we can also write
\begin{equation}\label{eqBadGenLMono}
1 - l + n/2 = \min S_{J_l+1}^{(i)} \geq \min S_{J_l+2}^{(i)} \geq \cdots\ \forall\ i \in \{ l + 1 + n/2 , l + 2 + n/2 \} .
\end{equation}
Third, we know $l+n/2$ contacts $n+1$ at phase $J_l$ (by $\mathcal{G}_1$), who responds with a currently active arm (by the malicious strategy), so since $\min S_{J_l}^{(l+n/2)} > 1 - l + n/2$ (by $\mathcal{H}_{l,3}$), we have
\begin{equation}\label{eqBadGenLlastBad}
\min S_{J_l+1}^{(l+n/2)} > 1 - l + n/2 .
\end{equation}
As a consequence of \eqref{eqBadGenLlastBad}, we see that when $l+1+n/2$ contacts $l+n/2$ at phase $J_l+1$ (which occurs by $\mathcal{G}_2$), $l+n/2$ recommends some arm strictly worse than $1-l+n/2$. On the other hand, by \eqref{eqBadGenLMono} and Lemma \ref{lemBestArmId}, we know the most played arm for $l+1+n/2$ in phase $J_l+2$ has index at most $1-l+n/2$. Taken together, the recommendation is not most played, so
\begin{equation} \label{eqLplus1blockL}
 l + n/2 \in P_j^{(l+1+n/2)}\ \forall\ j \in \{ J_l + 2 , \ldots , ( J_l + 2 )^2 = J_{l+1} \} .
\end{equation}

{\bf Step 4: Completing the proof.} Using the above, we prove in turn that $\mathcal{H}_{l+1,4}$, $\mathcal{H}_{l+1,3}$, and $\mathcal{H}_{l+1,1}$ hold. For $\mathcal{H}_{l+1,4}$, we use proof by contradiction: if $\mathcal{H}_{l+1,4}$ fails, we can find $i \geq l+1+n/2$ and $j \in \{ J_l+1 , \ldots , J_{l+1} \}$ such that $n+1 \in P_j^{(i)}$. Let $j^\dagger = \min \{ j \in \{ J_l+1 , \ldots , J_{l+1} \} : n+1 \in P_j^{(i)} \}$ be the minimal such $j$ (for this $i$). Since $n+1 \notin P_{J_l}^{(i)}$ (by $\mathcal{H}_{l,4}$) and $j^\dagger$ is minimal, we must have $n+1 \in P_{j^\dagger}^{(i)} \setminus P_{j^\dagger-1}^{(i)}$, i.e., $n+1$ was blocked for the recommendation it provided at $j^\dagger-1$. If $i \geq l+3+n/2$, this contradicts the malicious strategy, since $j^\dagger-1 \in \{ J_l,\ldots,J_{l+1}-1\}$ and the strategy avoids blocking for such $i$ and $j^\dagger$. A similar contradiction arises if $i \in \{ l + 1 + n/2 , l + 2 + n/2 \}$ and $j^\dagger \geq J_l+2$ (since $j^\dagger-1 \in \{ J_l+1, \ldots , J_{l+1}-1\}$ in this case), so we must have $i \in \{ l + 1 + n/2 , l + 2 + n/2 \}$ and $j^\dagger = J_l+1$. But in this case, $n+1 \in P_{j^\dagger}^{(i)} \setminus P_{j^\dagger-1}^{(i)} = P_{J_l+1}^{(i)} \setminus P_{J_l}^{(i)}$ contradicts \eqref{eqBadGenLConRec}.

Next, we show $\mathcal{H}_{l+1,3}$ holds. The logic is similar to the end of the Lemma \ref{lemBadBaseL} proof. If instead $\mathcal{H}_{l+1,3}$ fails, we can find $j \in \{J_l+1,\ldots,J_{l+1}\}$ and $i \in \{l+1+n/2,\ldots,n\}$ such that $\min S_j^{(i)} \leq (n/2)-l$. Let $j^\dagger$ be the minimal such $j$ and $i^\dagger \geq l+1+n/2$ an agent with $\min S_{j^\dagger}^{(i^\dagger)} = k^\dagger$ for some $k^\dagger \leq (n/2)-l$. Since $\min S_{J_l}^{(i^\dagger)} > 1 - l + n/2$ (by $\mathcal{H}_{l,3}$), $j^\dagger \geq J_l+1$, and $j^\dagger$ is minimal, we know that $k^\dagger$ was recommended to $i^\dagger$ at phase $j^\dagger - 1 \in \{ J_l , \ldots , J_{l+1} - 1 \}$. By the malicious strategy, this implies that the recommending agent (say, $i^\ddagger$) was honest. Therefore, $k^\dagger$ was active for $i^\ddagger$ at phase $j^\dagger-1$, so since $j^\dagger$ is minimal, $i^\ddagger \leq l + n/2$. Hence, by the assumed graph structure, $i^\dagger = l+1+n/2$ contacted $i^\ddagger = l+n/2$ at phase $j^\dagger-1$, who recommended $k^\dagger$. If $j^\dagger - 1  \in \{ J_l , J_l+2 , \ldots , J_{l+1} - 1 \}$, this contact cannot occur, since $l+1+n/2$ instead contacts $n+1$ at $J_l$ (by $\mathcal{G}_1$) and does not contact $l+n/2$ at $J_l+2,\ldots,J_{l+1}$ (by \eqref{eqLplus1blockL}). Hence, we must have $j^\dagger - 1 = J_l + 1$, so $\min S_{J_l+1}^{(l+n/2)} \leq k^\dagger \leq (n/2)-l$, contradicting \eqref{eqBadGenLlastBad}. 

Finally, we prove $\mathcal{H}_{l+1,1}$. Suppose instead that $l+1+n/2 \in P_{J_{l+1}}^{(l+2+n/2)}$, i.e., $l+1+n/2$ is blocked at $J_{l+1}$. Then since $P_0^{(l+2+n/2)} = \emptyset$, we must have $l+1+n/2 \in P_j^{(l+2+n/2)} \setminus P_{j-1}^{(l+2+n/2)}$ for some $j \in [ J_{l+1} ]$. Let $j^\dagger$ be the maximal such $j$. Then $j^\dagger \geq \sqrt{J_{l+1}} = J_l+2$; otherwise, if $j^\dagger < \sqrt{J_{l+1}}$, $l+1+n/2$ would have been un-blocked by phase $J_{l+1}$. Therefore, the blocking rule implies 
\begin{equation} \label{eqBadGenL3fin} 
B_{j^\dagger}^{(l+2+n/2)} \neq R_{j^\dagger-1}^{(l+2+n/2)} =  B_{j^\dagger-1}^{(l+1+n/2)} .
\end{equation}
By $j^\dagger \in \{ J_l+2 , \ldots , J_{l+1} \}$, $\mathcal{H}_{l+1,3}$, and \eqref{eqBadGenLMono}, we also know
\begin{equation*}
- l + n/2 < \min S_{j^\dagger}^{(l+2+n/2)} , \min S_{j^\dagger-1}^{(l+1+n/2)} \leq 1 - l + n/2 ,
\end{equation*}
so $\min S_{j^\dagger}^{(l+2+n/2)} =  \min S_{j^\dagger-1}^{(l+1+n/2)} = 1 - l + n/2$. Combined with \eqref{eqBadGenL3fin}, we must have $B_{j^\dagger+h-2}^{(l+h+n/2)} > \min S_{j^\dagger+h-2}^{(l+h+n/2)} = 1 - l +n/2$ for some $h \in \{1,2\}$, which contradicts Lemma \ref{lemBestArmId} (since $j^\dagger \geq J_l+2$).
\end{proof}

Finally, we can prove the theorem. Define $\sigma = \min \{ j \in \N : 1 \in S_j^{(n)} \}$. Then by definition, $I_t^{(n)} \neq 1$ for any $t \leq A_{\sigma-1}$. Hence, because $\Delta_2 = 1/15$ in the problem instance of the theorem, we obtain
\begin{equation*}
\frac{A_{\sigma-1} \wedge T}{15}  = \sum_{t=1}^{A_{\sigma-1} \wedge T} \frac{\ind ( I_t^{(n)} \neq 1 )}{15} = \sum_{t=1}^{A_\sigma \wedge T} \sum_{k=2}^K \frac{\ind ( I_t^{(n)} = k )}{15} \leq \sum_{t=1}^T \sum_{k=2}^K \Delta_k \ind ( I_t^{(n)} = k ) .
\end{equation*}
Thus, by Claim \ref{clmDecomposition} from Appendix \ref{appBasicRes} and since $A_{\sigma-1} = ( \sigma-1 )^2$ by the choice $\beta = 2$, we can write
\begin{equation} \label{eqBadJensen}
R_T^{(n)} = \E \left[ \sum_{t=1}^T \sum_{k=2}^K \Delta_k   \ind ( I_t^{(i)} = k ) \right] \geq \frac{\E [ A_{\sigma-1} \wedge T  ]}{15} = \frac{\E [ ( \sigma-1 )^2 \wedge T ]}{15} .
\end{equation}
Let $l \in [n/2]$ be chosen later. Then $\sigma > J_l$ implies $\sigma-1 \geq J_l$ (since $\sigma , J_l \in \N$). Thus, we can write
\begin{equation*}
\E [ ( \sigma-1 )^2 \wedge T ] \geq \E [ ( (  \sigma-1 )^2 \wedge T ) \ind ( \sigma > J_l ) ] \geq ( J_l^2 \wedge T ) \P ( \sigma > J_l ) .
\end{equation*}
By definition of $\sigma$ and $\mathcal{E}_l$, along with Lemmas \ref{lemBadBaseL} and \ref{lemBadGenL}, we also know
\begin{equation*}
\P ( \sigma > J_l ) \geq \P ( \mathcal{E}_l ) \geq 3^{-4(l-1)} \P(\mathcal{E}_1) \geq 3^{-4(l-1)}  \cdot 3^{-2(J_1-1)} = 9^{ 3 - 2 l - J_1 } .
\end{equation*}
By \eqref{eqJlDoubExp} from Section \ref{secBadInstance}, we know $J_l^2 \geq J_1^{2^l}  = ( 2^8 )^{ 2^l } = 2^{2^{l+3}}$. Combined with the previous three bounds, and letting $C$ denote the constant $C =  9^{3-J_1} / 15$, we thus obtain
\begin{equation} \label{eqBadBeforeL}
R_T^{(n)} \geq  ( 2^{2^{l+3}} \wedge T ) \cdot 9^{ 3 - 2 l - J_1 } / 15 = C \cdot 81^{-l} \cdot ( 2^{2^{l+3}}\wedge T )\ \forall\ l \in [n/2] . 
\end{equation}
We now consider three different cases, each with a different choice of $l$.
\begin{itemize}
\item If $T > 2^{2^{(n/2)+3}}$, choose $l=n/2$. Then \eqref{eqBadBeforeL} becomes $R_T^{(n)} \geq C \cdot 81^{-n/2} \cdot 2^{2^{(n/2)+3} }$. Observe that
\begin{align*}
81^{-n/2} \cdot 2^{2^{(n/2)+3} } &  \geq 16^{-n} \cdot 2^{2^{(n/2)+3} } = ( 2^4 )^{ 2 \cdot 2^{n/2} - n } \geq ( 2^4 )^{2^{n/2}} > \exp ( 2^{n/2} ) \\
& = \exp ( \exp ( n \log(2) / 2 ) )  > \exp ( \exp ( n/3) ) ,
\end{align*}
where the second inequality is $n \leq 2^{n/2}$ for $n \in \{2,4,8,\ldots\}$. On the other hand, Claim \ref{clmLogRegNaive} below shows $R_T^{(n)} \geq \log(T) / C_{\ref{clmLogRegNaive}}$ for some absolute constant $C_{\ref{clmLogRegNaive}} > 0$. Thus, we have shown
\begin{equation*}
R_T^{(n)} = ( R_T^{(n)} / 2 ) + ( R_T^{(n)} / 2 ) \geq ( C / 2 ) \exp ( \exp (n/3) ) + \log(T) / ( 2 C_{\ref{clmLogRegNaive}} ) .
\end{equation*}
\item If $T \in ( 2^8 ,  2^{2^{(n/2)+3}} ]$, let $l = \ceil{ \log_2 ( \log_2 ( T ) ) -3}$. Then $2^{2^{l+3}} \geq T$, so $2^{2^{l+3}}\wedge T  = T$. Furthermore, we know $l \leq \log_2 ( \log_2 ( T ) ) - 2$, which implies 
\begin{equation*}
\quad 81^{-l} \geq 81^2 \cdot 81^{ - \log_2 ( \log_2 ( T ) ) } \geq 81^2 \cdot 2^{ - 7 \log_2 ( \log_2 ( T ) ) } = 81^2 / \log_2^7 (T) = 81^2 \log^7(2) / \log^7(T) .
\end{equation*}
Next, observe that $0 = \log_2(\log_2(2^8)) -3 < \log_2(\log_2(T)) -3 \leq n/2$ for this case of $T$, so $l \in [n/2]$. Thus, we can choose this $l$ in \eqref{eqBadBeforeL} and combine with the above bounds to lower bound regret as $R_T^{(n)} \geq 81^2 \log^7(2) C T / \log^7(T)$.
\item If $T \leq 2^8$, choose $l = 1$. Then $2^{2^{l+3}} = 2^{16} \geq T$, so \eqref{eqBadBeforeL} implies $R_T^{(n)} \geq C T / 81$. 
\end{itemize}
Hence, in all three cases, we have shown $R_T^{(n)} \geq C' \min \{ \log(T) + \exp ( \exp ( n / 3 ) ) , T / \log^7(T) \}$ for some absolute constant $C' > 0$. This establishes the theorem.

We return to state and prove the aforementioned Claim \ref{clmLogRegNaive}. We note the analysis is rather coarse; our only goal here is to establish a $\log T$ scaling (not optimize constants).
\begin{clm} \label{clmLogRegNaive}
Under the assumptions of Theorem \ref{thmExisting}, we have $R_T^{(n)} \geq \log (T) / C_{\ref{clmLogRegNaive}}$, where $C_{\ref{clmLogRegNaive}} = 15 \log 99$.
\end{clm}
\begin{proof}
If $T=1$, the bound holds by nonnegativity. If $T \in \{2,\ldots,99\}$, then since $\min S_1^{(n)} > n/2$ and $\Delta_2 \geq 1/15$ by assumption in Theorem \ref{thmExisting}, we know $\Delta_{I_1^{(n)}} \geq 1/15$, which implies $R_T^{(n)} \geq 1/15 \geq \log(T) / C_{\ref{clmLogRegNaive}}$. Thus, only the case $T \geq 100$ remains. By Claim \ref{clmDecomposition} from Appendix \ref{appBasicRes} and $\Delta_2 \geq 1/15$,
\begin{equation*}
R_T^{(n)} \geq \frac{ \E [ \sum_{t=1}^T \ind (I_t^{(n)} \neq 1) ] }{15} = \frac{ \log (99) \E [ \sum_{t=1}^T \ind (I_t^{(n)} \neq 1) ]  }{ C_{\ref{clmLogRegNaive}} } > \frac{ 2 \E [ \sum_{t=1}^T \ind (I_t^{(n)} \neq 1) ]  }{ C_{\ref{clmLogRegNaive}} } . 
\end{equation*}
Thus, it suffices to show that $\sum_{t=1}^T \ind (I_t^{(n)} \neq 1) \geq \log(T) / 2$. Suppose instead that this inequality fails. Then since the left side is an integer, we have $\sum_{t=1}^T \ind (I_t^{(n)} \neq 1) \leq \floor{ \log(T) / 2 }$ by assumption. Therefore, we can find $t \in \{ T - \floor{\log(T)/2} + 1 , \ldots , T \}$ such that $I_t^{(n)} = 1$ (else, we violate the assumed inequality). By this choice of $t$ and the assumed inequality, we can then write
\begin{equation*}
T_1^{(n)}(t-1) = t - 1 - \sum_{s=1}^{t-1} \ind ( I_t^{(n)} \neq 1 ) \geq \left( T - \floor{\log(T)/2} \right) - \left( \floor{\log(T)/2} \right)  \geq T - \log T .
\end{equation*}
We can lower bound the right side by $4 \log T$ (else, applying Claim \ref{clmLogTrick} from Appendix \ref{appBasicRes} with $x=T$, $y=1$, and $z=5$ yields $T < 100$, a contradiction), which is further bounded by $4 \log t$. Combined with the fact that rewards are deterministic, $\mu_1 = 1$, and $\alpha = 4$ in Theorem \ref{thmExisting}, we obtain
\begin{equation} \label{eqLogRegNaive1}
\hat{\mu}_1^{(n)}(t-1) + \sqrt{ \alpha \log(t) / T_1^{(n)}(t-1) } = 1 + \sqrt{ 4 \log(t) / T_1^{(n)}(t-1) } \leq 2 .
\end{equation}
Next, let $k \in S_{A^{-1}(t)}^{(n)}$ be any other arm which is active for $n$ at time $t$. Then clearly
\begin{gather}
T_k^{(n)}(t-1) \leq \sum_{t=1}^T \ind ( I_t^{(n)} = k ) \leq \sum_{t=1}^T \ind ( I_t^{(n)} \neq 1 ) \leq \floor{ \log(T) / 2 } \leq \log(T) / 2 \\
\Rightarrow 
\hat{\mu}_k^{(n)}(t-1) + \sqrt{ \alpha \log(t) / T_k^{(n)}(t-1) } \geq \sqrt{ 8 \log(t) / \log(T) } = 2 \sqrt{ \log(t) / \log(\sqrt{T}) } . \label{eqLogRegNaiveK}
\end{gather}
By \eqref{eqLogRegNaive1}, \eqref{eqLogRegNaiveK}, the fact that $I_t^{(n)} = 1$, and the UCB policy, we conclude $t \leq \sqrt{T}$. Since $T \geq 4$, this further implies $t \leq T/2$. But we also know that $t \geq T - \floor{\log(T)/2} + 1 > T - \log(T)/2$. Combining these inequalities gives $T < \log T$, a contradiction.
\end{proof}

%% file: proofProposed.tex
\section{Proofs from Section \ref{secProposed}} \label{appProofProposed}

\subsection{Proof of Theorem \ref{thmProposed}} \label{appProofThmProposed}

Fix an honest agent $i \in [n]$. Let $\tau^{(i)} = \tau_{\text{spr}} \vee \tau_{\text{blk}}^{(i)}$, where we recall from the proof sketch that
\begin{gather}
\tau_{\text{spr}} = \inf \{ j \in \N : B_{j'}^{(i')} = 1\ \forall\ i' \in [n] , j' \geq j \} , \\ 
\tau_{\text{blk}}^{(i)} = \inf \{ j \in \N : H_{j'-1}^{(i)} \in P_{j'}^{(i)} \setminus P_{j'-1}^{(i)}\ \forall\ j' \geq j\ s.t.\ R_{j'-1}^{(i)} \neq 1 \}  . 
\end{gather}
Let $\gamma_i \in (0,1)$ be chosen later. Denote by $\overline{S}^{(i)} = \{2,\ldots,K\} \cap \hat{S}^{(i)}$ and $\underline{S}^{(i)} = \{2,\ldots,K\} \setminus \hat{S}^{(i)}$ the suboptimal sticky and non-sticky arms, respectively, for agent $i$. Then by Claim \ref{clmDecomposition} from Appendix \ref{appBasicRes}, we can decompose regret as $R_T^{(i)} = \sum_{h=1}^4 R_{T,h}^{(i)}$, where
\begin{gather}
R_{T,1}^{(i)} = \E \left[ \sum_{t=1}^{A_{\tau_{\text{spr}}} \wedge T}  \Delta_{I_t^{(i)}} \right] , \quad R_{T,2}^{(i)} = \sum_{k \in \overline{S}^{(i)} } \Delta_k \E \left[ \sum_{t=1+A_{\tau_{\text{spr}}}}^T  \ind( I_t^{(i)} = k ) \right] , \\
R_{T,3}^{(i)} = \sum_{k \in \underline{S}^{(i)} } \Delta_k \E \left[ \sum_{t=1+A_{\tau_{\text{spr}}}}^{ A_{ \tau^{(i)} \vee \ceil{T^{\gamma_i/\beta}} } \wedge T  } \ind( I_t^{(i)} = k ) \right]  , \quad R_{T,4}^{(i)} = \sum_{k \in \underline{S}^{(i)} } \Delta_k \E \left[ \sum_{t=1+ A_{ \tau^{(i)} \vee \ceil{T^{\gamma_i/\beta}} }}^T \ind( I_t^{(i)} = k ) \right] ,
\end{gather}
and where $\sum_{t=s_1}^{s_2} \ind ( I_t^{(i)} = k ) = 0$ whenever $s_1 > s_2$ by convention. Thus, we have rewritten regret as the sum of four terms: $R_{T,1}^{(i)}$, which accounts for regret before the best arm spreads; $R_{T,2}^{(i)}$, the regret due to sticky arms after the best arm spreads; $R_{T,3}^{(i)}$, regret from non-sticky arms after the best arm spreads but before phase $\tau^{(i)} \vee \ceil{ T^{\gamma_i/\beta} }$; and $R_{T,4}^{(i)}$, regret from non-sticky arms after this phase. The subsequent lemmas bound these terms in turn.

\begin{lem} \label{lemEarly}
Under the assumptions of Theorem \ref{thmProposed}, for any $i \in [n]$ and $T \in \N$, we have
\begin{equation*}
R_{T,1}^{(i)} \leq \E [A_{\tau_{\text{spr}}}] = O (  S^{\beta/( \rho_1^2(\beta-1)) }  \vee ( S \log(S/\Delta_2) / \Delta_2^2 )^{\beta/(\beta-1)} \vee (\bar{d} \log(n+m) )^{\beta/\rho_1} \vee n K^2 S ) +  \E [A_{ 2 \bar{\tau}_{\text{spr}} }]  .
\end{equation*}
\end{lem}
\begin{proof}
Assumption \ref{assReward} ensures $\Delta_k \leq 1$, so $R_{T,1}^{(i)} \leq \E [A_{\tau_{\text{spr}}}]$. The result follows from the bound on $\E [ A_{\tau_{\text{spr}}} ]$ discussed in Section \ref{secSpread} and formally stated as Corollary \ref{corEarly} in Appendix \ref{appSpread}.
\end{proof}

\begin{lem} \label{lemLateSticky}
Under the assumptions of Theorem \ref{thmProposed}, for any $i \in [n]$ and $T \in \N$, we have
\begin{equation}
R_{T,2}^{(i)} \leq \sum_{k \in \overline{S}^{(i)}}  \frac{4 \alpha \log  T }{\Delta_k} + \frac{ 4 ( \alpha-1 ) | \overline{S}^{(i)} | }{ 2 \alpha - 3 }   .
\end{equation}
\end{lem}
\begin{proof}
For any $k \in \overline{S}^{(i)}$, Claim \ref{clmBanditPlays} and Corollary \ref{corBanditTailDelta} from Appendix \ref{appBanditRes} imply
\begin{align}
\E \left[ \sum_{t = 1 + A_{\tau_{\text{spr}}}}^T \ind \left( I_t^{(i)} = k \right) \right] & =   \E \left[  \sum_{t = 1 + A_{\tau_{\text{spr}}}}^T \ind \left( I_t^{(i)} = k , T_k^{(i)}(t-1) < \frac{4 \alpha \log t}{\Delta_k^2} \right) \right] \\ 
& \quad + \E\left[  \sum_{t = 1 + A_{\tau_{\text{spr}}}}^T \ind \left( I_t^{(i)} = k , T_k^{(i)}(t-1) \geq \frac{4 \alpha \log t}{\Delta_k^2} \right) \right] \\
& \leq \frac{ 4 \alpha \log T}{ \Delta_k^2} + \frac{ 4(\alpha-1)}{2\alpha-3} ,
\end{align}
so multiplying by $\Delta_k$, using $\Delta_k \leq 1$, and summing over $k \in \overline{S}^{(i)}$ completes the proof.
\end{proof}

\begin{lem} \label{lemIntNonSticky}
Under the assumptions of Theorem \ref{thmProposed}, for any $i \in [n]$, $\gamma_i \in (0,1)$, and $T \in \N$, we have
\begin{equation*}
R_{T,3}^{(i)} \leq \sum_{k \in \underline{S}^{(i)}}  \frac{4 \alpha \log  A_{\ceil{T^{\gamma_i/\beta}}} }{\Delta_k}  + \frac{ 4(\alpha-1) |\underline{S}^{(i)}| }{2\alpha-3} + \frac{ 4 C_{\ref{clmIntNonStickSmallT}} \alpha K }{ \Delta_2 \gamma_i } \log \left( \frac{C_{\ref{clmIntNonStickSmallT}} K}{ \Delta_2 \gamma_i } \right) + 1 + \E [A_{\tau_{\text{spr}}}] .
\end{equation*}
\end{lem}
\begin{proof}
The proof is nontrivial; we defer it to the end of this sub-appendix.
\end{proof}

\begin{lem} \label{lemLateNonSticky}
Under the assumptions of Theorem \ref{thmProposed}, for any $i \in [n]$, $\gamma_i \in (0,1)$, and $T \in \N$, we have
\begin{equation*}
R_{T,4}^{(i)} \leq \frac{2 \eta-1}{\eta-1} \max_{ \tilde{S} \subset \underline{S}^{(i)} : |\tilde{S}| \leq d_{\text{mal}}(i)+2 } \sum_{k \in \tilde{S}} \frac{ 4 \alpha \log T }{ \Delta_k } + \frac{ 8 \alpha \beta \log_{\eta}(1/\gamma_i) (d_{\text{mal}}(i)+2) }{\Delta_2} + \frac{ 4(\alpha-1) |\underline{S}^{(i)} | }{2\alpha-3}  .
\end{equation*}
\end{lem}
\begin{proof}[Proof sketch]
The proof follows the same logic as that of \cite[Lemma 4]{vial2021robust}, so we omit it. The only differences are (1) we replace $m$ (the number of malicious agents connected to $i$ for the complete graph) with $d_{\text{mal}}(i)$, and (2) we use Claim \ref{clmPolySeries} from Appendix \ref{appBasicRes} to bound the summation in \cite[Lemma 4]{vial2021robust}. We refer the reader to the Theorem \ref{thmProposed} proof sketch for intuition.
\end{proof}

Additionally, we note the sum $R_{T,3}^{(i)} + R_{T,4}^{(i)}$ can be naively bounded as follows.
\begin{lem} \label{lemNonStickyNaive}
Under the assumptions of Theorem \ref{thmProposed}, for any $i \in [n]$ and $T \in \N$, we have
\begin{equation*}
R_{T,3}^{(i)} + R_{T,4}^{(i)} \leq \sum_{k \in \underline{S}^{(i)}}  \frac{4 \alpha \log  T }{\Delta_k} +  \frac{ 4 ( \alpha-1 ) | \underline{S}^{(i)} | }{ 2 \alpha - 3 }  .
\end{equation*}
\end{lem}
\begin{proof}
The proof follows the exact same logic as that of Lemma \ref{lemLateSticky} so is omitted.
\end{proof}

We can now prove the theorem. First, we use the regret decomposition $R_T^{(i)} = \sum_{h=1}^4 R_{T,h}^{(i)}$, Lemmas \ref{lemEarly}-\ref{lemLateNonSticky}, and the fact that $|\overline{S}^{(i)}|+|\underline{S}^{(i)}| \leq K$ to write
\begin{align}
R_T^{(i)} & \leq \frac{2 \eta-1}{\eta-1} \max_{ \tilde{S} \subset \underline{S}^{(i)} : |\tilde{S}| \leq d_{\text{mal}}(i)+2 } \sum_{k \in \tilde{S}} \frac{ 4 \alpha \log T }{ \Delta_k } + \sum_{k \in \overline{S}^{(i)}}  \frac{4 \alpha \log  T }{\Delta_k} + \sum_{k \in \underline{S}^{(i)}}  \frac{4 \alpha \log  A_{\ceil{T^{\gamma_i/\beta}}} }{\Delta_k} \label{eqProposedProofT} \\
& \quad + \frac{ 8 \alpha \beta \log_{\eta}(1/\gamma_i) (d_{\text{mal}}(i)+2) }{\Delta_2} + \frac{ 8(\alpha-1) K  }{2\alpha-3} +  \frac{ 4 C_{\ref{clmIntNonStickSmallT}} \alpha K }{ \Delta_2 \gamma_i } \log \left( \frac{C_{\ref{clmIntNonStickSmallT}} K}{ \Delta_2 \gamma_i } \right) + 1 + 2 \E [A_{\tau_{\text{spr}}}] . \label{eqProposedProofNoT} 
\end{align}
Now choose $\gamma_i = \Delta_2 / ( K \Delta_{S+d_{\text{mal}}(i) + 4} ) \in (0,1)$. Then
\begin{align}
\eqref{eqProposedProofNoT} & = \tilde{O} \left( ( d_{\text{mal}}(i) / \Delta_2 ) \vee ( K / \Delta_2 )^2 \right) + 2 \E [A_{\tau_{\text{spr}}}] \\
& = \tilde{O} \left( ( d_{\text{mal}}(i) / \Delta_2 ) \vee ( K / \Delta_2 )^2 \vee S^{\beta/( \rho_1^2(\beta-1)) }  \vee ( S / \Delta_2^2 )^{\beta/(\beta-1)} \vee \bar{d}^{\beta/\rho_1} \vee n K^2 S  \right)  +  4 \E [A_{ 2 \bar{\tau}_{\text{spr}}}] ,
\end{align}
where the second inequality is due to Lemma \ref{lemEarly}. Furthermore, by Claim \ref{clmPropAj} from Appendix \ref{appBasicRes}, we know that $\log ( A_{ \ceil{ T^{\gamma_i/\beta} } } ) \leq \log ( e^{2 \beta} ( T^{\gamma_i / \beta } )^\beta ) = 2 \beta + \gamma_i \log T$. Combined with $\underline{S}^{(i)} \leq K$, $\Delta_k \geq \Delta_2\ \forall\ k \in \underline{S}^{(i)}$, and the choice of $\gamma_i$, we can thus write
\begin{equation*}
\sum_{k \in \underline{S}^{(i)}} \frac{4 \alpha \log  A_{\ceil{T^{\gamma_i/\beta}}} }{\Delta_k} \leq \sum_{k \in \underline{S}^{(i)}} \frac{ 8 \alpha \beta + \gamma_i \log T}{\Delta_k} \leq \frac{ 8 \alpha \beta K + \gamma_i K \log T }{ \Delta_2 } = \frac{ 8 \alpha \beta K }{ \Delta_2 } + \frac{ \log T }{ \Delta_{S+d_{\text{mal}}(i) + 4} } .
\end{equation*}
Therefore, we can bound \eqref{eqProposedProofT} as follows:
\begin{align}
\eqref{eqProposedProofT}  & \leq \frac{2 \eta-1}{\eta-1} \max_{ \tilde{S} \subset \underline{S}^{(i)} : |\tilde{S}| \leq d_{\text{mal}}(i)+2 } \sum_{k \in \tilde{S}} \frac{ 4 \alpha \log T }{ \Delta_k } + \sum_{k \in \overline{S}^{(i)}}  \frac{4 \alpha \log  T }{\Delta_k} + \frac{ \log T }{ \Delta_{S+d_{\text{mal}}(i) + 4} } + \frac{ 8 \alpha \beta K }{ \Delta_2 } \\
& \leq \frac{2 \eta-1}{\eta-1} \sum_{k=2}^{d_{\text{mal}}(i)+3} \frac{4 \alpha \log T}{\Delta_k} + \sum_{k=d_{\text{mal}}(i)+4}^{S+d_{\text{mal}}(i)+4} \frac{4 \alpha \log T}{\Delta_k} + \frac{8\alpha \beta K}{\Delta_2} ,
\end{align}
where the second inequality holds by $\Delta_2 \leq \cdots \leq \Delta_K$ and $|\overline{S}^{(i)}| \leq S$. Combining the above yields
\begin{align} \label{eqProofProposedMin1}
R_T^{(i)} & \leq 4 \alpha \log(T) \left( \frac{2 \eta-1}{\eta-1} \sum_{k=2}^{d_{\text{mal}}(i)+3} \frac{1}{\Delta_k} + \sum_{k=d_{\text{mal}}(i)+4}^{S+d_{\text{mal}}(i)+4} \frac{1}{\Delta_k} \right) +  2 \E [A_{ 2 \bar{\tau}_{\text{spr}}}]  \\
& \quad +  \tilde{O} \left( ( d_{\text{mal}}(i) / \Delta_2 ) \vee ( K / \Delta_2 )^2 \vee S^{\beta/( \rho_1^2(\beta-1)) }  \vee ( S / \Delta_2^2 )^{\beta/(\beta-1)} \vee \bar{d}^{\beta/\rho_1} \vee n K^2 S  \right) . 
\end{align}
Alternatively, we can simply use Lemmas \ref{lemEarly}, \ref{lemLateSticky}, and \ref{lemNonStickyNaive} to write
\begin{equation} \label{eqProofProposedMin2}
R_T^{(i)} \leq 4 \alpha \log(T) \sum_{k=2}^K \frac{1}{\Delta_k} + \frac{4(\alpha-1)K}{2\alpha-3} + \E [A_{\tau_{\text{spr}}}] .
\end{equation}
Therefore, combining the previous two expressions and again invoking Lemma \ref{lemEarly} to bound the additive terms in \eqref{eqProofProposedMin2} by those in \eqref{eqProofProposedMin1}, we obtain the desired bound.

Thus, it only remains to prove Lemma \ref{lemIntNonSticky}. We begin by using some standard bandit arguments recounted in Appendix \ref{appBanditRes} to bound $R_{T,3}^{(i)}$ in terms of a particular tail of $\tau^{(i)}$.
\begin{clm} \label{clmIntNonStickDecomp}
Under the assumptions of Theorem \ref{thmProposed}, for any $i \in [n]$, $\gamma_i \in (0,1)$, and $T \in \N$, we have
\begin{align}
R_{T,3}^{(i)} & \leq \sum_{k \in \underline{S}^{(i)}}  \frac{4 \alpha \log  A_{\ceil{T^{\gamma_i/\beta}}} }{\Delta_k} + \frac{ 4 ( \alpha-1 ) | \underline{S}^{(i)} | }{ 2 \alpha - 3 } +  \E [A_{\tau_{\text{spr}}}] \\
& \quad + \frac{ 4 \alpha K \log T}{ \Delta_2 } \P \left( \tau^{(i)} > \ceil{ T^{\gamma_i/\beta} } , A_{\tau_{\text{spr}}} < \frac{4 \alpha K \log T}{\Delta_2} \right)  . \label{eqIntRemaining}
\end{align}
\end{clm}
\begin{proof}
If $T=1$, we can naively bound $R_{T,3}^{(i)} \leq 1$, which completes the proof. Thus, we assume $T>1$ (which will allow us to divide by $\log T$ later). For any $k \in \underline{S}^{(i)}$, we first write
\begin{align}
& \E \left[ \sum_{t = 1 + A_{\tau_{\text{spr}}}}^{ A_{\tau^{(i)} \vee \ceil{ T^{\gamma_i/\beta} } }\wedge T } \ind \left( I_t^{(i)} = k \right) \right] \leq \E \left[ \sum_{t = 1 + A_{\tau_{\text{spr}}}}^T \ind \left( I_t^{(i)} = k , T_k^{(i)}(t-1) \geq \frac{4 \alpha \log t}{\Delta_k^2} \right) \right]
\label{eqInterProbTerm} \\
&  \quad + \E \left[ \ind ( \tau^{(i)}\leq \ceil{ T^{\gamma_i/\beta} }  ) \sum_{t = 1 + A_{\tau_{\text{spr}}}}^{ A_{\ceil{ T^{\gamma_i/\beta} } } \wedge T } \ind \left( I_t^{(i)} = k , T_k^{(i)}(t-1) < \frac{4 \alpha \log t}{\Delta_k^2} \right) \right] \label{eqInterAsTermSmallTau} \\
&  \quad + \E \left[ \ind ( \tau^{(i)}> \ceil{ T^{\gamma_i/\beta} }  ) \sum_{t = 1 + A_{\tau_{\text{spr}}}}^{ A_{\tau^{(i)}}\wedge T } \ind \left( I_t^{(i)} = k , T_k^{(i)}(t-1) < \frac{4 \alpha \log t}{\Delta_k^2} \right) \right]. \label{eqInterAsTermLargeTau} 
\end{align}
By Corollary \ref{corBanditTailDelta} from Appendix \ref{appBanditRes}, \eqref{eqInterProbTerm} is bounded by $4(\alpha-1)/(2\alpha-3)$. By Claim \ref{clmBanditPlays} from Appendix \ref{appBanditRes} and $\ind ( \cdot )  \leq 1$, \eqref{eqInterAsTermSmallTau} is bounded by $4 \alpha \log ( A_{\ceil{ T^{\gamma_i/\beta} } } ) / \Delta_k^2$. For \eqref{eqInterAsTermLargeTau}, Claim \ref{clmBanditPlays} similarly gives
\begin{equation} \label{eqInterAsTermLargeTauInitA}
\sum_{t = 1 + A_{\tau_{\text{spr}}}}^{ A_{\tau^{(i)}}\wedge T } \ind \left( I_t^{(i)} = k , T_k^{(i)}(t-1) < \frac{4 \alpha \log t}{\Delta_k^2} \right) \leq \frac{ 4 \alpha \log T }{ \Delta_k^2 } ,
\end{equation}
which clearly implies the following bound for \eqref{eqInterAsTermLargeTau}:
\begin{equation}\label{eqInterAsTermLargeTauInitB}
\E \left[ \ind ( \tau^{(i)}> \ceil{ T^{\gamma_i/\beta} }  ) \sum_{t = 1 + A_{\tau_{\text{spr}}}}^{ A_{\tau^{(i)}}\wedge T } \ind \left( I_t^{(i)} = k , T_k^{(i)}(t-1) < \frac{4 \alpha \log t}{\Delta_k^2} \right) \right] \leq \frac{ 4 \alpha \log T }{ \Delta_k^2 } \P ( \tau^{(i)}> \ceil{ T^{\gamma_i/\beta} }  ) .
\end{equation}
For the remaining probability term, by Markov's inequality, we have
\begin{equation}
\P ( \tau^{(i)}> \ceil{ T^{\gamma_i/\beta} }  )  \leq \P \left( \tau^{(i)}> \ceil{ T^{\gamma_i/\beta} } , A_{\tau_{\text{spr}}} < \frac{4 \alpha K \log T}{\Delta_2} \right) + \frac{ \Delta_2 \E [A_{\tau_{\text{spr}}}]  }{ 4 \alpha K \log T } .
\end{equation}
Hence, combining the previous two inequalities, and since $\Delta_2 \leq \Delta_k$, we can bound \eqref{eqInterAsTermLargeTau} by
\begin{equation}\label{eqInterAsTermLargeTauFin}
\frac{ 4 \alpha \log T }{ \Delta_2 \Delta_k } \P \left( \tau^{(i)}> \ceil{ T^{\gamma_i/\beta} } , A_{\tau_{\text{spr}}} < \frac{4 \alpha K \log T}{\Delta_2} \right) + \frac{\E [A_{ \tau_{\text{spr}}} ] }{ K \Delta_k } .
\end{equation}
Finally, combining these bounds, then multiplying by $\Delta_k$, summing over $k \in \underline{S}^{(i)}$, and using $\Delta_k \leq 1$ and $| \underline{S}^{(i)} | \leq K$, we obtain the desired result.
\end{proof}

To bound \eqref{eqIntRemaining}, we consider two cases defined in terms of the following inequalities:
\begin{gather} 
4 \alpha K \log (T) / \Delta_2 < \floor{ \theta_{ \ceil{T^{\gamma_i/\beta}} } }^\beta , \quad 4 \alpha K \log (T) / \Delta_2 < \kappa_{ \ceil{T^{\gamma_i/\beta}} } ,  \label{eqInterLargeT1} \\
\frac{4 \alpha \log A_j }{ \Delta_2^2 }  < \ceil{ \kappa_j } - 1\ \forall\  j \geq  \ceil{T^{\gamma_i/\beta}}  , \quad \log (T) \sum_{j = \ceil{ T^{\gamma_i/\beta} } }^\infty ( \ceil{\kappa_j} - 1 )^{ 3-2\alpha }  < \frac{ \Delta_2 ( 2 \alpha - 3 ) }{ 8 \alpha K^2  } . \label{eqInterLargeT2}
\end{gather}

Roughly speaking, when all of these inequalities hold, then $T$ is large enough to ensure that the event $\{ \tau^{(i)} = \Omega( \text{poly}(T) ) , A_{\tau_{\text{spr}}} = O(\log T) \}$ in \eqref{eqIntRemaining} is unlikely. The next claim makes this precise.
\begin{clm} \label{clmIntNonStickLargeT}
Under the assumptions of Theorem \ref{thmProposed}, for any $i \in [n]$, $\gamma_i \in (0,1)$, and $T \in \N$ such that all of the inequalities in \eqref{eqInterLargeT1}-\eqref{eqInterLargeT2} hold, we have 
\begin{equation}
\frac{4 \alpha K \log T}{\Delta_2} \P \left( \tau^{(i)}> \ceil{ T^{\gamma_i/\beta} } , A_{\tau_{\text{spr}}} < \frac{4 \alpha K \log T}{\Delta_2} \right)  \leq  1 .
\end{equation}
\end{clm}
\begin{proof}
If $T = 1$, the left side is zero and the bound is immediate, so we assume $T>1$. First note that if $A_{\tau_{\text{spr}}} < 4 \alpha K \log(T) / \Delta_2$, then since $A_{\tau_{\text{spr}}} = \ceil{ \tau_{\text{spr}}^\beta } \geq \tau_{\text{spr}}^\beta$ by definition and $4 \alpha K \log(T) / \Delta_2 < \floor{ \theta_{ \ceil{T^{\gamma_i/\beta} } } }^\beta \wedge \kappa_{ \ceil{T^{\gamma_i/\beta}} }$ by \eqref{eqInterLargeT1}, $\tau_{\text{spr}} < \floor{ \theta_{ \ceil{T^{\gamma_i/\beta} } } } \wedge \kappa_{ \ceil{T^{\gamma_i/\beta}} }^{1/\beta}$. By definition $\theta_{ \ceil{T^{\gamma_i/\beta} } } = ( \ceil{T^{\gamma_i/\beta} } / 3 )^{\rho_1}$ with $\rho_1 \in (0,1)$, this further implies $\tau_{\text{spr}} < \ceil{T^{\gamma_i/\beta}}$. Thus, when $\tau^{(i)}> \ceil{ T^{\gamma_i/\beta} }$ and $A_{\tau_{\text{spr}}} < 4 \alpha K \log(T) / \Delta_2$, we have $\tau_{\text{spr}} < \tau^{(i)}$, which by definition of $\tau^{(i)}$ implies $\tau^{(i)} = \tau_{\text{blk}}^{(i)}$. Therefore, we can write
\begin{equation}
\P \left( \tau^{(i)}> \ceil{ T^{\gamma_i/\beta} } , A_{\tau_{\text{spr}}} < \frac{4 \alpha K \log T}{\Delta_2} \right) \leq \P ( \tau_{\text{blk}}^{(i)} > \ceil{ T^{\gamma_i/\beta} } , \tau_{\text{spr}} < \floor{ \theta_{ \ceil{T^{\gamma_i/\beta} } } } \wedge \kappa_{ \ceil{T^{\gamma_i/\beta}} }^{1/\beta} ) .
\end{equation}
Now by definition, $\tau_{\text{spr}} < \floor{ \theta_{ \ceil{T^{\gamma_i/\beta} } } }$ implies that $B_j^{(i)} = 1\ \forall\ j \geq \floor{ \theta_{ \ceil{T^{\gamma_i/\beta} } } }$. Also by definition, $\tau_{\text{blk}}^{(i)} > \ceil{ T^{\gamma_i/\beta} }$ implies that for some $j \geq \ceil{ T^{\gamma_i/\beta} }$ and $k > 1$, $R_{j-1}^{(i)} = k$ but $H_{j-1}^{(i)} \notin P_j^{(i)} \setminus P_{j-1}^{(i)}$. Thus,
\begin{align}
& \P ( \tau_{\text{blk}}^{(i)} > \ceil{ T^{\gamma_i/\beta} } , \tau_{\text{spr}} < \floor{ \theta_{ \ceil{T^{\gamma_i/\beta} } } } \wedge \kappa_{ \ceil{T^{\gamma_i/\beta}} }^{1/\beta} ) \\
& \quad \leq  \sum_{k=2}^K \sum_{j = \ceil{ T^{\gamma_i/\beta} } }^\infty \P ( R_{j-1}^{(i)} = k , H_{j-1}^{(i)} \notin P_j^{(i)} \setminus P_{j-1}^{(i)} , \tau_{\text{spr}} < \floor{ \theta_{ \ceil{T^{\gamma_i/\beta} } } } \wedge \kappa_{ \ceil{T^{\gamma_i/\beta}} }^{1/\beta} ) .
\end{align}
Now fix $k$ and $j$ as in the double summation. Again using $\tau_{\text{spr}} < \floor{ \theta_{ \ceil{T^{\gamma_i/\beta} } } } \Rightarrow B_j^{(i)} = 1\ \forall\ j \geq \floor{ \theta_{ \ceil{T^{\gamma_i/\beta} } } }$, the blocking rules implies that if $\tau_{\text{spr}} < \floor{ \theta_{ \ceil{T^{\gamma_i/\beta} } } }$, $R_{j-1}^{(i)} = k$, and $H_{j-1}^{(i)} \notin P_j^{(i)} \setminus P_{j-1}^{(i)}$, then $T_k^{(i)}(A_j) > \kappa_j$. Since $T_k^{(i)}(A_j) \in \N$, this means $T_k^{(i)}(A_j) \geq \ceil{ \kappa_j }$. Hence, there must exist some $t \in \{ \ceil{\kappa_j} , \ldots , A_j \}$ such that $T_k^{(i)}(t-1) \geq \ceil{ \kappa_j } - 1$ and $I_t^{(i)} = k$. Thus, taking another union bound,
\begin{align}
&  \P ( R_{j-1}^{(i)} = k , H_{j-1}^{(i)} \notin P_j^{(i)} \setminus P_{j-1}^{(i)} , \tau_{\text{spr}} < \floor{ \theta_{ \ceil{T^{\gamma_i/\beta} } } } \wedge \kappa_{ \ceil{T^{\gamma_i/\beta}} }^{1/\beta} )\\
& \quad \leq \sum_{ t = \ceil{\kappa_j} }^{A_j} \P ( T_k^{(i)}(t-1) \geq \ceil{ \kappa_j } - 1 , I_t^{(i)} = k , \tau_{\text{spr}} <  \kappa_{ \ceil{T^{\gamma_i/\beta}} }^{1/\beta}  ) .
\end{align}
Next, note $\tau_{\text{spr}} <  \kappa_{ \ceil{T^{\gamma_i/\beta}} }^{1/\beta}$ implies that for any $j \geq \ceil{ T^{\gamma_i/\beta} }$ and $t \geq \ceil{\kappa_j}$, we have $t \geq \ceil{ \kappa_{ \ceil{ T^{\gamma_i/\beta} } } } \geq \ceil{ \tau_{\text{spr}}^\beta } = A_{ \tau_{\text{spr}} }$, so $1 \in S_{A^{-1}(t)}^{(i)}$ by definition of $\tau_{\text{spr}}$. Therefore, for any $t \in \{ \ceil{\kappa_j} , \ldots , A_j \}$, 
\begin{equation}
\P ( T_k^{(i)}(t-1) \geq \ceil{ \kappa_j } - 1 , I_t^{(i)} = k , \tau_{\text{spr}} <  \kappa_{ \ceil{T^{\gamma_i/\beta}} }^{1/\beta}  )  \leq  \P ( T_k^{(i)}(t-1) \geq \ceil{ \kappa_j } - 1 , 1 \in S_{A^{-1}(t)}^{(i)} , I_t^{(i)} = k ) .
\end{equation}
Now let $k_1 = k$, $k_2 = 1$, and $\ell = \ceil{ \kappa_j } - 1$. Then $\mu_{k_2} - \mu_{k_1} = \Delta_k \geq \sqrt{ 4 \alpha \log(t) / ( \ceil{\kappa_j} - 1 ) }$ by definition and \eqref{eqInterLargeT2}, respectively. Therefore, we can use Corollary \ref{corBanditTailEll} from Appendix \ref{appBanditRes} to obtain
\begin{align}
\P ( T_k^{(i)}(t-1) \geq \ceil{ \kappa_j } - 1 , 1 \in S_{A^{-1}(t)}^{(i)} , I_t^{(i)} = k ) \leq 2 t^{2(1-\alpha)} .
\end{align}
Combining the above five inequalities, then using Claim \ref{clmPolySeries} from Appendix \ref{appBasicRes} and \eqref{eqInterLargeT2}, we obtain
\begin{align*}
\P \left( \tau^{(i)}> \ceil{ T^{\gamma_i/\beta} } , A_{\tau_{\text{spr}}} < \frac{4 \alpha K \log T}{\Delta_2} \right) & \leq 2 K \sum_{ j = \ceil{ T^{\gamma_i / \beta} } }^\infty \sum_{ t = \ceil{\kappa_j} }^\infty t^{2(1-\alpha) } \\
& \leq \frac{ 2 K }{  2 \alpha - 3  } \sum_{ j = \ceil{ T^{\gamma_i / \beta} } }^\infty ( \ceil{\kappa_j} - 1 )^{ 3-2\alpha} \leq \frac{ \Delta_2 }{ 4 \alpha K \log T } , 
\end{align*}
so multiplying both sides by $4 \alpha K \log(T) / \Delta$ completes the proof.
\end{proof}

On the other hand, when \eqref{eqInterLargeT1}-\eqref{eqInterLargeT2} fails, we can show that $T$ is bounded, and thus we bound the $\log T$ term in \eqref{eqIntRemaining} by a constant and the probability term by $1$, as shown in the following claim.
\begin{clm} \label{clmIntNonStickSmallT}
Under the assumptions of Theorem \ref{thmProposed}, there exists a constant $C_{\ref{clmIntNonStickSmallT}} > 0$ such that, for any $i \in [n]$, $\gamma_i \in (0,1)$, and $T \in \N$ for which any of the inequalities in \eqref{eqInterLargeT1}-\eqref{eqInterLargeT2} fails, we have
\begin{equation}
\frac{4 \alpha K \log T}{\Delta_2} \P \left( \tau^{(i)}> \ceil{ T^{\gamma_i/\beta} } , A_{\tau_{\text{spr}}} < \frac{4 \alpha K \log T}{\Delta_2} \right)  \leq \frac{ 4 C_{\ref{clmIntNonStickSmallT}} \alpha K }{ \Delta_2 \gamma_i } \log \frac{C_{\ref{clmIntNonStickSmallT}} K}{ \Delta_2 \gamma_i } . 
\end{equation}
\end{clm}
\begin{proof}
By Claims \ref{clmSmallTtheta}-\ref{clmSmallTSum} in Appendix \ref{appIntCalc}, we can set $C_{\ref{clmIntNonStickSmallT}} = \max_{i \in \{ \ref{clmSmallTtheta} , \ldots , \ref{clmSmallTSum} \} } C_i$ to ensure that, if any of the inequalities in \eqref{eqInterLargeT1}-\eqref{eqInterLargeT2} fail, then $\log T \leq ( C_{\ref{clmIntNonStickSmallT}} / \gamma_i ) \log ( C_{\ref{clmIntNonStickSmallT}}K / (\Delta_2 \gamma_i) )$. The claim follows after upper bounding the probability by $1$.
\end{proof}

Finally, Lemma \ref{lemIntNonSticky} follows by combining the previous three claims.

\subsection{Proof of Corollary \ref{corRegular}} \label{appProofCorRegular}

As discussed in the proof sketch, we will couple with the noiseless process. We define this process as follows: let $\{ \underline{H}_j^{(i)} \}_{j=1}^\infty$ be i.i.d.\ $\text{Uniform}(N_{\text{hon}}(i))$ random variables for each $i \in \N$, and
\begin{equation} \label{eqPullSets}
\underline{\mathcal{I}}_0 = \{ i^\star \} , \quad \underline{\mathcal{I}}_j  = \underline{\mathcal{I}}_{j-1} \cup \{ i \in [n] \setminus \underline{\mathcal{I}}_{j-1} :  \underline{H}_j^{(i)} \in \underline{\mathcal{I}}_{j-1}  \}\ \forall\ j \in \N .
\end{equation}
For the coupling, we first define
\begin{equation}
\sigma_0 = 0 , \quad \sigma_l = \inf \left\{ j > \sigma_{l-1} : \min_{i \in [n] } \sum_{j' = 1 + \sigma_{l-1} }^j \bar{Y}_{j'}^{(i)} \geq 1 \right\}\ \forall\ l \in \N .
\end{equation} 
Next, for each $i \in [n]$ and $l \in \N$, let $Z_l^{(i)} = \min \{ j \in \{ 1 + \sigma_{l-1} , \ldots , \sigma_l \} : \bar{Y}_{j}^{(i)} = 1 \}$. Note this set is nonempty, and since $Z_l^{(i)}$ is a deterministic function of $\{ \bar{Y}_j \}_{j=1}^\infty$, which is independent of $\{ \bar{H}_j^{(i)} \}_{j=1}^\infty$, $\bar{H}_{Z_l^{(i)}}^{(i)}$ is $\text{Uniform}(N_{\text{hon}}(i))$ for each $l \in \N$. Hence, we can set
\begin{equation*}
\underline{H}_j^{(i)} = \begin{cases} \bar{H}_{Z_l^{(i)}}^{(i)} , & \text{if } j = Z_l^{(i)} \text{ for some } l \in \N \\ \text{Uniform}(N_{\text{hon}}(i)) , & \text{if } j \notin \{ Z_l^{(i)} \}_{l=1}^\infty \end{cases}
\end{equation*}
without changing the distribution of $\{ \underline{\mathcal{I}}_j \}_{j=0}^\infty$. This results in a coupling where the noiseless process dominates the noisy one, in the following sense.

\begin{clm} \label{clmTwoRumorMono}
For the coupling described above, $\underline{\mathcal{I}}_j \subset \bar{\mathcal{I}}_{ \sigma_j }$ for any $j \geq 0$.
\end{clm}
\begin{proof}
We use induction on $j$. For $j = 0$, we simply have $\underline{\mathcal{I}}_j = \{ i^\star \} = \bar{\mathcal{I}}_j = \bar{\mathcal{I}}_{ \sigma_j }$. Now assume $\underline{\mathcal{I}}_{j-1} \subset \bar{\mathcal{I}}_{ \sigma_{j-1} }$; we aim to show that if $i \in \underline{\mathcal{I}}_j$, then $i \in \bar{\mathcal{I}}_{ \sigma_j }$. We consider two cases, the first of which is straightforward: if $i \in \underline{\mathcal{I}}_{j-1}$, then $i \in \bar{\mathcal{I}}_{ \sigma_{j-1} }$ by the inductive hypothesis, so since $\sigma_{j-1} < \sigma_j$ by definition and $\{ \bar{\mathcal{I}}_{j'} \}_{j'=0}^\infty$ increases monotonically, we obtain $i \in \bar{\mathcal{I}}_{ \sigma_{j} }$, as desired.

For the second case, we assume $i \in [n] \setminus \underline{\mathcal{I}}_{j-1}$ and $\underline{H}_j^{(i)} \in \underline{\mathcal{I}}_{j-1}$. Set $j' = Z_j^{(i)}$ and recall $j' \in \{ 1 + \sigma_{j-1} , \ldots , \sigma_j \}$ by definition. From the coupling above, we know $\bar{Y}_{ j' }^{(i)} = 1$ and $\bar{H}_{ j' }^{(i)} = \underline{H}_{ j }^{(i)}$. Since $\underline{H}_j^{(i)} \in \underline{\mathcal{I}}_{j-1}$ in the present case, we have $\bar{H}_{ j' }^{(i)}  \in \underline{\mathcal{I}}_{j-1}$ as well. Hence, because $\underline{\mathcal{I}}_{j-1} \subset \bar{\mathcal{I}}_{ \sigma_{j-1} }$ by the inductive hypothesis, $j' - 1 \geq \sigma_{j-1}$ by definition, and $\{ \bar{\mathcal{I}}_{j''} \}_{j''=0}^\infty$ is increasing, we obtain $\bar{H}_{ j' }^{(i)}  \in \bar{\mathcal{I}}_{j'-1}$. We have thus shown $\bar{Y}_{ j' }^{(i)} = 1$ and $\bar{H}_{ j' }^{(i)}  \in \bar{\mathcal{I}}_{j'-1}$, so $i \in \bar{\mathcal{I}}_{j'}$ by Definition \ref{defnRumor} Finally, using $j' \leq \sigma_j$ and again appealing to monotonocity, we conclude that $i \in \bar{\mathcal{I}}_{ \sigma_{j} }$.
\end{proof}

We can now relate the rumor spreading times of the two processes. In particular, let $\bar{\tau}_{\text{spr}} = \inf \{ j \in \N : \bar{\mathcal{I}}_j = [n] \}$ (as in Definition \ref{defnRumor}) and $\underline{\tau}_{\text{spr}} = \inf \{ j \in \N : \underline{\mathcal{I}}_j = [n] \}$. We then have the following.
\begin{clm} \label{clmTwoRumorProb}
For any $j \in \{3,4,\ldots\}$ and $\iota > 1$, we have $\P ( \bar{\tau}_{\text{spr}} > \iota  j \log(j) / \Upsilon  ) \leq \P ( \underline{\tau}_{\text{spr}} > j ) + 27 n j^{1-\iota}$.
\end{clm}
\begin{proof}
Let $h(0) = 0$ and $h(j') = \iota j' \log(j) / \Upsilon$ for each $j' \in \N$. Then clearly
\begin{equation} \label{eqTwoRumorProb}
\P ( \bar{\tau}_{\text{spr}} > \iota j \log(j) / \Upsilon  ) = \P ( \bar{\tau}_{\text{spr}} > h(j) , \sigma_j \leq h(j) ) + \P (  \bar{\tau}_{\text{spr}} > h(j) , \sigma_j > h(j) ) .
\end{equation}
For the first term in \eqref{eqTwoRumorProb}, by definition of $\bar{\tau}_{\text{spr}}$ and Claim \ref{clmTwoRumorMono}, we have
\begin{equation} \label{eqTwoRumorProbFirst}
\{ \bar{\tau}_{\text{spr}} > h(j) , \sigma_j \leq h(j) \} \subset \{ \bar{\tau}_{\text{spr}} > \sigma_j \} = \{ \bar{\mathcal{I}}_{\sigma_j} \neq [n] \} \subset \{ \underline{\mathcal{I}}_j \neq [n] \} = \{ \underline{\tau}_{\text{spr}} > j \} .
\end{equation}
For the second term in \eqref{eqTwoRumorProb}, we first observe that for any $j' \in \N$,
\begin{equation*}
\floor{h(j')} - \ceil{h(j'-1)} - 1 \geq h(j') - h(j'-1) - 3 = ( \iota \log(j) / \Upsilon ) - 3 > 0 ,
\end{equation*}
where the last inequality holds by assumption on $j$ and $\iota$. Thus, by the union bound, we can write
\begin{align}
 \P ( \sigma_{j'} > h(j') , \sigma_{j'-1} \leq h(j'-1) ) & \leq \sum_{i=1}^n \P ( \cap_{j'' = 1+ \ceil{h(j'-1)} }^{ \floor{h(j')} } \{ \bar{Y}_{j''}^{(i)} = 0 \} ) = n ( 1 - \Upsilon )^{ \floor{h(j')} - \ceil{h(j'-1)} - 1 } \\
&  \leq n ( 1 - \Upsilon )^{ ( \iota \log(j) / \Upsilon ) - 3 } \leq n \exp ( - \iota \log(j) + 3 \Upsilon ) < 27 n j^{-\iota}  .
\end{align}
Hence, because $\sigma_0 = 0 = h(0)$, we can iterate this argument to obtain that for any $j' \in \N$,
\begin{align} 
\P ( \sigma_{j'} > h(j') ) & \leq \P ( \sigma_{j'} > h(j') , \sigma_{j'-1} \leq h(j'-1) ) + \P (  \sigma_{j'-1} > h(j'-1) ) \\
& \leq 27 n j^{-\iota} + \P (  \sigma_{j'-1} > h(j'-1) ) \leq \cdots \leq 27 n j' j^{-\iota} .
\end{align}
In particular, $\P ( \sigma_j > h(j) ) \leq 27 n j^{1-\iota}$. Combining with \eqref{eqTwoRumorProb} and \eqref{eqTwoRumorProbFirst} completes the proof.
\end{proof}

To bound the tail of $\underline{\tau}_{\text{spr}}$, we will use the following result.
\begin{clm}[Lemma 19 from \cite{chawla2020gossiping}] \label{clmNoiselessTail}
Under the assumptions of Corollary \ref{corRegular}, there exists an absolute constant $C_{\ref{clmNoiselessTail}} > 0$ such that, for any $h \in \N$, we have $\P( \underline{\tau}_{\text{spr}} \geq C_{\ref{clmNoiselessTail}} h \log (n) / \phi ) \leq n^{-4h}$.
\end{clm}

Using the previous two claims, we can prove a tail bound for $\bar{\tau}_{\text{spr}}$.
\begin{clm} \label{clmNoisyTail}
Under the assumptions of Corollary \ref{corRegular}, there exists an absolute constant $C_{\ref{clmNoisyTail}} > 0$ such that, for any $h \in \{3,4,\ldots\}$, we have $\P ( \bar{\tau}_{\text{spr}} \geq \xi(h) ) \leq 56 \cdot 2^{-h}$, where we define
\begin{equation*}
\xi(h) =  C_{\ref{clmNoisyTail}} ( \log n )^2 h^3 \log ( C_{\ref{clmNoisyTail}} \log ( n ) / \phi ) / (\phi \Upsilon ) .
\end{equation*}
\end{clm}
\begin{proof}
Since $G_{\text{hon}}$ is $d$-regular with $d \geq 2$ by assumption, we have $n \geq 2$ as well. Therefore, setting $C_{\ref{clmNoisyTail}} = ( 2 C_{\ref{clmNoiselessTail}} ) \vee ( (e+1) / \log(2) )$, we know $\log ( C_{\ref{clmNoisyTail}}  \log (n) / \phi ) \geq 1$, which implies
\begin{equation*}
(h-1) \log ( C_{\ref{clmNoisyTail}} \log (n) / \phi ) \geq h  - 1 \geq \log h \quad \Rightarrow \quad h \log ( C_{\ref{clmNoisyTail}} \log (n) / \phi ) \geq \log ( C_{\ref{clmNoisyTail}} h \log (n) / \phi ) .
\end{equation*}
Consequently, if we define $\iota = h \log n$ and $j = \floor{C_{\ref{clmNoisyTail}} h \log(n) / \phi}$, we can write
\begin{align}
\xi(h) & = \left( h \log n \right)  \left( C_{\ref{clmNoisyTail}} h \log(n) / \phi \right)  \left( h \log ( C_{\ref{clmNoisyTail}} \log (n) / \phi ) \right) / \Upsilon \\
&  \geq \left( h \log n \right)  \left( C_{\ref{clmNoisyTail}} h \log(n) / \phi \right)  \log ( C_{\ref{clmNoisyTail}} h \log (n) / \phi ) / \Upsilon \geq \iota j \log(j) / \Upsilon .
\end{align}
Because $C_{\ref{clmNoisyTail}} \geq (e+1) / \log(2)$, $h \geq 3$, and $n \geq 2$, we also know
\begin{equation} \label{eqNoisyTailJlow}
j \geq C_{\ref{clmNoisyTail}} \cdot h \cdot \log(n)  - 1 \geq ( (e+1) / \log(2) ) \cdot 3 \cdot \log(2) - 1 = 3 (e+1) - 1 > 3 e > e^2 .
\end{equation}
Hence, $j \in \{3,4,\ldots\}$; combined with $\iota \geq 3 \log 2 > 1$, we can apply Claim \ref{clmTwoRumorProb} to obtain
\begin{equation} \label{eqNoisyTailJApp1}
\P \left( \bar{\tau}_{\text{spr}} > \xi(h) \right) \leq \P \left( \bar{\tau}_{\text{spr}} > 
\iota j \log (j) / \Upsilon \right) \leq \P \left( \underline{\tau}_{\text{spr}} > j \right) + 27 n j^{1-\iota} .
\end{equation}
On the other hand, \eqref{eqNoisyTailJlow} implies $C_{\ref{clmNoisyTail}} h \log(n) / \phi \geq 2$, so by definition of $C_{\ref{clmNoisyTail}}$,
\begin{equation*}
j \geq (C_{\ref{clmNoisyTail}} h \log(n) / \phi ) - 1 \geq ( C_{\ref{clmNoisyTail}} h \log(n) /  \phi ) / 2 = ( C_{\ref{clmNoisyTail}} / 2 ) h \log(n) / \phi \geq C_{\ref{clmNoiselessTail}} h \log(n) / \phi .
\end{equation*}
Therefore, by Claim \ref{clmNoiselessTail}, we know that 
\begin{equation*}
\P \left( \underline{\tau}_{\text{spr}} > j \right) \leq \P \left( \underline{\tau}_{\text{spr}} >  C_{\ref{clmNoiselessTail}} h \log (n) / \phi \right) \leq n^{-4h} < n^{1-h} .
\end{equation*}
Finally, notice that $\iota = h \log n \geq 3 \log 2 > 2$, so $1-\iota < -\iota/2$, thus by \eqref{eqNoisyTailJlow},
\begin{equation*}
27 n j^{1-\iota} < 27 n j^{-\iota/2} = 27 n \sqrt{j}^{-\iota} < 2 7 n \exp ( - \iota ) = 27 n^{1-h}  .
\end{equation*}
Hence, substituting the previous two inequalities into \eqref{eqNoisyTailJApp1} and using $n \geq 2$ completes the proof.
\end{proof}

We now bound $\E [A_{2 \bar{\tau}_{\text{spr}}}]$. First, we define
\begin{equation*}
\chi = \ceil{C_{\ref{clmNoisyTail}} ( \log n )^2 \log ( C_{\ref{clmNoisyTail}} \log(n) / \phi ) / ( \phi \Upsilon ) } , \quad C = 8 \beta / \log 2 , \quad h_\star = \ceil{C \log C} \vee 3 .
\end{equation*}
Notice that for any $h \geq h_\star  \geq C \log C$, we have $2^{-h} \leq h^{-4 \beta}$ (else, we can invoke Claim \ref{clmLogTrick} from Appendix \ref{appBasicRes} with $x = h$, $y=1$, and $z = C/2$ to obtain $h < C \log C$, a contradiction). We write
\begin{equation} \label{eqCondInit}
\sum_{j=1}^\infty ( A_j - A_{j-1} ) \P ( \bar{\tau}_{\text{spr}} \geq j/2 ) \leq A_{ 2 \chi h_\star^3  } + \sum_{h=h_\star}^\infty \sum_{j =  2 \chi h^3 + 1 }^{ 2 \chi (h+1)^3 }  ( A_j - A_{j-1} ) \P ( \bar{\tau}_{\text{spr}} \geq j/2 ) .
\end{equation}
Now for any $h \geq h_\star$ and $j \geq 2 \chi h^3+1$, we use $2^{-h} \leq h^{-4\beta}$, the definition of $\chi$, and Claim \ref{clmNoisyTail} to write
\begin{equation*}
\P ( \bar{\tau}_{\text{spr}} \geq j/2 ) \leq \P ( \bar{\tau}_{\text{spr}} \geq \chi h^3 ) \leq \P ( \bar{\tau}_{\text{spr}} \geq \xi(h) ) \leq 56 \cdot 2^{-h} \leq 56 h^{-4 \beta} .
\end{equation*}
Therefore, for any such $h$, we obtain
\begin{equation*}
\sum_{j =  2 \chi h^3 + 1 }^{ 2 \chi (h+1)^3 }  ( A_j - A_{j-1} ) \P ( \bar{\tau}_{\text{spr}} \geq j/2 ) \leq 56 h^{-4 \beta} \sum_{j =  2 \chi h^3 + 1 }^{ 2 \chi (h+1)^3 }  ( A_j - A_{j-1} ) \leq 56 h^{-4 \beta} A_{ 2 \chi (h+1)^3 } .
\end{equation*}
Furthermore, by Claim \ref{clmPropAj} in Appendix \ref{appBasicRes} and $h \geq 1$, we know that
\begin{equation*}
A_{ 2 \chi (h+1)^3 } \leq e^{2\beta} ( 2 \chi (h+1)^3 )^{\beta} \leq e^{2 \beta} ( 2 \chi ( 2 h )^3 )^\beta = ( 4 e )^{2 \beta} \chi^\beta h^{3\beta}\ \forall\ h \geq 1 .
\end{equation*}
Therefore, by the previous two inequalities and \eqref{eqCondInit}, and since $h_\star \geq 3$, we have shown
\begin{equation*}
\sum_{j=1}^\infty ( A_j - A_{j-1} ) \P ( \bar{\tau}_{\text{spr}} \geq j/2 ) \leq ( 4 e )^{2 \beta} \chi^\beta \left( h_\star^{3 \beta} + 56 \sum_{h=h_\star}^\infty h^{-\beta} \right) \leq ( 4 e )^{2 \beta} \chi^\beta \left( h_\star^{3 \beta} + \frac{56}{\beta-1} \right) ,
\end{equation*}
where the second inequality follows from Claim \ref{clmPolySeries} from Appendix \ref{appBasicRes}, $h_\star \geq 3$, and $\beta> 1$. Because $h_\star$ is a constant, the right side is $O(\chi^\beta)$. Therefore, we have shown
\begin{equation*}
\E [_{ 2 \bar{\tau}_{\text{spr}} } ]= \E \left[ \sum_{j=1}^\infty (A_j-A_{j-1}) \ind ( 2 \bar{\tau}_{\text{spr}} \geq j ) \right] = \sum_{j=1}^\infty  (A_j-A_{j-1}) \P (  \bar{\tau}_{\text{spr}} \geq j / 2 ) = O ( \chi^\beta ) .
\end{equation*}
Hence, by definition of $\chi$, we obtain $\E [A_{ 2 \bar{\tau}_{\text{spr}} }] = O ( ( ( \log n )^2 \log( \log(n) / \phi ) / ( \phi \Upsilon ) )^\beta )$. Combining this bound with Theorem \ref{thmProposed} completes the proof of Corollary \ref{corRegular}.

\subsection{Proof of Corollary \ref{corNoMal}}

Similar to the analysis in Appendix \ref{appProofThmProposed}, we can use the decomposition $R_T^{(i)} = \sum_{h=1}^4 R_{T,h}^{(i)}$, along with Lemmas \ref{lemEarly} and \ref{lemLateSticky}, to bound regret as follows:
\begin{align} \label{eqNoMalInit}
R_T^{(i)} \leq \sum_{k \in \overline{S}^{(i)}}  \frac{4 \alpha \log  T }{\Delta_k} + \frac{ 4 ( \alpha-1 ) | \overline{S}^{(i)} | }{ 2 \alpha - 3 } + R_{T,3}^{(i)} + R_{T,4}^{(i)} + \E [A_{\tau_{\text{spr}}}] .
\end{align}
Next, for each $k \in \underline{S}^{(i)}$, let $Y_k = \ind(\cup_{j=\tau_{\text{spr}}}^\infty \{k\in S_j^{(i)}\})$ be the indicator that $k$ was active after $A_{\tau_{\text{spr}}}$. Then as in the proof of Lemma \ref{lemLateSticky}, we can use Claim \ref{clmBanditPlays} and Corollary \ref{corBanditTailDelta} from Appendix \ref{appBanditRes} to write
\begin{equation} \label{eqNoMalYk}
\E \left[ \sum_{t = 1 + A_{\tau_{\text{spr}}}}^T \ind \left( I_t^{(i)} = k \right) \right]  \leq \E [ Y_k ] \frac{ 4 \alpha \log T}{ \Delta_k^2} + \frac{ 4(\alpha-1)}{2\alpha-3} .
\end{equation}
(The only difference from the proof of Lemma \ref{lemLateSticky} is that, when applying Claim \ref{clmBanditPlays}, we write
\begin{align*}
\sum_{t = 1 + A_{\tau_{\text{spr}}}}^T \ind \left( I_t^{(i)} = k , T_k^{(i)}(t-1) < \frac{4 \alpha \log t}{\Delta_k^2} \right) \leq Y_k \frac{ 4 \alpha \log T}{ \Delta_k^2} ,
\end{align*}
where we can multiply by $Y_k$ because the left side is also zero when $Y_k = 0$.) Combining \eqref{eqNoMalYk} with the definitions of $R_{T,3}^{(i)}$ and $R_{T,4}^{(i)}$ and using $\Delta_k \leq 1$, we thus obtain
\begin{equation*}
R_{T,3}^{(i)} + R_{T,4}^{(i)} \leq \E \left[  \sum_{k \in \underline{S}^{(i)}} Y_k \frac{4 \alpha \log T}{\Delta_k} \right] + \frac{ 4 ( \alpha-1 ) | \underline{S}^{(i)} | }{ 2 \alpha - 3 }  .
\end{equation*}
We claim, and will return to prove, that when $d_{\text{mal}}(i) = 0$,
\begin{equation} \label{eqNoMalKey}
\sum_{k \in \overline{S}^{(i)}}  \frac{ 1 }{\Delta_k}  + \sum_{k \in \underline{S}^{(i)}} \frac{Y_k}{\Delta_k} \leq \sum_{k=2}^{S+2} \frac{1}{\Delta_k} .
\end{equation}
Assuming \eqref{eqNoMalKey} holds, we can combine the previous two inequalities and substitute into \eqref{eqNoMalInit} to obtain $R_T^{(i)} \leq 4 \alpha \log(T) \sum_{k=2}^{S+2} \Delta_k^{-2} + O(K) + \E [A_{\tau_{\text{spr}}}]$. Bounding $\E [A_{\tau_{\text{spr}}}]$ as in Lemma \ref{lemEarly} yields the sharper version of Theorem \ref{thmProposed}, and further bounding $\E [A_{2\bar{\tau}_{\text{spr}}}]$ as in Appendix \ref{appProofCorRegular} sharpens Corollary \ref{corNoMal}.

To prove \eqref{eqNoMalKey}, we first show $\sum_{k \in \underline{S}^{(i)}} Y_k \leq 1 + \ind ( 1 \in \hat{S}^{(i)} )$. Suppose instead that $\sum_{k \in \underline{S}^{(i)}} Y_k \geq 2 + \ind ( 1 \in \hat{S}^{(i)} ) \triangleq H$. Then we can find $H$ distinct arms $k_1 , \ldots , k_H \in \underline{S}^{(i)}$, and $H$ corresponding phases $j_h \geq \tau_{\text{spr}}$, such that $k_h$ was active at phase $j_h$ for each $h \in [H]$. Without loss of generality, we can assume each $j_h$ is minimal, i.e., $j_h = \min \{ j \geq \tau_{\text{spr}}: k_h \in S_j^{(i)} \}$. We consider two cases (which are exhaustive since $j_h \geq \tau_{\text{spr}}$) and show that both yield contradictions.
\begin{itemize}
\item $j_h = \tau_{\text{spr}}\ \forall\ h \in [H]$: We consider two further sub-cases.
\begin{itemize}
\item $1 \in \hat{S}^{(i)}$, i.e., the best arm is sticky. Then $H=3$, so $k_1,\ldots,k_3$ are all active at phase $\tau_{\text{spr}}$. But all of these arms are non-sticky and only two such arms are active per phase.
\item $1 \notin \hat{S}^{(i)}$. Here $k_1 , k_2$ are both active at phase $\tau_{\text{spr}}$, as is $1$ (by definition of $\tau_{\text{spr}}$). But since $k_1$ and $k_2$ are suboptimal, we again have three non-sticky active arms.
\end{itemize}
\item $\max_{h \in [H]} j_h > \tau_{\text{spr}}$: We can assume (after possibly relabeling) that $j_1 > \tau_{\text{spr}}$. Thus, by minimality of $j_1$, $k_1$ was not active at phase $j_1-1$ but became active at $j_1$, so it was recommended by some neighbor $i'$ at $j_1-1$. But since $d_{\text{mal}}(i) = 0$, $i'$ is honest, and since $j_1-1 \geq \tau_{\text{spr}}$, the best arm was most played for $i'$ in phase $j_1-1$, so $i'$ would not have recommended $k_1$.
\end{itemize}
Thus, $\sum_{k \in \underline{S}^{(i)}} Y_k \leq 1 + \ind ( 1 \in \hat{S}^{(i)} )$ holds. Combined with the fact that $|\overline{S}^{(i)}| = S - \ind ( 1 \in \hat{S}^{(i)} )$ by definition, at most $S+1$ terms are nonzero in the summations on the left side of \eqref{eqNoMalKey}. Since $\Delta_2 \leq \cdots \leq \Delta_K$ by the assumed ordering of the arm means, this completes the proof.

\subsection{Coarse analysis of the noisy rumor process} \label{appCoarseBound}

For completeness, we provide a coarser though more general bound for $\E [A_{2\bar{\tau}_{\text{spr}}}]$ than the one derived in Appendix \ref{appProofCorRegular}. To begin, let $\P_j'$ and $\E_j'$ denote probability and expectation conditioned on $\{ \bar{Y}_{j'}^{(i)} , \bar{H}_{j'}^{(i)} \}_{j'=1}^{j-1}$. For each $h \in [n]$, define the random phase $\bar{\tau}(h) = \inf \{ j \in \N : |\bar{\mathcal{I}}_j| = h \}$. Note that $\bar{\tau}(1) = 1$ and $\bar{\tau}_{\text{spr}} = \bar{\tau}(n)$. We then have the following tail bound.

\begin{clm} \label{clmTauH}
For any $l ,j \in \N$, we have $\P(\bar{\tau}(l) > l j ) \leq l ( 1 - \Upsilon / \bar{d}_{\text{hon}} )^j$.
\end{clm}
\begin{proof}
We use induction on $l$. For $l=1$, $\bar{\tau}(1) = 1$ ensures $\P(\bar{\tau}(l) > l j ) = 0$ for any $j \in \N$, so the bound is immediate. Next, assume the bound holds for $l \in \N$. We first write
\begin{equation*}
\P(\bar{\tau}(l+1) > (l+1) j ) \leq \P (\bar{\tau}(l+1) > (l+1) j , \bar{\tau}(l) \leq l j ) + \P ( \bar{\tau}(l) > l j  ) .
\end{equation*}
Thus, by the inductive hypothesis, it suffices to bound the first term by $( 1 - \Upsilon / \bar{d}_{\text{hon}} )^j$. We first write
\begin{equation*}
\P (\bar{\tau}(l+1) > (l+1) j , \bar{\tau}(l) \leq l j ) = \E [ \ind ( \bar{\tau}(l) \leq l j ) \P_{lj+1}' (  \bar{\tau}(l+1) > (l+1) j ) ] .
\end{equation*}
Now suppose $\bar{\tau}(l) \leq l j$. By Assumption \ref{assGraph}, we can find $i \in \bar{\mathcal{I}}_{lj}$ and $i' \notin \bar{\mathcal{I}}_{lj}$ such that $i \in N_{\text{hon}}(i')$. Then for $ \bar{\tau}(l+1) > (l+1) j$ to occur, it must be the case that, for each $j' \in \{ lj + 1 , \ldots , (l+1) j \}$, the event $\{ \bar{H}_{j'}^{(i')} = i , \bar{Y}_{j'}^{(i')} = 1 \}$ did \textit{not} occur. Therefore, we have
\begin{equation}
\P_{lj+1}' (  \bar{\tau}(l+1) > (l+1) j ) \leq \P_{lj+1}' (  \cap_{j'=lj+1}^{(l+1)j} \{ \bar{H}_{j'}^{(i')} = i , \bar{Y}_{j'}^{(i')} = 1 \}^C ) .
\end{equation}
By the law of total expectation, we can write
\begin{align}
& \P_{lj+1}' (  \cap_{j'=lj+1}^{(l+1)j} \{ \bar{H}_{j'}^{(i')} = i , \bar{Y}_{j'}^{(i')} = 1 \}^C ) \\
& \quad = \E_{lj+1}' [ \ind (  \cap_{j'=lj+1}^{(l+1)j-1} \{ \bar{H}_{j'}^{(i')} = i , \bar{Y}_{j'}^{(i')} = 1 \}^C ) ( 1 - \P_{(l+1)j}' ( \bar{H}_{(l+1)j}^{(i')} = i , \bar{Y}_{(l+1)j}^{(i')} = 1 ) ) ] .
\end{align}
Since $\bar{H}_{(l+1)j}^{(i')}$ is $\text{Uniform}(N_{\text{hon}}(i))$ and $\bar{Y}_{(l+1)j}^{(i')}$ is $\text{Bernoulli}(\Upsilon)$, we have
\begin{equation*}
 \P_{(l+1)j}' ( \bar{H}_{(l+1)j}^{(i')} = i , \bar{Y}_{(l+1)j}^{(i')} = 1 ) = \Upsilon / d_{\text{hon}}(i) \geq \Upsilon / \bar{d}_{\text{hon}} .
\end{equation*}
Therefore, combining the previous two expressions and iterating, we obtain
\begin{align*}
\P_{lj+1}' (  \cap_{j'=lj+1}^{(l+1)j} \{ \bar{H}_{j'}^{(i')} = i , \bar{Y}_{j'}^{(i')} = 1 \}^C ) & \leq \P_{lj+1}' (  \cap_{j'=lj+1}^{(l+1)j-1} \{ \bar{H}_{j'}^{(i')} = i , \bar{Y}_{j'}^{(i')} = 1 \}^C ) ( 1 - \Upsilon / \bar{d}_{\text{hon}} ) \\
& \leq \cdots \leq ( 1 - \Upsilon / \bar{d}_{\text{hon}} )^j . \qedhere
\end{align*}
\end{proof}

Next, we have a simple technical claim.
\begin{clm} \label{clmCoarseBeta}
Let $h_\dagger = ( 8 \beta \bar{d}_{\text{hon}} / \Upsilon ) \log ( 8 \beta \bar{d}_{\text{hon}} n / \Upsilon )$. Then for any $h \geq h_\dagger$, $\exp ( - h \Upsilon / \bar{d}_{\text{hon}}  ) \leq h^{-2\beta} / n$.
\end{clm}
\begin{proof}
If the claimed bound fails, we have $h < ( 2 \beta \bar{d}_{\text{hon}} / \Upsilon ) \log(h) + ( \bar{d}_{\text{hon}} / \Upsilon ) \log(n)$. Then since $( \bar{d}_{\text{hon}} / \Upsilon ) \log(n) < h_\dagger / 2 \leq h / 2$, we obtain $h < ( 4  \beta \bar{d}_{\text{hon}} / \Upsilon ) \log h$. Applying Claim \ref{clmLogTrick} from Appendix \ref{appBasicRes} with $x=h$, $y=1$, and $z = 4 \beta \bar{d}_{\text{hon}} / \Upsilon$, we obtain $h < 2 z \log ( 2 z ) \leq h_\dagger$, a contradiction.
\end{proof}

We can now bound $\E [A_{ 2 \bar{\tau}_{\text{spr}} }]$. The analysis is similar to Appendix \ref{appProofCorRegular}. We first write
\begin{align*}
\E [A_{ 2 \bar{\tau}_{\text{spr}} }] \leq A_{2 n \ceil{h_\dagger}} + \sum_{h=\ceil{h_\dagger}}^\infty \P ( \bar{\tau}_{\text{spr}} > n h ) \sum_{j= 2 n h +1}^{2n (h+1) } ( A_j - A_{j-1} ) .
\end{align*}
By the previous two claims, we know
\begin{equation*}
\P ( \bar{\tau}_{\text{spr}} > n h ) = \P ( \bar{\tau}(n) > n h ) \leq n ( 1 - \Upsilon / \bar{d}_{\text{hon}} )^h \leq n \exp ( - h \Upsilon / \bar{d}_{\text{hon}} ) \leq h^{-2\beta} .
\end{equation*}
By Claim \ref{clmPropAj} from Appendix \ref{appBasicRes}, we also have
\begin{equation*}
A_{2 n \ceil{h_\dagger} } \leq ( 2 e )^{2\beta} ( n h_\dagger )^\beta , \quad A_{2n(h+1)} \leq e^{2\beta} ( 2n(h+1) )^{\beta} \leq e^{2\beta} ( 2 n (2h) )^\beta = ( 2 e )^{2\beta} ( n h )^\beta\ \forall\ h \in \N .
\end{equation*}
Therefore, combining the previous three expressions, we obtain
\begin{equation*}
\E A_{ 2 \bar{\tau}_{\text{spr} }} \leq ( 2 e )^{2 \beta} n^{\beta} \left( h_\dagger^\beta + \sum_{h = \ceil{h_\dagger} }^\infty h^{-\beta} \right) \leq ( 2 e )^{2 \beta} n^{\beta} \left( h_\dagger^\beta + \frac{1}{\beta-1} \right) ,
\end{equation*}
where the second inequality uses $\beta > 1$, $h_\dagger \geq 2$, and Claim \ref{clmPolySeries} from Appendix \ref{appBasicRes}. Hence, we have shown $\E [A_{ 2 \bar{\tau}_{\text{spr} }} ]= O ( ( n h_\dagger )^\beta ) = \tilde{O} ( ( n \bar{d}_{\text{hon}} / \Upsilon )^\beta )$. Note this bound cannot be improved in general -- for example, if $G_{\text{hon}}$ is a line graph, it becomes $\tilde{O} ( ( n / \Upsilon )^\beta )$, so since $\E [\bar{\tau}_{\text{spr}}]^\beta = O( \E [A_{\bar{\tau}_{\text{spr}}}] )$, we have $\E [\bar{\tau}_{\text{spr}}] = \tilde{O} ( n / \Upsilon )$, which is the correct scaling (up to log terms) in Definition \ref{defnRumor}.

%% file: proofSpread.tex
\section{Details from Section \ref{secAnalysis}} \label{appProofSpread}

In this appendix, we formalize the analysis that was discussed in Section \ref{secAnalysis}. In particular, the subsequent five sub-appendices provide details on the respective five subsections of Section \ref{secAnalysis}.

\subsection{Details from Section \ref{secLearnArm}} \label{appLearnArm}

Let $\psi' = (\rho_2 + \beta(2\alpha-1) ) / ( 2 \alpha \beta )$. Note that since $\alpha > 2$ and $\rho_2 \in (0,\beta)$ by assumption,
\begin{equation*}
0 < 1 - 1 / ( 2 \alpha ) =  \beta (2 \alpha - 1) / ( 2 \alpha \beta ) < \psi' < (\beta + \beta (2\alpha-1) ) / ( 2 \alpha \beta ) = 1 ,
\end{equation*}
so $\psi = \sqrt{ \psi' }$ is well-defined and $\psi \in (0,1)$. Next, for any $j \in \N$, define
\begin{equation}
\delta_{j,1} = \sqrt{ \frac{ 4 \alpha \log A_j }{ (\frac{A_j-A_{j-1}}{S+2} - 1 ) \vee 1  } }  , \quad \delta_{j,2} = \sqrt{ \alpha \log (A_{j-1}\vee 1) }  \left( \frac{1-\psi}{\sqrt{\kappa_j} } - \frac{2}{ \sqrt{(\frac{A_j-A_{j-1}}{S+2} - 1 ) \vee 1 }} \right) \vee 0 .
\end{equation}
Since $A_j - A_{j-1} = \Theta ( j^{\beta-1} )$, $\psi < 1$, and $\rho_2 < \beta - 1$, we are guaranteed that $A_j - A_{j-1} \geq 2(S+2)$ and $\delta_{j,2} > 0$ for large $j$, so the following is well-defined:
\begin{equation} \label{eqJ1starDefn}
J_1^\star = \min \left\{ j \in \N :  A_{j'}-A_{j'-1} \geq 2(S+2) , \delta_{j',2} > 0\ \forall\ j' \geq j \right\} .
\end{equation}
Also note $J_1^\star \geq 2$ (since $A_1 - A_0 = 1$). Next, recall from Section \ref{secLearnArm} that
\begin{gather}
\Xi_{j,1}^{(i)}   = \left\{ B_j^{(i)} \notin G_{\delta_{j,1}}( S_j^{(i)}) \right\} , \quad \Xi_{j,2}^{(i)} = \left\{ \min_{ w \in G_{\delta_{j,2}}( S_j^{(i)}) } T_w^{(i)} (A_j) \leq \kappa_j \right\}  , \quad \Xi_j^{(i)} = \Xi_{j,1}^{(i)} \cup \Xi_{j,2}^{(i)} .
\end{gather} 
Hence, if we let $\mathcal{S}^{(i)} = \{ W \subset [K] : |W| = S+2 , \hat{S}^{(i)} \subset W \}$ denote the possible active sets for agent $i$ (i.e., $S_j^{(i)} \in \mathcal{S}^{(i)}$ for any phase $j$), we can write
\begin{equation}
\Xi_j^{(i)}  = \cup_{W \in \mathcal{S}^{(i)}} (  ( \Xi_{j,1}^{(i)} \cap \{ S_j^{(i)} = W \}  ) \cup ( (\Xi_{j,1}^{(i)})^C \cap \Xi_{j,2}^{(i)} \cap \{ S_j^{(i)} = W \} ) ) .
\end{equation}
Consequently, by the union bound, we obtain
\begin{equation} \label{eqCapXiTwoTerms}
\P ( \Xi_j^{(i)} ) \leq \sum_{ W \in \mathcal{S}^{(i)} } \left( \P ( \Xi_{j,1}^{(i)} \cap \{ S_j^{(i)} = W \}  ) + \P ( (\Xi_{j,1}^{(i)})^C \cap \Xi_{j,2}^{(i)} \cap \{ S_j^{(i)} = W \} ) \right)  .
\end{equation}
The next two claims bound the two summands on the right side.

\begin{clm} \label{clmMostPlayed}
Under the assumptions of Theorem \ref{thmProposed}, for any $i \in [n]$, $j \geq J_1^\star$, and $W \in \mathcal{S}^{(i)}$, we have
\begin{equation}
\P ( \Xi_{j,1}^{(i)} \cap \{ S_j^{(i)} = W \}  )\leq   4 S (j-1)^{ \beta ( 3 - 2  \alpha ) } / (  2 \alpha - 3 ) .
\end{equation}
\end{clm}
\begin{proof}
If $W \setminus G_{\delta_{j,1}}(W) = \emptyset$, the claim is immediate. Otherwise, $\Xi_{j,1}^{(i)} \cap \{ S_j^{(i)} = W \}$ implies $B_j^{(i)} = w$ for some $w \in W \setminus G_{\delta_{j,1}}(W)$. Thus, by the union bound,
\begin{equation} \label{eqMostPlayedToW}
\P ( \Xi_{j,1}^{(i)} \cap \{ S_j^{(i)} = W \}  ) \leq \sum_{w \in W \setminus G_{\delta_{j,1}}(W)}  \P ( B_j^{(i)} = w  , \min W \in S_j^{(i)}  )  .
\end{equation}
Fix $w \in W \setminus G_{\delta_{j,1}}(W)$. Then $B_j^{(i)} = w$ implies $T_w^{(i)} ( A_j ) - T_w^{(i)} ( A_{j-1} ) \geq ( A_j - A_{j-1} ) / (S+2)$ (else, by definition of $B_j^{(i)}$, $\sum_{k \in S_j^{(i)}} ( T_k^{(i)}(A_j) - T_k^{(i)}(A_{j-1}) ) < A_j - A_{j-1}$). Since $A_j - A_{j-1} \geq S+2$ (by $j \geq J_1^\star$), we conclude $T_w^{(i)} ( A_j ) - T_w^{(i)} ( A_{j-1} ) \geq 1$, so there exists $t \in \{1+A_{j-1},\ldots,A_j\}$ such that
\begin{equation}
T_w^{(i)}(t-1) - T_w^{(i)} ( A_{j-1} )  = T_w^{(i)} ( A_j ) - T_w^{(i)} ( A_{j-1} ) - 1 , \quad I_t^{(i)} = w .
\end{equation}
Combining and using the union bound and with $T_w^{(i)} ( A_{j-1} )  \geq 0$ by definition, we obtain
\begin{align} \label{eqMostPlayedToT}
\P (   B_j^{(i)} = w , \min W \in S_j^{(i)}  )   \leq \sum_{t=1+A_{j-1}}^{A_j}  \P \left(  T_w^{(i)}(t-1) \geq \frac{A_j-A_{j-1}}{S+2} - 1  , \min W \in S_j^{(i)} , I_t^{(i)} = w \right)  .
\end{align}
Now fix $t$ as in the summation. Observe that since $w \in W \setminus G_{\delta_{j,1}}(W)$ and $j \geq J_1^\star$, we have
\begin{equation*}
\mu_{\min W}  - \mu_w > \delta_{j,1} =  \sqrt{ \frac{ 4\alpha \log A_j }{ \frac{A_j-A_{j-1}}{S+2} - 1   } } \geq \sqrt{ \frac{ 4\alpha \log t }{ \frac{A_j-A_{j-1}}{S+2} - 1   } } .
\end{equation*}
Therefore, for any such $t$, we can apply a basic bandit tail (namely, Corollary \ref{corBanditTailEll} from Appendix \ref{appBanditRes} with the parameters $k_1 = w$, $k_2 = \min W$, and $\ell = ( A_j - A_{j-1} ) / (S+2) - 1$) to obtain
\begin{equation*}
\P \left(  T_w^{(i)}(t-1) \geq \frac{A_j-A_{j-1}}{S+2} - 1 , \min W \in S_j^{(i)} , I_t^{(i)} = w  \right)  \leq 2 t^{2(1-\alpha)} .
\end{equation*}
Substituting into \eqref{eqMostPlayedToT} and using Claim \ref{clmPolySeries} from Appendix \ref{appBasicRes} (which applies since $\alpha > 2$), we obtain
\begin{equation*}
\P (   B_j^{(i)} = w , \min W \in S_j^{(i)}  ) \leq 2 \sum_{t=1+A_{j-1}}^\infty t^{2(1-\alpha)} \leq \frac{ 2 A_{j-1}^{ 3 - 2 \alpha } }{  2 \alpha - 3  } \leq \frac{ 2 (j-1)^{ \beta (3 - 2 \alpha ) } }{  2 \alpha - 3  } .
\end{equation*}
Substituting into \eqref{eqMostPlayedToW} and using $| W \setminus G_{\delta_{j,1}}(W)| \leq |W| - 1 = S+1 \leq 2 S$ completes the proof.
\end{proof}

\begin{clm} \label{clmDeltaOpt}
Under the assumption of Theorem \ref{thmProposed}, for any $i \in [n]$, $j \geq J_1^\star$, and $W \in \mathcal{S}^{(i)}$,
\begin{equation}
\P ( (\Xi_{j,1}^{(i)})^C \cap \Xi_{j,2}^{(i)} \cap \{ S_j^{(i)} = W \} )   \leq  6 \cdot 2^\beta  S (j-1)^{\beta(3-2\alpha)}  / (2 \alpha - 3 ) .
\end{equation}
\end{clm}
\begin{proof}
By definition, we have
\begin{gather}
 (\Xi_{j,1}^{(i)})^C \cap \Xi_{j,2}^{(i)} \cap \{ S_j^{(i)} = W \} = \left\{  B_j^{(i)} \in G_{\delta_{j,1}}(W) , \min_{ w \in G_{\delta_{j,2}}(W) } T_w^{(i)} (A_j) \leq \kappa_j , S_j^{(i)} = W \right\} .
\end{gather} 
As in the proof of Claim \ref{clmMostPlayed}, we know that
\begin{equation}
T_{B_j^{(i)}}^{(i)} ( A_j ) \geq T_{B_j^{(i)}}^{(i)} ( A_j ) - T_{B_j^{(i)}}^{(i)} ( A_{j-1} ) \geq \frac{A_j-A_{j-1}}{S+2} > \left( \frac{A_j-A_{j-1}}{S+2} - 1 \right) \vee 1 > \kappa_j ,
\end{equation}
where the final inequality holds since $\delta_{j,2} > 0$ by assumption $j \geq J_1^\star$, which implies
\begin{equation}
\frac{( \frac{A_j-A_{j-1}}{S+2} - 1 ) \vee 1}{ \kappa_j } > \left( \frac{2}{1-\psi} \right)^2 > 1 .
\end{equation}
Thus, $(\Xi_{j,1}^{(i)})^C \cap \Xi_{j,2}^{(i)} \cap \{ S_j^{(i)} = W \}$ implies $B_j^{(i)} \notin \argmin_{ w \in G_{\delta_{j,2}}(W) } T_w^{(i)} (A_j)$, so by the union bound,
\begin{align} \label{eqDeltaOptToW}
& \P ( (\Xi_{j,1}^{(i)})^C \cap \Xi_{j,2}^{(i)} \cap \{ S_j^{(i)} = W \} ) \\ & \quad \leq \sum_{ w_1 \in G_{\delta_{j,1}}(W) , w_2 \in G_{\delta_{j,2}}(W) \setminus \{w_1\}} \P \left( T_{w_1}^{(i)}(A_j) - T_{w_1}^{(i)}(A_{j-1}) \geq \frac{A_j-A_{j-1}}{S+2} , T_{w_2}^{(i)} (A_j) \leq \kappa_j , S_j^{(i)} = W \right) .
\end{align}
Now fix $w_1, w_2$ as in the double summation. Then similar to the proof of Claim \ref{clmMostPlayed},
\begin{align} \label{eqDeltaOptToT}
& \P \left(  T_{w_1}^{(i)}(A_j) - T_{w_1}^{(i)}(A_{j-1}) \geq \frac{A_j-A_{j-1}}{S+2} , T_{w_2}^{(i)} (A_j) \leq \kappa_j , S_j^{(i)} = W \right) \\
& \quad \leq \sum_{t=1+A_{j-1}}^{A_j} \P \left(  T_{w_1}^{(i)}(t-1) \geq \frac{A_j-A_{j-1}}{S+2} - 1 , T_{w_2}^{(i)} (A_j) \leq \kappa_j , w_2 \in S_j^{(i)} , I_t^{(i)} = w_1 \right) \\
& \quad \leq  \sum_{t=1+A_{j-1}}^{A_j} 2 \kappa_j t^{1-2\alpha \psi^2} = 2  \sum_{t=1+A_{j-1}}^{A_j} ( \kappa_j t^{ - \rho_2 / \beta } ) t^{ 2 - 2 \alpha   } ,
\end{align}
where the second bound follows from applying Claim \ref{clmBanditTail} from Appendix \ref{appBanditRes} with $k_1 = w_1$, $k_2 = w_2$, $\ell = ( A_j - A_{j-1} ) / (S+2) - 1$, $u = \kappa_j$, and $\iota = \psi$; note this claim applies since by assumption $j \geq J_1^\star$,
\begin{equation*}
\mu_{w_2} - \mu_{w_1} \geq \mu_{w_2} - \mu_{ \min W } \geq - \delta_{j,2} \geq \sqrt{ \alpha \log t }  \left( \frac{2}{ \sqrt{\frac{A_j-A_{j-1}}{S+2} - 1 }}  - \frac{1-\psi}{\sqrt{\kappa_j} }\right) .
\end{equation*}
Next, observe that for any $t \geq A_{j-1} \geq (j-1)^\beta$, by definition of $\kappa_j$, $j \geq J_1^\star \geq 2$, and $\rho_2 < \beta-1$,
\begin{equation*}
\kappa_j t^{ - \rho_2 / \beta } \leq \kappa_j ( j-1 )^{ - \rho_2 } = ( 1 - 1 / j )^{ - \rho_2 } / ( K^2 S ) \leq 2^{ \beta - 1 } / S .
\end{equation*}
Similar to the proof of Claim \ref{clmMostPlayed}, we can then use Claim \ref{clmPolySeries} to obtain
\begin{equation*}
 2  \sum_{t=1+A_{j-1}}^{A_j} ( \kappa_j t^{ - \rho_2 / \beta } ) t^{ 2 - 2 \alpha   }  \leq \frac{2^\beta}{S} \sum_{t=1+A_{j-1}}^{A_j} t^{2-2\alpha} \leq \frac{ 2^\beta (j-1)^{\beta(3-2\alpha) } }{ S (2\alpha-3) } .
\end{equation*}
Combining with \eqref{eqDeltaOptToW} and \eqref{eqDeltaOptToT} completes the proof, since
\begin{align*}
& | G_{\delta_{j,1}}(W) | | G_{\delta_{j,2}}(W) \setminus \{w_1\} | < (S+2)(S+1) \leq (3S)(2S) = 6 S^2 . \qedhere
\end{align*}
\end{proof}

Finally, we provide the tail for $\tau_{\text{arm}} = \inf \{ j \in \N  : \ind (  \Xi_{j'}^{(i)} ) = 0\ \forall\ i \in [n], j' \geq j \}$.
\begin{lem} \label{lemLearnArm}
Under the assumptions of Theorem \ref{thmProposed}, for any $j  \geq J_1^\star \vee 3$,
\begin{equation}
\P ( \tau_{\text{arm}} > j ) \leq \frac{ ( 6^\beta  + 2 ) n K^2 S (j-2)^{ \beta ( 3 - 2 \alpha) + 1 } }{ (2 \alpha  - 3) ( \beta ( 2 \alpha  - 3 ) - 1 ) } . 
\end{equation}
\end{lem}
\begin{proof}
By \eqref{eqCapXiTwoTerms}, Claims \ref{clmMostPlayed} and \ref{clmDeltaOpt}, $|\mathcal{S}^{(i)}| = \binom{K}{2} < \frac{K^2}{2}$, and $\beta \geq 1$, we can write
\begin{equation}
\P ( \Xi_{j'}^{(i)} ) \leq  \frac{ ( 6 \cdot 2^\beta  + 4 ) K^2 S (j'-1)^{\beta(3-2\alpha)} }{ 2 (2 \alpha - 3) } \leq \frac{ ( 6^\beta  + 2 ) K^2 S (j'-1)^{\beta(3-2\alpha)} }{ 2 \alpha - 3 } .
\end{equation}
Thus, because $\tau_{\text{arm}} > j$ implies $\ind(\cup_{i=1}^n \Xi_{j'}^{(i)}) = 1$ for some $i \in [n]$ and $j' \geq j$, the union bound gives
\begin{equation}
\P ( \tau_{\text{arm}} > j ) \leq \sum_{j'=j}^\infty \sum_{i=1}^n \P ( \Xi_{j'}^{(i)} ) \leq \frac{ ( 6^\beta  + 2 ) n K^2 S  }{ 2 \alpha - 3 } \sum_{j'=j}^\infty (j'-1)^{\beta(3-2\alpha)} .
\end{equation}
Finally, use Claim \ref{clmPolySeries} (which applies since $\beta(2\alpha-3) > 1$) to bound the sum.
\end{proof}

\subsection{Details from Section \ref{secCommFreq}} \label{appCommFreq}

Recall $\theta_j = (j/3)^{\rho_1}$, where $\rho_1 \in (0,1/\eta]$ and $\eta > 1$. Hence, for all large $j$, we have
\begin{equation}
1 \leq \floor{\theta_j} \leq j-2 , \quad \ceil{ \floor{ \theta_j }^\eta } \leq \theta_j^\eta + 1 \leq (j/3) +1 \leq (j-2) -  j / 3  .
\end{equation}
Thus, the following is well-defined:
\begin{equation} \label{eqJ2starDefn}
J_2^\star = \min \left\{ j \in \N : 1 \leq \floor{\theta_{j'}} \leq j'-2 ,  j' / 3  \leq (j'-2) - \ceil{ \floor{ \theta_{j'} }^\eta }\ \forall\ j' \geq j \right\} .
\end{equation}
Now recall from Section \ref{secCommFreq} that $\Xi_j^{(i \rightarrow i')} = \cap_{j'=\floor{\theta_j} }^{j-2} \{ H_{j'}^{(i')} \neq i \}$, and
\begin{equation*}
\tau_{\text{com}} = \inf \{ j \in \N : \ind ( \cup_{ (i, i') \in E_{\text{hon}} } \Xi_{j'}^{(i \rightarrow i')}  ) = 0\ \forall\ j' \in \{ j , j+1 , \ldots \} \} .
\end{equation*}
The next lemma provides a tail bound for this random phase.

\begin{lem} \label{lemTauCom}
Under the assumptions of Theorem \ref{thmProposed}, for any $j \geq J_2^\star$,
\begin{equation}
\P ( \tau_{\text{com}} > j ) \leq  3 (n+m)^3 \exp ( - j / ( 3 \bar{d} ) ) .
\end{equation}
\end{lem}
\begin{proof}
We first use the union bound to write
\begin{align} \label{eqFreqCommUnion}
\P ( \tau_{\text{com}} > j ) \leq \sum_{j'=j}^\infty \sum_{ i \rightarrow i' \in E_{\text{hon}}} \P ( \Xi_{j'}^{(i \rightarrow i')}  ) .
\end{align}
Fix $i \rightarrow i' \in E_{\text{hon}}$ and $j' \geq j$. Suppose $\Xi_{j'}^{(i \rightarrow i')}$ holds. Then $i \notin P_{j''}^{(i')} \setminus P_{j''-1}^{(i')}\ \forall\ j'' \in \{ \floor{ \theta_{j'} }+1,\ldots,j'-1 \}$; else, we can find $j'' \in \{ \floor{ \theta_{j'} }+1,\ldots,j'-1\}$ such that $H_{j''-1}^{(i')} = i$ (i.e., $H_{j''}^{(i')} = i$ for $j'' \in \{ \floor{ \theta_{j'} }, \ldots , j' -2\}$), contradicting $\Xi_{j'}^{(i \rightarrow i')}$. Hence, we have two cases: $i \notin P_{j''}^{(i')} \setminus P_{j''-1}^{(i')}\ \forall\ j'' \in [j'-1]$, or $i \in P_{j''}^{(i')} \setminus P_{j''-1}^{(i')}$ for some $j'' \in [j'-1]$ and $\max \{ j'' \in [j'-1] : i \in P_{j''}^{(i')} \setminus P_{j''-1}^{(i')} \} \leq \floor{ \theta_{j'} }$. In the former case, $i \notin P_{j''}^{(i')}\ \forall\ j'' \in [j'-1]$; in the latter, $i \notin P_{j''}^{(i')}\ \forall\ j'' \in \{ \ceil{ \floor{ \theta_{j'} }^\eta } +1 , \ldots , j'-1 \}$. Thus,
\begin{align} 
\P ( \Xi_{j'}^{(i \rightarrow i')} ) & \leq \P ( \cap_{j''=\ceil{\floor{ \theta_{j'} }^\eta} + 1}^{j'-2} \{ i \notin P_{j''}^{(i')} , H_{j''}^{(i')} \neq i \} ) \\
& = \E [ \ind ( \cap_{j''=\ceil{\floor{ \theta_{j'} }^\eta} + 1}^{j'-3} \{ i \notin P_{j''}^{(i')} , H_{j''}^{(i')} \neq i \} ) \ind (i \notin P_{j'-2}^{(i')}  ) \P_{j'-2} (  H_{j'-2}^{(i)} \neq i ) ] .
\end{align}
Now given that $i \notin P_{j'-2}^{(i')}$, $H_{j'-2}^{(i)}$ is sampled uniformly from a set of at most $\bar{d}$ elements which includes $i$, so $\P_{j'-2} (  H_{j'-2}^{(i)} \neq i ) \leq ( 1 - 1 / \bar{d} )$. Substituting above and iterating yields
\begin{equation} \label{eqFreqCommDmax}
 \P ( \Xi_{j'}^{(i \rightarrow i')} ) \leq ( 1 - 1 / \bar{d} )^{ j' - 2 - \ceil{ \floor{ \theta_j' }^\eta } } \leq ( 1 - 1 / \bar{d} )^{ j' / 3 } ,
\end{equation}
where the final inequality uses $j' \geq j \geq J_2^\star$. Combining \eqref{eqFreqCommUnion} and \eqref{eqFreqCommDmax} and computing a geometric series, we obtain
\begin{equation}
\P ( \tau_{\text{com}} > j )  \leq \frac{ |E_{\text{hon}}| ( 1 - 1 / \bar{d} )^{j/3} }{ 1 - ( 1 - 1 / \bar{d} )^{ j/3  } } \leq \frac{ |E_{\text{hon}}| ( 1 - 1 / \bar{d} )^{j/3} }{ 1 - ( 1 - 1 / \bar{d} )^{ 1/3  } } 
\end{equation}
Finally, using $|E_{\text{hon}}| \leq n^2 < (n+m)^2$, $1-x \leq e^{-x}\ \forall\ x \in \R$, $(1+x)^r \leq 1 + r x$ for any $r \in (0,1)$ and $x \geq -1$, and $\bar{d} \leq m+n$, we obtain the desired bound.
\end{proof}

\subsection{Details from Section \ref{secNoBlock}} \label{appNoBlock}

We begin with some intermediate claims.

\begin{clm} \label{clmMuDecayOne}
If the assumptions of Theorem \ref{thmProposed} hold, then for any $i \in [n]$ and $j \geq \tau_{\text{arm}}$, we have
\begin{equation*}
\mu_{\min S_j^{(i)}}  \leq \mu_{ B_j^{(i)} } + \delta_{j,1} \leq \mu_{\min S_{j+1}^{(i)}} + \delta_{j,1} .
\end{equation*}
\end{clm}
\begin{proof}
The first inequality holds by definition of $\tau_{\text{arm}}$ and assumption $j \geq \tau_{\text{arm}}$. The second holds since $\min S_{j+1}^{(i)}$ is the best arm in $S_{j+1}^{(i)}$ and $B_j^{(i)} \in S_{j+1}^{(i)}$ in the algorithm. 
\end{proof}

\begin{clm} \label{clmMuDecayMulti}
If the assumptions of Theorem \ref{thmProposed} hold, then for any $i \in [n]$ and $j' \geq j \geq \tau_{\text{arm}}$,
\begin{equation}
\mu_{\min S_{j'}^{(i)}} \geq \mu_{ \min S_{j}^{(i)} } - (K-1) \sup_{j'' \in \{ j, \ldots, j'\} } \delta_{j'',1} .
\end{equation}
\end{clm}
\begin{proof}
If $j = j'$ or $\mu_{\min S_{j'}^{(i)}} \geq \mu_{\min S_{j}^{(i)}}$, the bound is immediate, so we assume $j' > j$ and $\mu_{\min S_{j'}^{(i)}} < \mu_{\min S_{j}^{(i)}}$ for the remainder of the proof. Under this assumption, there must exist a phase $j'' \in \{j+1,\ldots,j'\}$ at which the mean of the best active arm reaches a new strict minimum since $j$, i.e., $\mu_{ \min S_{j''}^{(i)} } < \mu_{ \min S_{j}^{(i)}  }$. Let $m$ denote the number of phases this occurs and $j^{(1)} , \ldots , j^{(m)}$ these (ordered) phases; formally,
\begin{gather}
j^{(0)} = j , \quad j^{(l)} = \min \left\{ j'' \in \{ j^{(l-1)}+1, \ldots , j' \} : \mu_{\min S_{j''}^{(i)}} < \mu_{\min S_{j^{(l-1)}}^{(i)}} \right\}\ \forall\ l \in [m] .
\end{gather}
The remainder of the proof relies on the following three inequalities:
\begin{equation} \label{eqMuDecayMulti}
m \leq K-1 , \quad \mu_{\min S_{j^{(m)}}^{(i)} } \leq \mu_{\min S_{j'}^{(i)} } , \quad \mu_{\min S_{j^{(l-1)}}^{(i)} } \leq \mu_{\min S_{j^{(l)}-1}^{(i)} }\ \forall\ l \in [m] . 
\end{equation}
The first inequality holds since $\mu_{\min S_{j^{(0)}}^{(i)} } > \cdots > \mu_{\min S_{j^{(m)}}^{(i)} }$ by definition, so $\min S_{j^{(0)}}^{(i)}, \ldots , \min S_{j^{(m)}}^{(i)}$ are distinct arms; since there are $m+1$ of these arms and $K$ in total, $m+1 \leq K$. For the second, we have $\mu_{\min S_{j^{(m)}}^{(i)} } = \mu_{\min S_{j'}^{(i)} }$ when $j^{(m)} = j'$ and $\mu_{\min S_{j^{(m)}}^{(i)} } \leq \mu_{\min S_{j'}^{(i)} }$ when $j^{(m)} < j'$ (if the latter fails, we contradict the definition of $m$). For the third, note $j^{(l)} \geq j^{(l-1)} + 1$ by construction, so if $j^{(l)} = j^{(l-1)} + 1$, the bound holds with equality; else, $j^{(l)} - 1 \geq j^{(l-1)} + 1$, so if the bound fails,
\begin{equation}
j^{(l)} - 1 \in \left\{ j'' \in \{ j^{(l-1)}+1, \ldots , j' \} : \mu_{\min S_{j''}^{(i)}} < \mu_{\min S_{j^{(l-1)}}^{(i)}} \right\} ,
\end{equation}
which is a contradiction, since $j^{(l)}$ is the minimal element of the set at right. Hence, \eqref{eqMuDecayMulti} holds. Combined with Claim \ref{clmMuDecayOne} (note $j^{(l)} - 1 \geq j^{(0)} = j \geq \tau_{\text{arm}}\ \forall\ l \in [m]$, as required), we obtain
\begin{align}
\mu_{\min S_{j}^{(i)} }  - \mu_{\min S_{j'}^{(i)}} & = \sum_{l=1}^m \left( \mu_{\min S_{j^{(l-1)}}^{(i)} } - \mu_{\min S_{j^{(l)}}^{(i)} } \right) +  \mu_{\min S_{j^{(m)}}^{(i)} } - \mu_{\min S_{j'}^{(i)} } \\
& \leq \sum_{l=1}^m \left( \mu_{\min S_{j^{(l)}-1}^{(i)} } - \mu_{\min S_{j^{(l)}}^{(i)} } \right) \leq \sum_{l=1}^m \delta_{j^{(l)}-1,1} \leq  (K-1) \sup_{j'' \in \{ j, \ldots, j'\} } \delta_{j'',1} ,
\end{align}
where the last inequality uses $j = j^{(0)} \leq j^{(l)} - 1 < j'\ \forall\ l \in [m]$. 
\end{proof}

As a simple corollary of the previous two claims, we have the following.
\begin{cor} \label{corMuDecay}
If the assumptions of Theorem \ref{thmProposed} hold, then for any $i \in [n]$ and $j' \geq j \geq \tau_{\text{arm}}$,
\begin{equation}
\mu_{B_{j'}^{(i)}} \geq \mu_{ \min S_{j}^{(i)} } - K \sup_{j'' \in \{ j, \ldots, j'\} } \delta_{j'',1} .
\end{equation}
\end{cor}
\begin{proof}
Since $j' \geq j \geq \tau_{\text{arm}}$, we can use Claims \ref{clmMuDecayOne} and \ref{clmMuDecayMulti}, respectively, to obtain
\begin{align}
\mu_{B_{j'}^{(i)}} & \geq \mu_{\min S_{j'}^{(i)} } - \delta_{j',1} \geq \mu_{\min S_{j'}^{(i)} } - \sup_{j'' \in \{ j, \ldots, j'\} } \delta_{j'',1}  \geq \mu_{\min S_{j}^{(i)} } - K \sup_{j'' \in \{ j, \ldots, j'\} } \delta_{j'',1} . \qedhere
\end{align}
\end{proof}

Next, inspecting the analysis in Section \ref{secNoBlock}, we see that $\delta_{j,2} \geq (K+1) \sup_{j' \in \{ \floor{ \theta_j } , \ldots , j \} } \delta_{j',1}$ for large $j$. Thus, the following is well-defined: 
\begin{equation} \label{eqJ3starDefn}
J_3^\star = \min \left\{ j \in \N  : \delta_{j',2} \geq (K+1) \sup_{j'' \in \{ \floor{ \theta_{j' } } , \ldots , j' \} } \delta_{j'',1}\ \forall\ j' \geq j \right\} . 
\end{equation}

As discussed in Section \ref{secNoBlock}, we can now show that no \textit{new} accidental blocking occurs at late phases, at least among pairs of honest agents that have recently communicated.
\begin{clm} \label{clmNoEnterIfComm}
Under the assumptions of Theorem \ref{thmProposed}, if $j \geq J_3^\star$, $\floor{ \theta_j } \geq \tau_{\text{arm}}$, and $H_{j'}^{(i')} = i$ for some $i,i' \in [n]$ and $j' \in \{ \floor{ \theta_j }  , \ldots , j -2 \}$, then $i' \notin P_{j}^{(i)} \setminus P_{j-1}^{(i)}$. 
\end{clm}
\begin{proof}
Suppose instead that $i' \in P_{j}^{(i)} \setminus P_{j-1}^{(i)}$. Then by the algorithm,
\begin{equation} \label{eqNoEnterIfComm}
B_{j}^{(i)} = \cdots = B_{\floor{\theta_j}}^{(i)}  , \quad T_{ R_{j-1}^{(i)} }^{(i)} ( A_{j} ) \leq \kappa_j , \quad R_{j-1}^{(i)} = B_{j-1}^{(i')} \in S_{j}^{(i)} . 
\end{equation}
Since $j \geq \floor{ \theta_j } \geq \tau_{\text{arm}}$, this implies $B_{j-1}^{(i')} \notin G_{\delta_{j,2}}(S_{j}^{(i)})$. We then observe the following:
\begin{itemize}
\item Since $B_{j-1}^{(i')} \in S_j^{(i)} \setminus G_{\delta_{j,2}}(S_{j}^{(i)})$, the definition of $G_{\delta_{j,2}}(S_j^{(i)})$ implies $\mu_{ B_{j-1}^{(i')}  } < \mu_{ \min S_{j}^{(i)}  } - \delta_{j,2}$.
\item Again using $j \geq \tau_{\text{arm}}$, Claim \ref{clmMuDecayOne} implies $\mu_{ \min S_{j}^{(i)} } \leq \mu_{ B_{j}^{(i)} } + \delta_{j,1}$.
\item Since $\floor{\theta_j} \leq j' \leq j$, \eqref{eqNoEnterIfComm} implies $B_j^{(i)} = B_{j'}^{(i)}$, so $\mu_{ B_{j}^{(i)} }  = \mu_{ B_{j'}^{(i)}  }$.
\item Since $H_{j'}^{(i')} = i$, the algorithm implies $B_{j'}^{(i)} \in S_{j'+1}^{(i')}$, so $\mu_{ B_{j'}^{(i)}  } \leq \mu_{ \min S_{j'+1}^{(i')}  }$.
\item Since $\tau_{\text{arm}} \leq \floor{ \theta_j }  < j'+1 < j$, Corollary \ref{corMuDecay} implies $\mu_{ \min S_{j'+1}^{(i')}  } \leq \mu_{ B_{j-1}^{(i')} } + K \sup_{j'' \in \{ \floor{\theta_j},\ldots,j\} } \delta_{j'',1}$.
\end{itemize}
Stringing together these inequalities, we obtain
\begin{equation}
\mu_{ B_{j-1}^{(i')}  } <  \mu_{ B_{j-1}^{(i')} } + K \sup_{j'' \in \{ \floor{\theta_j},\ldots,j\} } \delta_{j'',1}  - \delta_{j,2} + \delta_{j,1} \leq \mu_{ B_{j-1}^{(i')} } + (K+1) \sup_{j'' \in \{ \floor{\theta_j},\ldots,j\} } \delta_{j'',1}  - \delta_{j,2} ,
\end{equation}
i.e., $\delta_{j,2} < (K+1) \sup_{j'' \in \{ \floor{\theta_j},\ldots,j\} } \delta_{j'',1}$, which contradicts $j \geq J_3^\star$.
\end{proof}

Finally, we prove that honest agents eventually stop blocking each other.
\begin{lem} \label{lemNoBlock}
Under the assumptions of Theorem \ref{thmProposed}, if $j \geq J_2^\star$, $\floor{\theta_j} \geq \tau_{\text{com}} \vee J_3^\star$, and $\floor{\theta_{\floor{\theta_j}}} \geq \tau_{\text{arm}}$, then for any $i \in [n]$ and $j' \geq j$, $P_{j'}^{(i)} \cap [n] = \emptyset$.
\end{lem}
\begin{proof}
Fix $i \rightarrow i' \in E_{\text{hon}}$ and $j' \geq j$; we aim to show $i' \notin P_{j'}^{(i')}$. This clearly holds if $i' \notin P_{j''}^{(i)} \setminus P_{j''-1}^{(i)}\ \forall\ j'' \leq j'$. Otherwise, $j_{B,1} = \max \{ j'' \leq j' : i' \in P_{j''}^{(i)} \setminus P_{j''-1}^{(i)} \}$ (the latest phase up to and including $j'$ that $i$ blocked $i'$) is well-defined. We consider two cases of $j_{B,1}$.

The first case is $j_{B,1} \leq \floor{ \theta_j }$. Let $j_{B,2} = \min \{ j'' > j' : P_{j''}^{(i)} \setminus P_{j''-1}^{(i)} \}$ denote the first phase after $j'$ that $i$ blocked $i'$. Combined with the definition of $j_{B,1}$, the algorithm implies
\begin{equation} \label{eqBetweenJB1JB2}
i' \notin P_{j''}^{(i)}\ \forall\ j'' \in \{ \ceil{ j_{B,1}^\eta } + 1 , \ldots , j_{B,2} \} .
\end{equation}
Since $j \geq J_2^\star$, the definition of $J_2^\star$ implies
\begin{equation}
\ceil{ \floor{ \theta_j }^\eta } + 1 \leq (j-2) - (j/3) + 1 < (j-2)+1 = j-1 ,
\end{equation}
so $\ceil{ j_B^\eta } + 1 < j-1 < j'$ as well. Combined with $j' \leq j_{B,2}-1$ by definition, $i' \notin P_{j'}^{(i')}$ holds by \eqref{eqBetweenJB1JB2}.

The second case is $j_{B,1} > \floor{ \theta_j }$. By assumption, $j_{B,1} > \floor{ \theta_j } \geq \tau_{\text{com}}$. Hence, by definition of $\tau_{\text{com}}$, $H_{j_C}^{(i')} = i$ for some $j_C \in \{ \floor{ \theta_{j_{B,1}} } , \ldots , j_{B,1} - 2 \}$. Note that $j_{B,1} \geq \floor{ \theta_j } \geq J_3^\star$ and $\floor{ \theta_{j_{B,1}} } \geq \floor{ \theta_{\floor{ \theta_j }} } \geq \tau_{\text{arm}}$ by assumption (and by monotonicity of $\{ \floor{\theta_{j''} } \}_{j'' \in \N}$ in the latter case). Hence, we can apply Claim \ref{clmNoEnterIfComm} (with $j = j_{B,1}$ and $j' = j_C$ in the claim) to obtain $i' \notin P_{j_{B,1}}^{(i)} \setminus P_{j_{B,1}-1}^{(i)}$. This is a contradiction.
\end{proof}

\subsection{Details from Section \ref{secCoupling}} \label{appCoupling}

We first verify that the sampling strategy in Section \ref{secCoupling} is identical to the one in Algorithm \ref{algGetArm}.
\begin{clm} \label{clmCoupling}
Suppose we replace the sampling of $H_j^{(i)}$ in Algorithm \ref{algGetArm} with the sampling of Section \ref{secCoupling}, and recall $\P_j$ denotes probability conditioned on all randomness before this sampling occurs. Then
\begin{equation}
\P_j ( H_j^{(i)} = i' ) =  \ind ( i' \in N(i) \setminus P_j^{(i)} ) / | N(i) \setminus P_j^{(i)} |\ \forall\ i \in [n] , i' \in [n+m] , j \in \N .
\end{equation}
\end{clm}
\begin{proof}
Since $\P_j$ conditions on $P_j^{(i)}$, we can prove the identity separately in the cases $P_j^{(i)} \cap [n] \neq \emptyset$ and $P_j^{(i)} \cap [n] = \emptyset$. The identity is immediate in the former case. For the latter, we have
\begin{align} \label{eqNewProc}
\P_j ( H_j^{(i)} = i' ) & = \P_j ( H_j^{(i)} = i' | Y_j^{(i)} = 1 ) \P_j ( Y_j^{(i)} = 1 ) + \P_j ( H_j^{(i)} = i' | Y_j^{(i)} = 0 ) \P_j ( Y_j^{(i)} = 0 ) \\
& = \frac{ \ind ( i' \in N_{\text{hon}}(i) ) }{ d_{\text{hon}}(i) } \frac{ d_{\text{hon}}(i) }{ | N(i) \setminus P_j^{(i)}| }  + \frac{ \ind ( i' \in N_{\text{mal}}(i) \setminus P_j^{(i)} )  }{  | N_{\text{mal}}(i) \setminus P_j^{(i)} |  } \left( 1 - \frac{ d_{\text{hon}}(i) }{ | N(i) \setminus P_j^{(i)}| }\right) \\
& = \frac{ \ind ( i' \in N_{\text{hon}}(i) ) }{ |N(i) \setminus P_j^{(i)} | } + \frac{ \ind ( i' \in N_{\text{mal}}(i) \setminus P_j^{(i)} )  }{  | N_{\text{mal}}(i) \setminus P_j^{(i)} |  }  \frac{|N(i) \setminus P_j^{(i)} | - d_{\text{hon}}(i) }{ |N(i) \setminus P_j^{(i)} |} \\
& = \frac{ \ind ( i' \in N_{\text{hon}}(i) ) + \ind ( i' \in N_{\text{mal}}(i) \setminus P_j^{(i)} )  }{ |N(i) \setminus P_j^{(i)} | } = \frac{ \ind ( i' \in N(i) \setminus P_j^{(i)} ) }{ | N(i) \setminus P_j^{(i)} | } . \qedhere
\end{align}
\end{proof}

Next, note that since $\delta_{j',1} \rightarrow 0$ as $j' \rightarrow \infty$, the following is well-defined:
\begin{equation} \label{eqJ4starDefn}
J_4^\star = \min \{ j \in \{2,3,\ldots\} : \delta_{j',1} < \Delta_2\ \forall\ j' \geq \floor{j/2} \} .
\end{equation}
As in Section \ref{secCoupling}, we let $\mathcal{I}_j = \{ i \in [n] : 1 \in S_j^{(i)} \}$ be the agents with the best arm active at phase $j$.

\begin{clm} \label{clmInformedRecBest}
Under the assumptions of Theorem \ref{thmProposed}, if $j \geq J_4^\star$ and $\floor{j/2} \geq \tau_{\text{arm}}$, then $B_{j'}^{(i)}= 1\ \forall\ j' \geq \floor{ j/2 } , i \in \mathcal{I}_{j'}$.
\end{clm}
\begin{proof}
Suppose instead that $B_{j'}^{(i)} \neq 1$ for some $j' \geq \floor{j/2}$ and $i \in \mathcal{I}_{j'}$. Since $j' \geq \floor{j/2} \geq \floor{J_4^\star/2}$, we know $\delta_{j',1} < \Delta_2$. Hence, because $1 \in S_{j'}^{(i)}$ by definition of $\mathcal{I}_{j'}$, we have $G_{ \delta_{j',1} } ( S_{j'}^{(i)} ) = \{ 1 \}$. Combined with $B_{j'}^{(i)} \neq 1$, we get $B_{j'}^{(i)} \in S_{j'}^{(i)} \setminus G_{ \delta_{j',1} } (S_{j'}^{(i)} )$, which contradicts $j' \geq \floor{j/2} \geq \tau_{\text{arm}}$.
\end{proof}

Finally, recall $\tau_{\text{spr}} = \inf \{ j \in \N : B_{j'}^{(i)} = 1\ \forall\ i \in [n] , j' \geq j \}$ and $\bar{\tau}_{\text{spr}} = \inf \{ j \in \N : \bar{\mathcal{I}}_j = [n] \}$, where $\{ \bar{\mathcal{I}}_j \}_{j=1}^\infty$ is the noisy rumor process from Definition \ref{defnRumor}.

\begin{lem} \label{lemTXfail}
Under the assumptions of Theorem \ref{thmProposed}, if $j \geq J_4^\star$, $\floor{j/2} \geq J_2^\star$, and $\floor{ \theta_{\floor{j/2}} } \geq J_3^\star$, then
\begin{equation}
\P ( \tau_{\text{com}} \leq \floor{ \theta_{\floor{j/2}} }  , \tau_{\text{arm}} \leq \floor{ \theta_{ \floor{ \theta_{\floor{j/2}} } } } , \tau_{\text{spr}} > j   ) \leq \P (\bar{\tau}_{\text{spr}} >j/2 ) . 
\end{equation}
\end{lem}
\begin{proof}
Let $\mathcal{E}_j = \{ \tau_{\text{com}} \leq \floor{ \theta_{\floor{j/2}} }  , \tau_{\text{arm}} \leq \floor{ \theta_{ \floor{ \theta_{\floor{j/2}} } } } \}$ and
\begin{equation} \label{eqInformTilde}
\tilde{\mathcal{I}}_{\floor{j/2}} = \{ i^\star \} ,\ \tilde{\mathcal{I}}_j  = \bar{\mathcal{I}}_{j -1} \cup \{ i \in [n] \setminus \tilde{\mathcal{I}}_{j -1} : \bar{Y}_j^{(i)} = 1 , \bar{H}_j^{(i)} \in \tilde{\mathcal{I}}_{j -1}  \}\ \forall\ j \in \{ \floor{j/2} + 1 , \floor{j/2} + 2 , \ldots \} .
\end{equation}
Then it suffices to prove the following:
\begin{equation} \label{eqTXfail}
\P_{\floor{j/2}} ( \mathcal{E}_j \cap \{ \tau_{\text{spr}} > j \}  ) \leq \P_{\floor{j/2}} ( \mathcal{E}_j \cap \{ \mathcal{I}_j \neq [n] \} ) \leq \P_{\floor{j/2}} ( \tilde{\mathcal{I}}_j \neq [n] ) = \P (\bar{\tau}_{\text{spr}} > j/2 ) .
\end{equation}

For the first inequality in \eqref{eqTXfail}, we begin by proving
\begin{equation}\label{eqTXfailEmptyset}
\mathcal{E}_j \cap \{ \tau_{\text{spr}} > j \} \cap \{ \mathcal{I}_j = [n] \} = \emptyset .
\end{equation}
To do so, we show $\mathcal{E}_j , \mathcal{I}_j = [n] \Rightarrow \tau_{\text{spr}} \leq j$. Assume $\mathcal{E}_j$ and $\mathcal{I}_j = [n]$ hold; by definition of $\tau_{\text{spr}}$, we aim to show $B_{j'}^{(i)} = 1\ \forall\ i \in [n] , j' \geq j$. We use induction. The base of induction ($j' = j$) holds by $\mathcal{I}_j = [n]$, $\mathcal{E}_j \subset \{ \tau_{\text{arm}} \leq \floor{ \theta_{ \floor{ \theta_{\floor{j/2}}}}} \leq \floor{j/2} \}$, and Claim \ref{clmInformedRecBest}. Given the inductive hypothesis $B_{j'}^{(i)} = 1\ \forall\ i \in [n]$, we have $1 \in S_{j'+1}^{(i)}$ by the algorithm and $\mathcal{I}_{j'+1} = [n]$ by definition, so we again use the assumption that $\mathcal{E}_j$ holds and Claim \ref{clmInformedRecBest} to obtain $B_{j'+1}^{(i)} = 1\ \forall\ i \in [n]$. Hence, \eqref{eqTXfailEmptyset} holds, so
\begin{equation}
\P_{\floor{j/2}} ( \mathcal{E}_j \cap \{ \tau_{\text{spr}} > j \}  ) = \P_{\floor{j/2}} ( \mathcal{E}_j \cap \{ \tau_{\text{spr}} > j \} \cap \{ \mathcal{I}_j \neq [n] \} ) \leq \P_{\floor{j/2}} ( \mathcal{E}_j \cap \{ \mathcal{I}_j \neq [n] \} )  .
\end{equation}

For the second inequality in \eqref{eqTXfail}, we claim, and will return to prove, the following:
\begin{equation} \label{eqTXfailEj}
\mathcal{E}_j \subset \cap_{j'=\floor{j/2}}^j \{ \tilde{\mathcal{I}}_{j'} \subset \mathcal{I}_{j'} \} .
\end{equation}
Assuming \eqref{eqTXfailEj} holds, we obtain
\begin{align}
\P_{\floor{j/2}} ( \mathcal{E}_j \cap \{ \mathcal{I}_j \neq [n] \} ) \leq \P_{\floor{j/2}} (  \tilde{\mathcal{I}}_j \subset \mathcal{I}_j , \mathcal{I}_j \neq [n] ,  ) \leq  \P_{\floor{j/2}} ( \tilde{\mathcal{I}}_j \neq [n] ) .
\end{align}
Hence, it only remains to prove \eqref{eqTXfailEj}. We show by induction on $j'$ when $\mathcal{E}_j$ holds, $\tilde{\mathcal{I}}_{j'} \subset \mathcal{I}_{j'}$ for each $j' \in \{ \floor{j/2} , \ldots , j \}$. For $j' = \floor{j/2}$, recall $1 \in S_{\floor{j/2}}^{(i^\star)}$ by Assumption \ref{assSticky}, so $\mathcal{I}_{\floor{j/2}} \supset \{ i^\star \} = \tilde{\mathcal{I}}_{\floor{j/2}}$. Now assume $\tilde{\mathcal{I}}_{j'-1} \subset \mathcal{I}_{j'-1}$ for some $j' \in \{\floor{j/2}+1,\ldots,j\}$. Let $i \in \tilde{\mathcal{I}}_{j'}$; we aim to show that $i \in \mathcal{I}_{j'}$, i.e., that $1 \in S_{j'}^{(i)}$. By \eqref{eqInformTilde}, we have two cases to consider:
\begin{itemize}
\item $i \in \tilde{\mathcal{I}}_{j'-1}$: By the inductive hypothesis, $i \in {\mathcal{I}}_{j'-1}$ as well, so $1 \in S_{j'-1}^{(i)}$ by definition. Combined with $j'-1 \geq \floor{j/2}$ and Claim \ref{clmInformedRecBest} (recall $j \geq J_4^\star$ and $\floor{j/2} \geq \floor{ \theta_{\floor{\theta_{\floor{j/2}}}}} \geq \tau_{\text{arm}}$ when $\mathcal{E}_j$ holds, so the claim applies), this implies $B_{j'-1}^{(i)} = 1$, so $1 \in S_{j'}^{(i)}$ by the algorithm.
\item $i \in [n] \setminus \tilde{\mathcal{I}}_{j' -1} , \bar{Y}_{j'}^{(i)} = 1 , \bar{H}_{j'}^{(i)} \in \tilde{\mathcal{I}}_{j'-1}$: First observe that since $\floor{ \theta_{\floor{j/2}} } \geq J_3^\star$ and $j' \geq \floor{j/2} \geq J_2^\star$, we can apply Lemma \ref{lemNoBlock} on the event $\mathcal{E}_j$ (with $j$ replaced by $\floor{j/2}$ in that lemma) to obtain $P_{j'}^{(i)} \cap [n] = \emptyset$. On the other hand, recall that $\bar{Y}_{j'}^{(i)}$ is $\text{Bernoulli}(\Upsilon)$ in Definition \ref{defnRumor} and $\upsilon_{j'}^{(i)}$ is $\text{Uniform}[0,1]$ for the Section \ref{secCoupling} sampling, so we can realize the former as $\bar{Y}_{j'}^{(i)} =  \ind ( \upsilon_{j'}^{(i)} \leq \Upsilon )$. Hence, by assumption $\bar{Y}_{j'}^{(i)} = 1$ and definition $\Upsilon = \min_{i \in [n]} d_{\text{hon}}(i) / d(i)$, we obtain
\begin{equation}
1 = \bar{Y}_{j'}^{(i)} =  \ind ( \upsilon_{j'}^{(i)} \leq \Upsilon ) \leq \ind \left( \upsilon_{j'}^{(i)} \leq \frac{ d_{\text{hon}}(i) }{ d(i) } \right) \leq \ind \left( \upsilon_{j'}^{(i)} \leq  \frac{ d_{\text{hon}}(i) }{ |N(i) \setminus P_j^{(i)}| } \right) = Y_j^{(i)} .
\end{equation}
In summary, we have shown that $P_{j'}^{(i)} \cap [n] = \emptyset$ and $Y_j^{(i)} = 1$. Hence, by the Section \ref{secCoupling} sampling, we conclude $H_j^{(i)} = \bar{H}_j^{(i)}$. Let $i' = H_j^{(i)}$ denote this honest agent. Then by the inductive hypothesis, we know that $i' \in \tilde{\mathcal{I}}_{j'-1} \subset {\mathcal{I}}_{j'-1}$, i.e., that $1 \in S_{j'-1}^{(i')}$. By Claim \ref{clmInformedRecBest}, this implies $B_{j'-1}^{(i')} = 1$, so by the algorithm, $1 = B_{j'-1}^{(i')} = R_{j'-1}^{(i)}  \in S_{j'}^{(i)}$.
\end{itemize}

Finally, for the equality in \eqref{eqTXfail}, note that $\tilde{\mathcal{I}}_{j'}$ is independent of the randomness before $\floor{j/2}$. By $\tilde{\mathcal{I}}_{\floor{j/2}} = \bar{\mathcal{I}}_0 = \{ i^\star \}$, the fact that $\bar{Y}_{j'}^{(i)}$ and $\bar{Y}_{j'-\floor{j/2}}^{(i)}$ are both $\text{Bernoulli}(\Upsilon)$ random variables and $\bar{H}_{j'}^{(i)}$ and $\bar{H}_{j'-\floor{j/2}}^{(i)}$ are both sampled uniformly from $N_{\text{hon}}(i)$, $\tilde{\mathcal{I}}_{j'}$ has the same distribution as $\bar{\mathcal{I}}_{j'-\floor{j/2}}$. Also, note that $\tilde{\mathcal{I}}_{j'}$ is independent of the randomness before $\floor{j/2}$. Finally, by definition of $\bar{\tau}_{\text{spr}}$, $\bar{\mathcal{I}}_{j-\floor{j/2}} \neq [n]$ implies that $\bar{\tau}_{\text{spr}} > j - \floor{j/2} \geq j/2$. These observations successively imply
\begin{align}
& \P_{\floor{j/2}} ( \tilde{\mathcal{I}}_j \neq [n] ) = \P ( \tilde{\mathcal{I}}_j \neq [n] ) = \P ( \bar{\mathcal{I}}_{j'-\floor{j/2}} \neq [n] ) \leq \P ( \bar{\tau}_{\text{spr}} > j /2 ) . \qedhere
\end{align}
\end{proof}

\subsection{Details from Section \ref{secSpread}} \label{appSpread}

Combining the lemmas of the above sub-appendices, we can bound the tail of the spreading time.

\begin{thm}\label{thmTail}
Under the assumptions of Theorem \ref{thmProposed}, for any $j \geq J_\star$, where $J_\star$ is defined in Claim \ref{clmJstar} from Appendix \ref{appEarlyCalc}, we have
\begin{equation}
\P ( \tau_{\text{spr}} > j  ) \leq \left( \frac{ 84^{ \rho_1^2 ( \beta ( 2 \alpha - 3 ) - 1 ) } ( 6^\beta  + 2 ) n K^2 S  }{  (2 \alpha  - 3) ( \beta ( 2 \alpha  - 3 ) - 1 ) } + 3 \right) j^{\rho_1^2 ( \beta (3-2\alpha) + 1 )} + \P \left( \bar{\tau}_{\text{spr}} > \frac{j}{2} \right) .
\end{equation}
\end{thm}
\begin{proof}
For any $j \geq J_\star$, we can use Claim \ref{clmJstar} to obtain
\begin{gather}
j \geq J_4^\star , \quad \floor{\theta_{\floor{j/2}}} \geq J_2^\star \vee  J_3^\star , \quad \floor{ \theta_{ \floor { \theta_{ \floor{j/2} } } } } \geq J_1^\star \vee ( 2 +  j^{\rho_1^2 } / 84 ) \geq 3 \label{eqUglyJapp1} \\
(n+m)^3 \exp ( - \floor{ \theta_{ \floor{j/2}}}  / ( 3 \bar{d} ) ) \leq j^{\rho_1^2 ( \beta ( 3 - 2 \alpha ) + 1 ) } . \label{eqUglyJapp2}
\end{gather}
In particular, $\floor{ \theta_{ \floor{ \theta_{\floor{j/2}} } } } \geq J_1^\star \vee 3$ implies that we can use Lemma \ref{lemLearnArm} with $j$ replaced by $\floor{ \theta_{ \floor{ \theta_{\floor{j/2}} } } }$. Combined with \eqref{eqUglyJapp1} (namely, $\floor{ \theta_{ \floor{ \theta_{ \floor{j/2} } } } } -2 \geq j^{\rho_1^2} / 84$); this yields
\begin{equation}
\P ( \tau_{\text{arm}} > \floor{ \theta_{ \floor{ \theta_{\floor{j/2}} } } } ) \leq \frac{ 84^{ \rho_1^2 ( \beta ( 2 \alpha - 3 ) - 1 ) } ( 6^\beta  + 2 ) n K^2 S  }{  (2 \alpha  - 3) ( \beta ( 2 \alpha  - 3 ) - 1 ) }   j^{\rho_1^2 ( \beta (3-2\alpha) + 1 )} .
\end{equation}
Since $\floor{ \theta_{\floor{j/2}} } \geq J_2^\star$ by \eqref{eqUglyJapp1}, we can use Lemma \ref{lemTauCom} (with $j$ replaced by $\floor{ \theta_{\floor{j/2}} }$) and \eqref{eqUglyJapp2} to obtain 
\begin{equation}
\P ( \tau_{\text{com}} > \floor{ \theta_{\floor{j/2}} } ) \leq 3(n+m)^2 \exp \left( - \frac{\floor{ \theta_{\floor{j/2}} }}{3 \bar{d}} \right) \leq 3 j^{\rho_1^2 ( \beta ( 3 - 2 \alpha ) + 1 ) } .
\end{equation}
Furthermore, using the bounds $j \geq J_4^\star$, $\floor{ \theta_{\floor{j/2}} } \geq J_3^\star$, and $\floor{ \theta_{\floor{j/2}} } \geq J_2^\star$ from \eqref{eqUglyJapp1} (the last of which implies $\floor{j/2} \geq J_2^\star$, since $\theta_{\floor{j/2}} = ( \floor{j/2} / 3 )^{\rho_1} \leq \floor{j/2}$ by $\rho_1 \in (0,1)$), we can use Lemma \ref{lemTXfail} to get
\begin{equation}
\P ( \tau_{\text{arm}} \leq \floor{ \theta_{ \floor{ \theta_{\floor{j/2}} } } } , \tau_{\text{com}} \leq \floor{ \theta_{\floor{j/2}} }  , \tau_{\text{spr}} > j   )  \leq \P (\bar{\tau}_{\text{spr}} >j/2 ) .
\end{equation}
Finally, by the union bound, we have
\begin{align}
\P ( \tau_{\text{spr}} > j  ) & \leq \P ( \tau_{\text{arm}} > \floor{ \theta_{ \floor{ \theta_{\floor{j/2}} } } } ) + \P ( \tau_{\text{com}} > \floor{ \theta_{\floor{j/2}} } ) \\
& \quad + \P ( \tau_{\text{arm}} \leq \floor{ \theta_{ \floor{ \theta_{\floor{j/2}} } } } , \tau_{\text{com}} \leq \floor{ \theta_{\floor{j/2}} }  , \tau_{\text{spr}} > j   ) ,
\end{align}
so combining the previous four inequalities yields the desired result.
\end{proof}

Finally, as a corollary, we can bound $\E [A_{ \tau_{\text{spr}}}]$.
\begin{cor} \label{corEarly}
Under the assumptions of Theorem \ref{thmProposed}, we have
\begin{align}
\E [A_{ \tau_{\text{spr}} }] & \leq  (J_\star)^\beta + 
\left( \frac{ 84^{ \rho_1^2 ( \beta ( 2 \alpha - 3 ) - 1 ) } ( 6^\beta  + 2 )   }{  (2 \alpha  - 3) ( \beta ( 2 \alpha  - 3 ) - 1 ) } + 3 \right) \frac{2 \beta n K^2 S }{\rho_1^2(\beta(2\alpha-3)-1)-\beta} +  \E [A_{ 2 \bar{\tau}_{\text{spr}} }]  \\
& = O \left(  S^{\beta/( \rho_1^2(\beta-1)) }  \vee ( S \log(S/\Delta_2) / \Delta_2^2 )^{\beta/(\beta-1)} \vee (\bar{d} \log(n+m) )^{\beta/\rho_1} \vee n K^2 S \right) +  \E [A_{ 2 \bar{\tau}_{\text{spr}} }]   ,
\end{align}
where $J_\star$ is defined as in Claim \ref{clmJstar} from Appendix \ref{appEarlyCalc}.
\end{cor}
\begin{proof}
We first observe
\begin{equation} \label{eqEarlyTwoTerms}
\E [A_{ \tau_{\text{spr}} }] = \sum_{j=1}^\infty ( A_j - A_{j-1} ) \P ( \tau_{\text{spr}} \geq j ) \leq A_{J_\star-1}  + \sum_{j=J_\star}^\infty ( A_j - A_{j-1} ) \P  ( \tau_{\text{spr}} \geq j ) .
\end{equation}
For the first term in \eqref{eqEarlyTwoTerms}, using Claim \ref{clmPropAj} from Appendix \ref{appBasicRes}, we compute
\begin{equation} \label{eqEarlyAjStar}
A_{J_\star-1}  = A_{J_\star} - ( A_{J_\star} - A_{J_\star-1} ) = \ceil{ ( J_\star )^\beta } - ( A_{J_\star} - A_{J_\star-1} ) \leq ( J_\star )^\beta + 1 - 1 = ( J_\star )^\beta .
\end{equation}
For the second term in \eqref{eqEarlyTwoTerms}, define the constant
\begin{equation} \label{eqEarlyCdef}
C =  \frac{ 84^{ \rho_1^2 ( \beta ( 2 \alpha - 3 ) - 1 ) } ( 6^\beta  + 2 ) \ }{  (2 \alpha  - 3) ( \beta ( 2 \alpha  - 3 ) - 1 ) } + 3 .
\end{equation}
Then by Theorem \ref{thmTail}, we have
\begin{align} 
\sum_{j=J_\star}^\infty ( A_j - A_{j-1} ) \P  ( \tau_{\text{spr}} \geq j ) & \leq  \sum_{j=J_\star}^\infty ( A_j - A_{j-1} ) \P ( \bar{\tau}_{\text{spr}} > j/2 ) \label{eqEarlyTwoMoreA} \\
& \quad  + C n K^2 S \sum_{j=J_\star}^\infty ( A_j - A_{j-1} ) j^{\rho_1^2 ( \beta (3-2\alpha) + 1 )}  . \label{eqEarlyTwoMoreB}
\end{align}
For \eqref{eqEarlyTwoMoreA}, we simply use nonnegativity to write
\begin{equation} \label{eqEarlySpread}
\sum_{j=J_\star}^\infty ( A_j - A_{j-1} ) \P ( \bar{\tau}_{\text{spr}} > j/2 ) \leq \sum_{j=1}^\infty ( A_j - A_{j-1} ) \P ( 2 \bar{\tau}_{\text{spr}} \geq j ) = \E [A_{ 2 \bar{\tau}_{\text{spr}} }] .
\end{equation}
For the second term in \eqref{eqEarlyTwoMoreB}, we use Claim \ref{clmPropAj} and $\beta > 1$ and $\rho_1^2 ( \beta(2\alpha-3) - 1 ) > \beta$ by the assumptions of Theorem \ref{thmProposed} to write
\begin{align}
& C n K^2 S \sum_{j=J_\star}^\infty ( A_j - A_{j-1} ) j^{\rho_1^2 ( \beta (3-2\alpha) + 1 )}  < 2 \beta C n K^2 S \sum_{j=J_\star}^\infty j^{-1-(\rho_1^2 ( \beta (2\alpha-3) - 1 ) - \beta)} \\
& \quad \leq 2 \beta C n K^2 S \int_{j=1}^\infty j^{-1-(\rho_1^2 ( \beta (2\alpha-3) - 1 ) - \beta)} dj  = \frac{ 2 \beta C  n K^2 S  }{ \rho_1^2 ( \beta (2\alpha-3) - 1 ) - \beta } .
\end{align}
Finally, combining the above bounds completes the proof.
\end{proof}

%% file: proofOther.tex
\section{Other proofs} \label{appProofOther}

\subsection{Basic inequalities} \label{appBasicRes}

In this sub-appendix, we prove some simple inequalities used frequently in the analysis.

\begin{clm} \label{clmPropAj}
For any $j \in \N$, we have
\begin{equation*}
j^\beta \leq A_j \leq  j^{2\beta}   , \quad \beta (j-1)^{\beta-1} - 1 < A_j - A_{j-1} < \beta j^{\beta-1} + 1 , \quad A_j - A_{j-1} \geq 1  ,
\end{equation*}
and for any $z \geq 1$ and $l \in \N$, we have $A_{l \ceil{z}} \leq e^{2\beta} (lz)^\beta$.
\end{clm}
\begin{proof}
For the first pair of inequalities, observe $A_j = j^\beta = j^{2\beta} = 1$ when $j=1$, and for $j \geq 2$,
\begin{equation} \label{eqPropAj}
j^\beta \leq A_j \leq j^\beta + 1 \leq 2 j^\beta \leq j^{1+\beta} \leq j^{2 \beta} .
\end{equation}
For the second pair of inequalities, we first observe 
\begin{equation} \label{eqPropAjMvt}
A_j - A_{j-1} > j^\beta - (j-1)^\beta - 1 = \beta x^{\beta-1} - 1 \geq \beta (j-1)^{\beta-1} - 1 ,
\end{equation}
where the equality holds for some $x \in [j-1,j]$ by the mean value theorem and the second inequality is $x \geq j-1$. By analogous reasoning, one can also show $A_j - A_{j-1} < \beta j^{\beta-1} + 1$, so the second pair of inequalities holds. The third inequality holds with equality when $j=1$, and for $j \geq 2$, the lower bound in \eqref{eqPropAjMvt} and $\beta > 1$ imply $A_j - A_{j-1} > 0$, so since $A_j$ and $A_{j-1}$ are integers, $A_j - A_{j-1} \geq 1$. Finally, using $z \geq 1$, $\beta > 1$, and $2 < e$, we can write
\begin{equation*}
A_{ l \ceil{z} } < ( l(z + 1) )^\beta + 1 < ( 2 l z )^\beta + ( 2 l z )^\beta = 2^{\beta+1} (lz)^\beta < e^{2\beta} (zl)^\beta . \qedhere
\end{equation*}
\end{proof}

\begin{clm}\label{clmPolySeries}
For any $j \in \N$ and $c > 1$, $\sum_{i=j+1}^\infty i^{-c} \leq j^{1-c} / (c-1)$.
\end{clm}
\begin{proof}
Since $j \geq 1$ and $c > 1$, we can write
\begin{align*}
& \sum_{i=j+1}^\infty i^{-c} = \sum_{i=j+1}^\infty \int_{x=i-1}^{i} i^{-c} d x \leq  \sum_{i=j+1}^\infty \int_{x=i-1}^{i} x^{-c} d x = \int_{x=j}^\infty x^{-c} dx = \frac{ j^{1-c} }{c-1} . \qedhere
\end{align*}
\end{proof}

\begin{clm} \label{clmLogTrick}
For any $x , y , z  > 0$ such that $x^y \leq z \log x$, $x < ( (2z/y) \log(2z/y) )^{1/y} < (2z/y)^{2/y}$.
\end{clm}
\begin{proof}
Multiplying and dividing the right side of the assumed inequality by $y/2$, we obtain $x^y \leq ( 2z/y ) \log x^{y/2}$. We can then loosen this bound to get $x^y < (2z/y) x^{y/2}$, or $x < ( 2 z / y )^{2/y}$. Plugging into the $\log$ term of the assumed inequality yields $x^y < ( 2 z / y ) \log ( 2 z / y )$. Raising both sides to the power $1/y$ establishes the first bound. The second bound follows by using $\log(2z/y) < 2z/y$.
\end{proof}

\begin{rem}
We typically apply Claim \ref{clmLogTrick} with $y$ constant but $z$ not. It allows us to invert inequalities of the form $x^y \leq z \log x$ to obtain $x = \tilde{O}(z^{1/y})$.
\end{rem}

\subsection{Bandit inequalities} \label{appBanditRes}

Next, we state and prove some basic bandit inequalities. The proof techniques are mostly modified from existing work (e.g., \cite{auer2002finite}), but we provide the bounds in forms useful for our setting.

\begin{clm} \label{clmBanditTail}
Suppose that $k_1 , k_2 \in [K]$, $t \in \N$, $\ell , u > 0$, and $\iota \in (0,1]$ satisfy
\begin{equation*}
\mu_{k_2} - \mu_{k_1} \geq \sqrt{ \alpha \log t } \left( \frac{ 2 }{ \sqrt{\ell} } - \frac{ 1 - \iota }{ \sqrt{u} } \right) .
\end{equation*}
Let $j \in \N$ be such that $t \in \{1+A_{j-1} , \ldots , A_j\}$, i.e., $j \in A^{-1}(t)$. Then for any $i \in [n]$, we have
\begin{equation*}
\P ( T_{k_1}^{(i)}(t-1) \geq \ell , T_{k_2}^{(i)} (A_j) \leq u , k_2 \in S_j^{(i)} ,  I_t^{(i)} = k_1 ) \leq 2 ( \floor{u} \wedge t ) t^{1 - 2 \alpha \iota^2 }  
\end{equation*}
\end{clm}
\begin{proof}
For $k \in [K]$, let $\{ X_k(s) \}_{s=1}^\infty$ be an i.i.d.\ sequence distributed as $\nu_k$, and for $s \in \N$, let
\begin{equation}
\hat{\mu}_k^{(i)}(s) = \frac{1}{s} \sum_{s'=1}^s X_k(s') , \quad U_k^{(i)} ( t , s ) = \hat{\mu}_k^{(i)}(s) + \sqrt{ \frac{\alpha \log t}{s} } 
\end{equation}
denote the empirical mean and UCB index of if $i$ has pulled the $k$-th arm $s$ times before $t$. Then Algorithm \ref{algGeneral} implies that if $k_2 \in S_j^{(i)}$ and $I_t^{(i)} = k_1$, we must have
\begin{equation} \label{eqBanditTailUcb}
U_{k_1}^{(i)} ( t , T_{k_1}^{(i)}(t-1) ) \geq U_{k_2}^{(i)} ( t , T_{k_2}^{(i)}(t-1) ) .
\end{equation}
Next, note that if $T_{k_2}^{(i)}(A_j) \leq u$ , then by monotonicity, $T_{k_2}^{(i)}(t-1) \leq u$ as well. Combined with the fact that $T_{k_2}^{(i)}(t-1) \in [t]$ by definition, we conclude that $T_{k_2}^{(i)}(A_j) \leq u$ implies $T_{k_2}^{(i)}(t-1) \leq \floor{u} \wedge t$. Similarly, $T_{k_1}^{(i)}(t-1) \geq \ell$ implies $T_{k_1}^{(i)}(t-1) \geq \ceil{\ell}$ (since $T_{k_1}^{(i)}(t-1) \in \N$). Combined with \eqref{eqBanditTailUcb}, we obtain that if the event in the statement of the claim occurs, it must be the case that
\begin{equation*}
\max_{ s_1 \in \{ \ceil{ \ell } , \ldots , t \} } U_{k_1}^{(i)} ( t , s_1 ) \geq \min_{ s_2 \in [ \floor{u} \wedge t ] } U_{k_2}^{(i)} ( t , s_2 ) .
\end{equation*}
Therefore, by the union bound, we obtain
\begin{equation} \label{eqBanditTailToPlays}
\P ( T_{k_1}^{(i)}(t-1) \geq \ell , T_{k_2}^{(i)} (A_j) \leq u , k_2 \in S_j^{(i)} ,  I_t^{(i)} = k_1 )  \leq \sum_{s_1 = \ceil{\ell}}^t \sum_{s_2=1}^{ \floor{u} \wedge t } \P ( U_{k_1}^{(i)} ( t , s_1 ) \geq U_{k_2}^{(i)} ( t , s_2 ) ) .
\end{equation}
Now fix $s_1$ and $s_2$ as in the double summation. We claim $U_{k_1}^{(i)} ( t , s_1 ) \geq U_{k_2}^{(i)} ( t , s_2 )$ implies
\begin{equation*}
\hat{\mu}_{k_1}^{(i)}(s_1) \geq \mu_{k_1} + \sqrt{ \alpha \log (t) / s_1 } \quad \text{or} \quad \hat{\mu}_{k_2}^{(i)}(s_2) \leq \mu_{k_2} - \iota \sqrt{  \alpha \log (t) / s_2 } .
\end{equation*}
Indeed, if instead both inequalities fail, then by choice of $s_1,s_2$ and the assumption of the claim,
\begin{align}
U_{k_1}^{(i)} (t,s_1) & < \mu_{k_1} + 2 \sqrt{  \alpha \log (t)/s_1 } \leq \mu_{k_1} + 2 \sqrt{  \alpha \log (t) / \ell } \\
& \leq \mu_{k_2} + ( 1 - \iota ) \sqrt{ \alpha \log (t) / u }  \leq \mu_{k_2} + (1-\iota) \sqrt{  \alpha \log (t) / s_2 } < U_{k_2}^{(i)} (t,s_2) ,
\end{align}
which is a contradiction. Thus, by the union bound, Hoeffding's inequality, and $\iota \in (0,1)$, we obtain
\begin{align*}
\P ( U_{k_1}^{(i)} ( t , s_1 ) \geq U_{k_2}^{(i)} ( t , s_2 ) ) & \leq \P ( \hat{\mu}_{k_1}^{(i)}(s_1) \geq \mu_{k_1} + \sqrt{  \alpha \log (t)/s_1} ) + \P (  \hat{\mu}_{k_2}^{(i)}(s_2) \leq \mu_{k_2} - \iota \sqrt{  \alpha \log (t)/s_2 }   ) \\
& \leq e^{ - 2 \alpha \log t } + e^{ - 2 \alpha \iota^2 \log t } = t^{ -2 \alpha } + t^{ - 2 \alpha \iota^2 } \leq 2 t^{ -2 \alpha \iota^2 } ,
\end{align*}
so plugging into \eqref{eqBanditTailToPlays} completes the proof.
\end{proof}

\begin{cor} \label{corBanditTailEll}
Suppose that $k_1 , k_2 \in [K]$, $t \in \N$,and  $\ell > 0$ satisfy $\mu_{k_2} - \mu_{k_1} \geq \sqrt{ 4 \alpha \log  (t) / \ell  }$. Let $j \in \N$ be such that $t \in \{1+A_{j-1} , \ldots , A_j\}$, i.e., $j = A^{-1}(t)$. Then for any $i \in [n]$, we have
\begin{equation*}
\P ( T_{k_1}^{(i)}(t-1) \geq \ell ,  k_2 \in S_j^{(i)} ,  I_t^{(i)} = k_1 ) \leq 2 t^{2 (1- \alpha) } .  
\end{equation*}
\end{cor}
\begin{proof}
Using $T_{k_2}^{(i)}(A_j) \leq A_j$ by definition and applying Claim \ref{clmBanditTail} with $u = A_j$ and $\iota = 1$,
\begin{align*}
\P ( T_{k_1}^{(i)}(t-1) \geq \ell ,  k_2 \in S_j^{(i)} ,  I_t^{(i)} = k_1 ) & = \P ( T_{k_1}^{(i)}(t-1) \geq \ell , T_{k_2}^{(i)}(A_j) \leq A_j ,  k_2 \in S_j^{(i)} ,  I_t^{(i)} = k_1 ) \\
& \leq 2 ( \floor{A_j} \wedge t ) t^{1-2\alpha} = t^{2(1-\alpha)} . \qedhere
\end{align*}
\end{proof}

\begin{cor} \label{corBanditTailDelta}
For any $i \in [n]$, $k \in [K]$, and $T \in \N$, we have
\begin{equation*}
\E \left[  \sum_{t = A_{\tau_{\text{spr}} } + 1}^T \ind \left(  I_t^{(i)} = k , T_k^{(i)}(t-1) \geq \frac{4 \alpha \log t}{\Delta_k^2} \right) \right] \leq  \frac{ 4(\alpha-1) }{ 2 \alpha - 3 }  .
\end{equation*}
\end{cor}
\begin{proof}
First observe that since $1 \in S_{A^{-1}(t)}^{(i)}$ whenever $t \geq A_{\tau_{\text{spr}}}+1$ by definition, we have
\begin{align} \label{eqBanditTailDeltaAs}
\sum_{t = A_{\tau_{\text{spr}} } + 1}^T \ind \left(  I_t^{(i)} = k , T_k^{(i)}(t-1) \geq \frac{4 \alpha \log t}{\Delta_k^2} \right) \leq \sum_{t=1}^\infty \ind \left( T_k^{(i)}(t-1) \geq \frac{4 \alpha \log t}{\Delta_k^2} , 1 \in S_{A^{-1}(t)}^{(i)} , I_t^{(i)} = k \right) .
\end{align}
Next, let $k_1 = k$, $k_2 = 1$, and $\ell = 4 \alpha \log(t) / \Delta_k^2$. Then by definition, we have $\mu_{k_2} - \mu_{k_1} = \Delta_k = \sqrt{ 4 \alpha \log(t) / \ell }$. Therefore, we can use Corollary \ref{corBanditTailEll} to obtain
\begin{equation}  \label{eqBanditTailDeltaProb}
\P ( T_k^{(i)}(t-1) \geq 4 \alpha \log (t) / \Delta_k^2 , 1 \in S_{A^{-1}(t)}^{(i)} , I_t^{(i)} = k ) \leq 2 t^{2(1-\alpha)} .
\end{equation}
Hence, taking expectation in \eqref{eqBanditTailDeltaAs}, then plugging in \eqref{eqBanditTailDeltaProb} to the right side and using Claim \ref{clmPolySeries}, yields
\begin{equation*}
\E \left[ \sum_{t = A_{\tau_{\text{spr}} } + 1}^T \ind \left(  I_t^{(i)} = k , T_k^{(i)}(t-1) \geq \frac{4 \alpha \log t}{\Delta_k^2} \right) \right] \leq 2 \left( 1 + \sum_{t=2}^\infty t^{2(1-\alpha)} \right) \leq  \frac{ 4(\alpha-1) }{ 2 \alpha - 3 }   .\qedhere
\end{equation*}
\end{proof}

\begin{clm} \label{clmBanditPlays}
For any $i \in [n]$, $t_1 , t_2 \in \N$ such that $t_1 < t_2$, and $\{ \ell_t \}_{t=t_1}^{t_2} \subset (0,\infty)$,
\begin{equation*}
\sum_{t=t_1}^{t_2} \ind \left( I_t^{(i)} = k , T_k^{(i)} (t-1) < \ell_t \right) \leq \max_{ t \in \{t_1,\ldots,t_2\} } \ell_t .
\end{equation*}
\end{clm}
\begin{proof}
Set $\ell = \max_{ t \in \{t_1,\ldots,t_2\} } \ell_t$. Then clearly
\begin{equation*}
\sum_{t=t_1}^{t_2} \ind \left( I_t^{(i)} = k , T_k^{(i)} (t-1) < \ell_t \right) \leq \sum_{t=t_1}^{t_2} \ind \left( I_t^{(i)} = k , T_k^{(i)} (t-1) < \ell \right) .
\end{equation*}
Now suppose the right strictly exceeds $\ell$. Then since the right side is an integer, we can find $\ceil{\ell}$ times $t \in \{t_1,\ldots,t_2\}$ such that $I_t^{(i)} = k$ and $T_k^{(i)}(t-1) < \ell$. Let $\bar{t}$ denote the largest such $t$. Because $I_t^{(i)} = k$ occurred at least $\ceil{\ell}-1$ times before $t$, we know $T_k^{(i)}(\bar{t}-1) \geq \ceil{\ell} - 1$. But since $T_k^{(i)}(\bar{t}-1) < \ell$, this implies $\ell + 1 > \ceil{\ell}$, which is a contradiction.
\end{proof}

Finally, we recall a well-known regret decomposition.
\begin{clm} \label{clmDecomposition}
The regret $R_T^{(i)}$ defined in \eqref{eqDefnRegret} satisfies $R_T^{(i)} = \sum_{k=2}^K \Delta_k \E [ \sum_{t=1}^T \ind ( I_t^{(i)} = k ) ]$.
\end{clm}
\begin{proof}
See, e.g., the proof of \cite[Lemma 4.5]{lattimore2020bandit}.
\end{proof}

\subsection{Calculations for the early regret} \label{appEarlyCalc}

In this sub-appendix, we assume $\alpha$, $\beta$, $\eta$, $\theta_j$, $\kappa_j$, $\rho_1$, and $\rho_2$ are chosen as in Theorem \ref{thmProposed}. Recall $C_i$, $C_i'$, etc.\ denote constants associated with Claim $i$ that only depend on $\alpha$, $\beta$, $\eta$, $\rho_1$, and $\rho_2$.

\begin{clm} \label{clmThetaLower}
There exists $C_{\ref{clmThetaLower}}, C_{\ref{clmThetaLower}}' > 0$ such that $\floor{ \theta_{\floor{j/2}} } \geq C_{\ref{clmThetaLower}} j^{\rho_1}$ and $\floor{ \theta_{ \floor{ \theta_{ \floor{j/2} } } } } \geq C_{\ref{clmThetaLower}} j^{\rho_1^2}\ \forall\ j \geq C_{\ref{clmThetaLower}}'$.
\end{clm}
\begin{proof}
This follows from the choice $\theta_{j'} = (j'/3)^{\rho_1}\ \forall\ j' \in \N$ in Theorem \ref{thmProposed}.
\end{proof}

\begin{clm} \label{clmJ2starFin}
There exists $C_{\ref{clmJ2starFin}} > 0$ such that $\floor{ \theta_{ \floor{j/2} } } \geq J_2^\star\ \forall\ j \geq C_{\ref{clmJ2starFin}}$.
\end{clm}
\begin{proof}
This follows from Claim \ref{clmThetaLower} and the fact that $J_2^\star$ is a constant by definition \eqref{eqJ2starDefn}.
\end{proof}

\begin{clm} \label{clmAjDiffOverS}
There exists $C_{\ref{clmAjDiffOverS}} , C_{\ref{clmAjDiffOverS}}' > 0$ such that for any $j  \geq C_{\ref{clmAjDiffOverS}} S^{1/(\beta-1)}$, 
\begin{equation} \label{eqAjDiffOverS}
\left(\frac{A_j-A_{j-1}}{S+2} - 1 \right) \vee 1 = \frac{A_j-A_{j-1}}{S+2} - 1 \geq \frac{ j^{\beta-1} }{ C_{\ref{clmAjDiffOverS}}' S } .
\end{equation}
\end{clm}
\begin{proof}
By Claim \ref{clmPropAj}, we can find $C , C' > 0$ depending only on $\beta$ such that $A_j - A_{j-1} \geq C j^{\beta-1}$ whenever $j \geq C'$. Hence, for any $j \geq (6 S / C )^{1/(\beta-1)} \vee C'$, we know $A_j - A_{j-1} \geq C j^{\beta-1} \geq 6 S$, so
\begin{gather}
\frac{A_j-A_{j-1}}{S+2} - 1 \geq \frac{ C j^{\beta-1} - 3 S}{3S} \geq \frac{ C j^{\beta-1} - C j^{\beta-1} / 2 }{ 3 S } = \frac{ C j^{\beta-1} }{ 6 S } \geq 1 ,
\end{gather}
where we also used $S\geq 1$. The claim follows if we set $C_{\ref{clmAjDiffOverS}} = (6 / C )^{1/(\beta-1)} \vee C'$ and $C_{\ref{clmAjDiffOverS}}' = C / 6$.
\end{proof}

\begin{clm} \label{clmUglyDiff}
There exists $C_{\ref{clmUglyDiff}} , C_{\ref{clmUglyDiff}}' > 0$ such that for any $j \geq C_{\ref{clmUglyDiff}} S^{1/(\beta-1)}$,
\begin{equation} \label{eqUglyDiff}
\log(A_{j-1} \vee 1) \geq \beta \log(j) / 2 > 0 , \quad \frac{1-\psi}{\sqrt{\kappa_j} } - \frac{2}{ \sqrt{(\frac{A_j-A_{j-1}}{S+2} - 1 ) \vee 1 }} > \sqrt{ \frac{ C_{\ref{clmUglyDiff}}' K^2 S }{  j^{\rho_2} } }  . 
\end{equation}
\end{clm}
\begin{proof}
By Claim \ref{clmAjDiffOverS}, we can find constants $C_{\ref{clmAjDiffOverS}} , C_{\ref{clmAjDiffOverS}}' > 0$ such that for any $j  \geq C_{\ref{clmAjDiffOverS}} S^{1/(\beta-1)}$,
\begin{equation*}
2 / \sqrt{( [ (A_j-A_{j-1} ) / (S+2) ] - 1 ) \vee 1 } \leq \sqrt{ 4 C_{\ref{clmAjDiffOverS}}' S / j^{\beta-1} }  \leq \sqrt{  4 C_{\ref{clmAjDiffOverS}}' K^2 S / j^{\beta-1}  } ,
\end{equation*}
where we also used $K \geq 1$. Furthermore, since $\rho_2 < \beta - 1$ by assumption, we can find $C > 0$ depending only on $C_{\ref{clmAjDiffOverS}}'$, $\psi$, $\beta$, and $\rho_2$, such that for any $j \geq C$, we have $4 C_{\ref{clmAjDiffOverS}}' / j^{\beta-1} < (1-\psi)^2 / ( 4 j^{\rho_2} )$. Combined with the previous inequality and the choice $\kappa_j = j^{\rho_2} / ( K^2 S )$ in Theorem \ref{thmProposed}, we obtain
\begin{equation*}
\frac{1-\psi}{\sqrt{\kappa_j} } - \frac{2}{ \sqrt{(\frac{A_j-A_{j-1}}{S+2} - 1 ) \vee 1 }} >  \sqrt{ \frac{(1-\psi)^2 K^2 S}{ 4 j^{\rho_2}} }\ \forall\ j \geq ( C_{\ref{clmAjDiffOverS}} \vee C ) S^{1/(\beta-1)} .
\end{equation*}
Hence, if we set $C_{\ref{clmUglyDiff}} = C_{\ref{clmAjDiffOverS}} \vee C \vee 4$ and $C_{\ref{clmUglyDiff}}' = (1-\psi)^2 / 4$, the second inequality in \eqref{eqUglyDiff} holds for $j \geq C_{\ref{clmUglyDiff}} S^{1/(\beta-1)}$. Finally, define $h(j) = j-1-\sqrt{j}\ \forall\ j \in \N$. Then $h(4) = 1$ and $h'(j) = 1 - 1/(2\sqrt{j}) > 0\ \forall\ j \geq 4$, so $h(j) \geq 0\ \forall\ j \geq 4$. Thus, for any $j \geq C_{\ref{clmUglyDiff}} S^{1/(\beta-1)} \geq 4$, we know $j-1 \geq \sqrt{j}$, so by Claim \ref{clmPropAj}, $\log ( A_{j-1}  ) \geq \log ( (j-1)^\beta ) \geq \log ( \sqrt{j}^\beta ) = \beta \log(j) / 2$, i.e., the first inequality in \eqref{eqUglyDiff} holds.
\end{proof}

\begin{clm} \label{clmJ1starFin}
There exists $C_{\ref{clmJ1starFin}} > 0$ such that for any $j \geq C_{\ref{clmJ1starFin}} S^{1/(\rho_1^2(\beta-1))}$, $\floor{ \theta_{ \floor { \theta_{ \floor{j/2} } } } } \geq J_1^\star$.
\end{clm}
\begin{proof}
By Claims \ref{clmAjDiffOverS} and \ref{clmUglyDiff}, we can set $C = C_{\ref{clmAjDiffOverS}} \vee C_{\ref{clmUglyDiff}}$ to ensure that for $j \geq C S^{1/(\beta-1)}$, $A_j - A_{j-1} \geq 2(S+2)$ and $\delta_{j,2} > 0$. Hence, $J_1^\star \leq C S^{1/(\beta-1)}$ by definition \eqref{eqJ1starDefn}. On the other hand, by Claim \ref{clmThetaLower}, we know $\floor{ \theta_{ \floor { \theta_{ \floor{j/2} } } } } \geq C_{\ref{clmThetaLower}} j^{\rho_1^2}$ for $j \geq C_{\ref{clmThetaLower}}'$. Thus, if we set $C_{\ref{clmJ1starFin}} = (C / C_{\ref{clmThetaLower}})^{1/\rho_1^2} \vee C_{\ref{clmThetaLower}}'$, then for any $j \geq C_{\ref{clmJ1starFin}} S^{1/(\rho_1^2 (\beta-1))}$, we obtain $\floor{ \theta_{ \floor { \theta_{ \floor{j/2} } } } } \geq C_{\ref{clmThetaLower}} j^{\rho_1^2}$ (since $j \geq C_{\ref{clmThetaLower}}'$) and $C_{\ref{clmThetaLower}}  j^{\rho_1^2} \geq C S^{1/(\beta-1)}$ (since $j \geq (C / C_{\ref{clmThetaLower}})^{1/\rho_1^2} S^{1/(\rho_1^2 (\beta-1))} = ( (C / C_{\ref{clmThetaLower}}) S^{1/(\beta-1)} )^{1/\rho_1^2}$), which implies $\floor{ \theta_{ \floor { \theta_{ \floor{j/2} } } } } \geq C S^{1(\beta-1)} \geq J_1^\star$.
\end{proof}

\begin{clm} \label{clmJ3starFin}
There exists $C_{\ref{clmJ3starFin}} > 0$ such that for any $j \geq C_{\ref{clmJ3starFin}} S^{1/(\rho_1^2(\beta-1))}$, $\floor{\theta_{\floor{j/2}}} \geq J_3^\star$.
\end{clm}
\begin{proof}
We first upper bound $J_3^\star$. By Claim \ref{clmThetaLower}, we can find $C_{\ref{clmThetaLower}}, C_{\ref{clmThetaLower}}' > 0$ such that $\floor{\theta_j} \geq C_{\ref{clmThetaLower}} j^{\rho_1}$ when $j \geq C_{\ref{clmThetaLower}}'$. Let $j \geq C_{\ref{clmThetaLower}}' \vee ( ( C_{\ref{clmAjDiffOverS}} / C_{\ref{clmThetaLower}} ) S^{1/(\beta-1)} )^{1/\rho_1}$, where $C_{\ref{clmAjDiffOverS}}$ is from Claim \ref{clmAjDiffOverS}, and $j' \in \{ \floor{ \theta_j } , \ldots , j \}$. Then $j' \geq \floor{\theta_j} \geq C_{\ref{clmThetaLower}} j^{\rho_1}  \geq C_{\ref{clmAjDiffOverS}} S^{1/(\beta-1)}$, so we can find $C_{\ref{clmAjDiffOverS}}' > 0$ such that
\begin{equation*}
(K+1) \delta_{j',1} \leq (K+1) \sqrt{ \frac{ 4 \alpha \log (A_{j'}) C_{\ref{clmAjDiffOverS}}' S  }{ (j')^{\beta-1} } } \leq \sqrt{ \frac{ ( 2 K)^2  4 \alpha \log ( j^{2 \beta} ) C_{\ref{clmAjDiffOverS}}'  }{ ( C_{\ref{clmThetaLower}} j^{\rho_1}  )^{\beta-1} } } = \sqrt{ \frac{ C K^2 S \log j }{ j^{\rho_1(\beta-1)} } } ,
\end{equation*}
where the first inequality uses Claim \ref{clmAjDiffOverS}; the second inequality uses $K \geq 1$, $j' \geq \floor{\theta_j} \geq C_{\ref{clmThetaLower}} j^{\rho_1}$, and Claim \ref{clmPropAj}; and the equality defines $C = 32 \alpha \beta C_{\ref{clmAjDiffOverS}}' / C_{\ref{clmThetaLower}}^{\beta-1}$. On the other hand, if $j \geq C_{\ref{clmUglyDiff}} S^{1/(\beta-1)}$, where $C_{\ref{clmUglyDiff}}$ is from Claim \ref{clmUglyDiff}, we can find $C_{\ref{clmUglyDiff}}' > 0$ such that
\begin{equation*}
\delta_{j,2} > \sqrt{ \alpha \beta \log(j) C_{\ref{clmUglyDiff}}'  K^2 S / (2 j^{\rho_2}) } = \sqrt{ C' K^2 S \log(j) / j^{\rho_2} } ,
\end{equation*}
where the inequality uses Claim \ref{clmUglyDiff} and the equality defines $C' = \alpha \beta C_{\ref{clmUglyDiff}}' / 2$. Finally, by assumption $\rho_2 < \rho_1(\beta-1)$, we can find $C'' > 0$ such that, for any $j \geq C''$, we have $C' / j^{\rho_2} \geq C / j^{\rho_1(\beta-1)}$. Combined with the previous two inequalities, we obtain that for $j \geq C_{\ref{clmThetaLower}}' \vee ( (C_{\ref{clmAjDiffOverS}} / C_{\ref{clmThetaLower}}) S^{1/(\beta-1)} )^{1/\rho_1} \vee ( C_{\ref{clmUglyDiff}} S^{1/(\beta-1)} ) \vee C''$ and any $j' \in \{ \floor{\theta_j} , \ldots , j \}$, $\delta_{j,2} > (K+1) \delta_{j',1}$. Therefore, by definition of $J_3^\star$ \eqref{eqJ3starDefn}, we conclude that $J_3^\star \leq \tilde{C} S^{1/(\rho_1(\beta-1))}$, where $\tilde{C} = C_{\ref{clmThetaLower}}' \vee (C_{\ref{clmAjDiffOverS}} / C_{\ref{clmThetaLower}}) \vee C_{\ref{clmUglyDiff}} \vee C''$. Therefore, if we set $C_{\ref{clmJ3starFin}} = C_{\ref{clmThetaLower}}' \vee ( \tilde{C} / C_{\ref{clmThetaLower}} )^{1/\rho_1}$, we obtain that for any $j \geq C_{\ref{clmJ3starFin}} S^{1/(\rho_1^2(\beta-1))}$, $\floor{\theta_{\floor{j/2}}} \geq C_{\ref{clmThetaLower}} j^{\rho_1}$ (since $j \geq C_{\ref{clmThetaLower}}'$) and $C_9 j^{\rho_1} \geq \tilde{C} S^{1/(\rho_1(\beta-1))}$ (since $j \geq ( \tilde{C} / C_{\ref{clmThetaLower} } )^{1/\rho_1} S^{1/(\rho_1^2(\beta-1))} = ( \tilde{C} S^{1/(\rho_1(\beta-1))}  / C_{\ref{clmThetaLower} } )^{1/\rho_1}$), so stringing together the inequalities, we conclude $\floor{\theta_{\floor{j/2}}} \geq C_{\ref{clmThetaLower}} j^{\rho_1} \geq \tilde{C} S^{1/(\rho_1(\beta-1))} \geq J_3^\star$.
\end{proof} 

\begin{clm} \label{clmJ4star}
There exists $C_{\ref{clmJ4star}} > 0$ such that $J_4^\star \leq C_{\ref{clmJ4star}} ( S  \log ( C_{\ref{clmJ4star}} S / \Delta_2^2 ) /\Delta_2^2)^{1/(\beta-1)}$.
\end{clm}
\begin{proof}
Let $\tilde{C}_{\ref{clmJ4star}} = 16 C_{\ref{clmAjDiffOverS}}' \alpha \beta  / (\beta-1)$ and $C_{\ref{clmJ4star}} = \tilde{C}_{\ref{clmJ4star}} \vee ( 3 \tilde{C}_{\ref{clmJ4star}}^{1/(\beta-1)} ) \vee ( 3 C_{\ref{clmAjDiffOverS}} ) \vee 16$, where $C_{\ref{clmAjDiffOverS}}$ and $C_{\ref{clmAjDiffOverS}}'$ are the constants from Claim \ref{clmAjDiffOverS}. Also define $J_4^\dagger =  C_{\ref{clmJ4star}}  ( S \log ( C_{\ref{clmJ4star}} S / \Delta_2^2 ) / \Delta_2^2 )^{1/(\beta-1)}$. Then by definition of $J_4^\star$ \eqref{eqJ4starDefn}, it suffices to show $\delta_{j,1} < \Delta_2\ \forall\ j \geq \floor{ \floor{J_4^\dagger}/2}$. Fix such a $j$ and suppose instead that $\delta_{j,1} \geq \Delta_2$. Since $C_{\ref{clmJ4star}} \geq 16$, we know $J_4^\dagger \geq C_{\ref{clmJ4star}} ( \log C_{\ref{clmJ4star}} )^{1/(\beta-1)} \geq 16$, so
\begin{equation} \label{eqJtoJ4dag}
j \geq ( \floor{J_4^\dagger}/2 ) - 1 \geq ( (J_4^\dagger - 1 ) / 2 ) - 1 = (J_4^\dagger / 2 ) - (3/2) > J_4^\dagger / 3 .
\end{equation}
Hence, because $C_{\ref{clmJ4star}} \geq  3 C_{\ref{clmAjDiffOverS}}$, we have $j \geq J_4^\dagger / 3 \geq C_{\ref{clmAjDiffOverS}} S^{1/(\beta-1)}$, so by Claims \ref{clmAjDiffOverS} and \ref{clmPropAj}, respectively,
\begin{equation} \label{eqJ4starUglyDiff}
\delta_{j,1}^2 \leq  4 \alpha \log(A_j) C_{\ref{clmAjDiffOverS}}' S / j^{\beta-1} \leq  4 \alpha \log(j^{2\beta}) C_{\ref{clmAjDiffOverS}}' S / j^{\beta-1} =  8 C_{\ref{clmAjDiffOverS}}' \alpha \beta S \log(j) / j^{\beta-1} .
\end{equation}
Rearranging and using the assumption $\delta_{j,1} \geq \Delta_2$, this implies $j^{\beta-1} \leq 8 C_{\ref{clmAjDiffOverS}}' \alpha \beta S \log(j) / \Delta_2^2$. Hence, applying Claim \ref{clmLogTrick} with $x=j$, $y = \beta-1$, and $z = 8 C_{\ref{clmAjDiffOverS}}' \alpha \beta S / \Delta_2^2$, we obtain
\begin{equation*}
j \leq ( 16 C_{\ref{clmAjDiffOverS}}' \alpha \beta S   \log (   16 C_{\ref{clmAjDiffOverS}}' \alpha \beta S  / ( (\beta-1) \Delta_2^2 ) ) / ( (\beta-1) \Delta_2^2 ) )^{1/(\beta-1)} = ( \tilde{C}_{\ref{clmJ4star}} S \log ( \tilde{C}_{\ref{clmJ4star}} S / \Delta_2^2 ) / \Delta_2^2 )^{1/(\beta-1)} .
\end{equation*}
But since $C_{\ref{clmJ4star}} \geq \tilde{C}_{\ref{clmJ4star}} \vee ( 3 \tilde{C}_{\ref{clmJ4star}}^{1/(\beta-1)} )$, we have shown $j \leq J_4^\dagger / 3$, which contradicts \eqref{eqJtoJ4dag}.
\end{proof}

\begin{clm} \label{clmExpToPoly}
There exists $C_{\ref{clmExpToPoly}}> 0$ such that, for any $j \geq (C_{\ref{clmExpToPoly}} \bar{d} \log( C_{\ref{clmExpToPoly}} \bar{d} (n+m) ) )^{1/\rho_1}$,
\begin{equation} \label{eqExpToPoly}
(n+m)^3 \exp ( - \floor{ \theta_{ \floor{j/2}}} / (3 \bar{d} ) ) \leq j^{\rho_1^2 ( \beta ( 3 - 2 \alpha ) + 1 ) } .
\end{equation}
\end{clm}
\begin{proof}
Let $\tilde{C}_{\ref{clmExpToPoly}} =12 \rho_1 ( \beta ( 2 \alpha - 3 ) - 1 ) / C_{\ref{clmThetaLower}}$ and $C_{\ref{clmExpToPoly}} = C_{\ref{clmThetaLower}}' \vee 3 \vee ( 18 / C_{\ref{clmThetaLower}} ) \vee \tilde{C}_{\ref{clmExpToPoly}}$, and suppose instead that \eqref{eqExpToPoly} fails for some $j \geq (C_{\ref{clmExpToPoly}} \bar{d} \log( C_{\ref{clmExpToPoly}} \bar{d} (n+m) ) )^{1/\rho_1}$. Then we can write
\begin{equation} \label{eqExpToPolyInit}
\floor{ \theta_{ \floor{j/2}}} < 9 \bar{d} \log(n+m) + 3 \bar{d} \rho_1^2 ( \beta ( 2 \alpha - 3 ) - 1 ) \log j .
\end{equation}
Since $C_{\ref{clmExpToPoly}} \geq C_{\ref{clmThetaLower}}' \vee 3$ and $\rho_1 \in (0,1)$, we know that $j \geq ( C_{\ref{clmThetaLower}}' \log(3) )^{1/\rho_1} > C_{\ref{clmThetaLower}}'$, so $\floor{ \theta_{ \floor{j/2}}} \geq C_{\ref{clmThetaLower}} j^{\rho_1}$ by Claim \ref{clmThetaLower}. Since $C_{\ref{clmExpToPoly}} \geq (18 / C_{\ref{clmThetaLower}}) \vee 1$ and $\bar{d} \geq 1$, we also have $j \geq ( ( 18 / C_{\ref{clmThetaLower}} ) \log(n+m) )^{1/\rho_1}$, or $C_{\ref{clmThetaLower}} j^{\rho_1} / 2 \geq 9 \bar{d} \log(n+m)$. Combining these two bounds with \eqref{eqExpToPolyInit}, we conclude
\begin{equation} \label{eqExpToPolyFin}
j^{\rho_1} < ( 6 \rho_1^2 ( \beta ( 2 \alpha - 3 ) - 1 ) / C_{\ref{clmThetaLower}} )  \bar{d} \log j = \tilde{C}_{\ref{clmExpToPoly}} (\rho_1/2) \bar{d} \log j .
\end{equation}
Applying Claim \ref{clmLogTrick} with $x = j$, $y = \rho_1$, and $z = \tilde{C}_{\ref{clmExpToPoly}} (\rho_1/2) \bar{d}$, we obtain $j < ( \tilde{C}_{\ref{clmExpToPoly}} \bar{d} \log ( \tilde{C}_{\ref{clmExpToPoly}} \bar{d} ) )^{1/\rho_1}$. But since $C_{\ref{clmExpToPoly}} \geq \tilde{C}_{\ref{clmExpToPoly}}$, this contradicts the assumed lower bound on $j$.
\end{proof}

\begin{clm} \label{clmJstar}
Define $C_{\ref{clmJstar}} = C_{\ref{clmJ2starFin}} \vee C_{\ref{clmJ1starFin}} \vee C_{\ref{clmJ3starFin}}$ and
\begin{align*}
J_\star & = ( C_{\ref{clmJstar}} S^{1/(\rho_1^2(\beta-1))} ) \vee ( C_{\ref{clmJ4star}}  (S \log ( C_{\ref{clmJ4star}} S / \Delta_2^2 ) /\Delta_2^2)^{1/(\beta-1)} ) \vee  (C_{\ref{clmExpToPoly}} \bar{d} \log( C_{\ref{clmExpToPoly}}'(n+m) ) )^{1/\rho_1}  . \\
& = \Theta \left( S^{1/( \rho_1^2(\beta-1)) }  \vee ( S \log(S/\Delta_2) / \Delta_2^2 )^{1/(\beta-1)} \vee (\bar{d} \log(n+m) )^{1/\rho_1} \right) .
\end{align*}
Then for any $j \in \N$ such that $j \geq J_\star$, we have
\begin{gather} \label{eqLemJstar}
j \geq J_4^\star , \quad \floor{\theta_{\floor{j/2}}} \geq J_2^\star \vee  J_3^\star , \quad \floor{ \theta_{ \floor { \theta_{ \floor{j/2} } } } } \geq J_1^\star \vee ( 2 +  j^{\rho_1^2 } / 84 ) \geq 3 \\
(n+m)^3 \exp ( - \floor{ \theta_{ \floor{j/2}}}  / ( 3 \bar{d} ) ) \leq j^{\rho_1^2 ( \beta ( 3 - 2 \alpha ) + 1 ) } .
\end{gather}
\end{clm}
\begin{proof}
The first bound holds by $j \geq C_{\ref{clmJ4star}} ( (S/\Delta_2^2) \log ( C_{\ref{clmJ4star}}' S / \Delta_2^2 ) )^{1/(\beta-1)}$ and Claim \ref{clmJ4star}. The second holds since $\floor{\theta_{\floor{j/2}}} \geq J_2^\star$  by $j \geq C_{\ref{clmJ2starFin}}$ and Claim \ref{clmJ2starFin}, and since $\floor{\theta_{\floor{j/2}}} \geq J_3^\star$ by $j \geq  C_{\ref{clmJ3starFin}} S^{1/( \rho_1^2 (\beta-1) ) }$ and Claim \ref{clmJ3starFin}. The third holds since $\floor{ \theta_{ \floor { \theta_{ \floor{j/2} } } } } \geq J_1^\star$ by $j \geq  C_{\ref{clmJ1starFin}} S^{1/(\rho_1^2(\beta-1))}$ and Claim \ref{clmJ1starFin}, and since $\floor{ \theta_{ \floor { \theta_{ \floor{j/2} } } } } \geq ( 2 +  j^{\rho_1^2 } / 48 )$ for large enough $C_{\ref{clmThetaLower}}'$. The fourth holds since $\floor{ \theta_{ \floor { \theta_{ \floor{j/2} } } } } > 2$ (by the third) and $\floor{ \theta_{ \floor { \theta_{ \floor{j/2} } } } } \in \N$. The fifth holds by $j \geq (C_{\ref{clmExpToPoly}} \bar{d} \log( C_{\ref{clmExpToPoly}}'(n+m) ) )^{1/\rho_1}$ and Claim \ref{clmExpToPoly}.
\end{proof}

\subsection{Calculations for the later regret} \label{appIntCalc}

In this sub-appendix, we assume $\alpha$, $\beta$, $\eta$, $\theta_j$, $\kappa_j$, $\rho_1$, and $\rho_2$ are chosen as in Theorem \ref{thmProposed}. Recall $C_i$, $C_i'$, etc.\ denote constants associated with Claim $i$ that only depend on $\alpha$, $\beta$, $\eta$, $\rho_1$, and $\rho_2$.

\begin{clm} \label{clmSmallTtheta}
There exists $C_{\ref{clmSmallTtheta}} > 0$ such that, for any $\gamma_i \in (0,1)$,
\begin{equation} \label{eqSmallTtheta}
\floor{ \theta_{ \ceil{T^{\gamma_i/\beta}} } }^\beta \leq 4 \alpha K \log (T) / \Delta_2 \quad \Rightarrow \quad \log T \leq ( C_{\ref{clmSmallTtheta}} / \gamma_i  ) \log \left(  C_{\ref{clmSmallTtheta}} K / ( \Delta_2 \gamma_i  ) \right) .
\end{equation}
\end{clm}
\begin{proof}
Similar to Claim \ref{clmThetaLower} from Appendix \ref{appEarlyCalc}, we can find constants $C , C' > 0$ such that for any $j \geq C'$, $\floor{\theta_j} \geq C j^{\rho_1}$. If $\ceil{T^{\gamma_i/\beta}} < C'$, then $T^{\gamma_i}  \leq \ceil{T^{\gamma_i/\beta}}^\beta < ( C' )^\beta$, so $\log T \leq (\beta/\gamma_i) \log(C')$, and the right side of \eqref{eqSmallTtheta} will hold for $C_{\ref{clmSmallTtheta}} \geq \beta \vee C'$. If instead $\ceil{T^{\gamma_i/\beta}} \geq C'$, then $\floor{ \theta_{ \ceil{T^{\gamma_i/\beta}} } } \geq C \ceil{T^{\gamma_i/\beta}}^{\rho_1} \geq C T^{\gamma_i \rho_1 / \beta}$, so if the left side of \eqref{eqSmallTtheta} holds, we can write
\begin{equation*}
C^\beta T^{\gamma_i \rho_1} = ( C T^{\gamma_i \rho_1 / \beta} )^\beta \leq \floor{ \theta_{ \ceil{T^{\gamma_i/\beta}} } }^\beta \leq 4 \alpha \log(T) / \Delta_2 .
\end{equation*}
Hence, applying Claim \ref{clmLogTrick} with $x = T$, $y = \gamma_i \rho_1$, and $z = 4 \alpha / ( C^\beta \Delta_2 )$, we obtain 
\begin{equation*}
\log T \leq \log (   8 \alpha / ( \Delta_2 \gamma_i C^\beta \rho_1   ) )^{2/(\gamma_i \rho_1)} \leq ( C_{\ref{clmSmallTtheta}} / \gamma_i ) \log ( C_{\ref{clmSmallTtheta}} K / ( \Delta_2 \gamma_i ) ) ,
\end{equation*}
where the last inequality holds for any $C_{\ref{clmSmallTtheta}} \geq ( 8 \alpha / ( C^\beta \rho_1 ) ) \vee (2/\rho_1)$.
\end{proof}

\begin{clm} \label{clmSmallTKappa1}
There exists $C_{\ref{clmSmallTKappa1}} > 0$ such that, for any $\gamma_i \in (0,1)$,
\begin{equation*}
\kappa_{ \ceil{T^{\gamma_i/\beta} } } \leq 4 \alpha K \log (T) / \Delta_2 \quad \Rightarrow \quad \log T \leq  ( C_{\ref{clmSmallTKappa1}} / \gamma_i  ) \log \left(  C_{\ref{clmSmallTKappa1}} K / ( \Delta_2 \gamma_i  ) \right).
\end{equation*}
\end{clm}
\begin{proof}
Recall $\kappa_j = j^{\rho_2} / ( K^2 S )$ in Theorem \ref{thmProposed}. Hence, because $S \leq K$, we know that $\kappa_{ \ceil{T^{\gamma_i/\beta} } } \geq T^{\gamma_i \rho_2 /\beta} / K^3$. Rearranging and using the assumed bound, we obtain
\begin{equation*}
T^{\gamma_i \rho_2 /\beta} \leq K^3  \cdot \kappa_{ \ceil{T^{\gamma_i/\beta} } } \leq K^3 \cdot 4 \alpha K \log (T) / \Delta_2 = ( 4 \alpha K^4 / \Delta_2 )  \log T \leq ( 4 \alpha K / \Delta_2 )^4 \log T ,
\end{equation*}
where the last inequality uses $\alpha \geq 1$ and $\Delta_2 \in (0,1)$. Applying Claim \ref{clmLogTrick} with $x = T$, $y = \gamma_i \rho_2 / \beta$, and $z = ( 4 \alpha K / \Delta_2 )^4$, and noting that $2/y \leq (2/y)^4$ (since $\gamma_i \in (0,1)$ and $\rho_2 \in (0,\beta-1)$), we obtain
\begin{equation*}
\log T \leq \log ( 8 \alpha \beta K / ( \Delta_2 \gamma_i \rho_2 ) )^{8 \beta / (\gamma_i \rho_2 )} \leq  ( C_{\ref{clmSmallTKappa1}} / \gamma_i ) \log ( C_{\ref{clmSmallTKappa1}} K / ( \Delta_2 \gamma_i ) ) ,
\end{equation*}
where the second inequality holds for $C_{\ref{clmSmallTKappa1}} \geq 8 \alpha \beta / \rho_2$.
\end{proof}

\begin{clm} \label{clmSmallTKappa2}
There exists $C_{\ref{clmSmallTKappa2}} > 0$ such that, for any $\gamma_i \in (0,1)$,
\begin{equation*}
\exists\ j \geq \ceil{ T^{\gamma_i / \beta} }\ s.t.\ \ceil{\kappa_j} \leq 1 + 4 \alpha \log (A_j) / \Delta_2^2  \quad \Rightarrow \quad \log T \leq ( C_{\ref{clmSmallTKappa2}} / \gamma_i  ) \log \left(  C_{\ref{clmSmallTKappa2}} K / ( \Delta_2 \gamma_i  ) \right).
\end{equation*}
\end{clm}
\begin{proof}
Fix $j \geq \ceil{ T^{\gamma_i / \beta} }$ such that $\ceil{\kappa_j} \leq 1 + 4 \alpha \log(A_j) / \Delta_2^2$. Note that if $j = 1$, then $1 \geq \ceil{ T^{\gamma_i / \beta} } \geq T^{\gamma_i / \beta}$, so since $T \in \N$, we must have $T = 1$. This implies $\log T = 0$, so the claimed bound is immediate. Hence, we assume $j \geq 2$ moving forward. We first observe
\begin{equation*}
j^{\rho_2} \leq K^3 \ceil{ \kappa_j }  \leq K^3  \left( 1 + \frac{ 4 \alpha \log A_j }{ \Delta_2^2 } \right) \leq K^3 \left( 1 + \frac{ 8 \alpha \beta \log j }{ \Delta_2^2 } \right) \leq K^3\left(  \frac{ 16 \alpha \beta \log j }{ \Delta_2^2 } \right) = \frac{ 16 \alpha \beta K^3 \log j }{ \Delta_2^2 } .
\end{equation*}
where the inequalities use $\ceil{\kappa_j} = \ceil{ j^{\rho_2} / ( K^2 S ) } \geq j^{\rho_2} / K^3$, the assumed upper bound on $\ceil{\kappa_j}$, Claim \ref{clmPropAj}, and $8 \alpha \beta \log(j) / \Delta_2^2 \geq 1$ (since $\alpha , \beta \geq 1$, $j \geq 2$, and $\Delta_2 \in (0,1)$), respectively. Applying Claim \ref{clmLogTrick} with $x = j$, $y = \rho_2$, and $z = 16 \alpha \beta K^3 / \Delta_2^2$, and noting that $2 z / y \leq ( 32 \alpha \beta K / ( \Delta_2 \rho_2 ) )^3$ (since $\alpha \geq 1$, $\rho_2 \in (0,\beta-1)$, and $\Delta_2 \in (0,1)$), we obtain $j \leq ( 32 \alpha \beta K / ( \Delta_2 \rho_2 ) )^{6 / \rho_2} \leq (  C_{\ref{clmSmallTKappa2}} K / ( \Delta_2 \gamma_i ) )^{ 6 / \rho_2 }$ for any $C_{\ref{clmSmallTKappa2}} \geq 32 \alpha \beta / \rho_2$.  Therefore, by assumption $j \geq \ceil{ T^{\gamma_i / \beta} }$, we obtain that for any $C_{\ref{clmSmallTKappa2}} \geq 32 \alpha \beta / \rho_2$,
\begin{equation*}
T \leq \ceil{ T^{\gamma_i / \beta} }^{\beta/\gamma_i} \leq j^{\beta/\gamma_i} \leq  \left(  C_{\ref{clmSmallTKappa2}} K / ( \Delta_2 \gamma_i ) \right)^{ 6 \beta / ( \rho_2 \gamma_i ) } \leq \left(  C_{\ref{clmSmallTKappa2}} K / ( \Delta_2 \gamma_i ) \right)^{ C_{\ref{clmSmallTKappa2}} / \gamma_i } . \qedhere
\end{equation*}
\end{proof}

\begin{clm} \label{clmSmallTSum}
There exists $C_{\ref{clmSmallTSum}} > 0$ such that, for any $\gamma_i \in (0,1)$,
\begin{equation} \label{eqSmallTSum}
\frac{ \Delta_2 ( 2 \alpha - 3 ) }{ 8 \alpha K^2  } \leq \log (T) \sum_{j = \ceil{ T^{\gamma_i/\beta} } }^\infty ( \ceil{\kappa_j} - 1 )^{ 3-2\alpha }   \quad \Rightarrow \quad \log T \leq ( C_{\ref{clmSmallTSum}} / \gamma_i ) \log ( C_{\ref{clmSmallTSum}} K / ( \Delta_2 \gamma_i ) )  .
\end{equation}
\end{clm}
\begin{proof}
We first eliminate the corner case where $\min \{ T^{\gamma_i/\beta} , \kappa_{  \ceil{ T^{\gamma_i/\beta} }  }  \} < 2$. In this case, one of $T^{\gamma_i/\beta} < 2$ and $\kappa_{  \ceil{ T^{\gamma_i/\beta} }  } < 2$ must hold. If the former holds, then $\log T < ( \beta / \gamma_i ) \log 2$, and if the latter holds, then $2 > \kappa_{  \ceil{ T^{\gamma_i/\beta} }  }  =  \ceil{ T^{\gamma_i/\beta} }^{\rho_2} / ( K^2 S ) \geq  T^{ \gamma_i \rho_2 / \beta } / K^3$, so $\log T \leq ( \beta / (\gamma_i \rho_2) ) \log ( 2 K^3 )$. In both cases, we can clearly find $C_{\ref{clmSmallTSum}} > 0$ satisfying the right side of \eqref{eqSmallTSum}.

Next, we assume $\kappa_{  \ceil{ T^{\gamma_i/\beta} }  } \geq 2$ and $T^{\gamma_i/\beta} \geq 2$. By monotonicity, the former implies $\kappa_j \geq 2$ for any $j \geq \ceil{ T^{\gamma_i/\beta} }$. For any such $j$, by definition and $S \leq K$, we can then write
\begin{equation*}
\ceil{ \kappa_j } - 1 \geq \kappa_j - 1 \geq \kappa_j / 2 =  j^{\rho_2} / ( 2 K^2 S ) \geq  j^{\rho_2} / ( 2 K^3 ) .
\end{equation*}
Therefore, since $3-2\alpha < 0$ by assumption in Theorem \ref{thmProposed}, we obtain
\begin{equation*}
\sum_{j = \ceil{ T^{\gamma_i/\beta} } }^\infty ( \ceil{\kappa_j}- 1 )^{ 3-2\alpha}  \leq 2^{ 2 \alpha - 3 } K^{3(2 \alpha - 3)} \sum_{j = \ceil{ T^{\gamma_i/\beta} } }^\infty  j^{ \rho_2 ( 3 - 2\alpha ) } .
\end{equation*}
For the summation at right, we use $T^{\gamma_i/\beta}\geq 2$ (which implies $\ceil{ T^{\gamma_i/\beta} } - 1 \geq T^{\gamma_i/\beta} - 1 \geq T^{\gamma_i/\beta}/2$) and $\rho_2(2\alpha-3) > 1$ by assumption in Theorem \ref{thmProposed}, along with Claim \ref{clmPolySeries}, to write
\begin{equation*}
\sum_{j = \ceil{ T^{\gamma_i/\beta} } }^\infty  j^{ \rho_2 ( 3 - 2\alpha ) } \leq \frac{ ( T^{\gamma_i/\beta}/2 )^{ 1 + \rho_2 ( 3 - 2 \alpha ) } } {  \rho_2 (2\alpha-3) - 1  } = \frac{ 2^{ \rho_2 (2\alpha-3) } T^{ \gamma_i ( 1 + \rho_2(3-2\alpha) ) / \beta } }{  2 ( \rho_2 (2\alpha-3) - 1 ) } .
\end{equation*}
Using $\rho_2(2\alpha-3) > 1$ (by assumption), we also know
\begin{equation*}
T^{\gamma_i ( 1 + \rho_2 ( 3 - 2 \alpha ) )/\beta} \log(T) = \frac{ 2 \beta T^{\gamma_i ( 1 + \rho_2 ( 3 -2 \alpha ) )/\beta} \log( T^{  \gamma_i ( \rho_2(2\alpha-3) - 1 ) / ( 2 \beta )  } ) }{  \gamma_i (  \rho_2(2\alpha-3) - 1 )   } \leq  \frac{ 2 \beta T^{\gamma_i ( 1 +  \rho_2 ( 3 - 2 \alpha ) )/  ( 2 \beta ) }  }{  \gamma_i ( \rho_2(2\alpha-3) - 1 )  } .
\end{equation*}
Combining the previous three inequalities, we then obtain
\begin{equation*}
\log ( T )  \sum_{j = \ceil{ T^{\gamma_i/\beta} } }^\infty ( \ceil{\kappa_j}- 1 )^{ 3-2\alpha) } \leq  \frac{ 2^{ ( \rho_2 + 1 ) (2\alpha-3) } \beta  }{ (  \rho_2 (2\alpha-3) - 1 )^2  } \frac{ K^{ 3 (2 \alpha-3 ) } }{ \gamma_i } T^{\gamma_i ( 1 +  \rho_2 ( 3 - 2 \alpha ) )/  ( 2 \beta ) }  .
\end{equation*}
Therefore, if the left side of \eqref{eqSmallTSum} holds, we are guaranteed that
\begin{equation*}
\frac{ \Delta_2 ( 2\alpha-3 ) }{ 8 \alpha K^2 } \leq \frac{ 2^{ ( \rho_2 + 1 ) (2\alpha-3) } \beta  }{ (  \rho_2 (2\alpha-3) - 1 )^2  } \frac{ K^{ 3 (2 \alpha-3 ) } }{ \gamma_i } T^{\gamma_i ( 1 +  \rho_2 ( 3 - 2 \alpha ) )/  ( 2 \beta ) } ,
\end{equation*}
or, after rearranging, then using $\alpha > 1$ and $\Delta_2 , \gamma_i \in (0,1)$,
\begin{equation*}
T^{\gamma_i ( \rho_2 (  2 \alpha - 3 ) - 1 )/  ( 2 \beta ) } \leq \frac{ 8 \cdot 2^{ ( \rho_2 + 1 ) (2\alpha-3) }  \alpha \beta }{ (2\alpha-3) (  \rho_2 (2\alpha-3) - 1 )^2  } \frac{  K^{ 6 \alpha - 7 } }{ \Delta_2 \gamma_i } \leq ( C_{\ref{clmSmallTSum}} )^{ 6 \alpha }\frac{  K^{ 6 \alpha - 7 } }{ \Delta_2 \gamma_i } \leq \left( \frac{ C_{\ref{clmSmallTSum}} K }{ \Delta_2 \gamma_i }  \right)^{6\alpha} ,
\end{equation*}
where the second inequality holds for large $C_{\ref{clmSmallTSum}}$ and the third uses $\alpha \geq 1$ and $\Delta_2 , \gamma_i \in (0,1)$. Taking logarithms and choosing $C_{\ref{clmSmallTSum}}$ appropriately in terms of $\rho_2$, $\alpha$, and $\beta$ yields the right side of \eqref{eqSmallTSum}.
\end{proof}

%% file: main.bbl

\begin{thebibliography}{58}


\ifx \showCODEN    \undefined \def \showCODEN     #1{\unskip}     \fi
\ifx \showDOI      \undefined \def \showDOI       #1{#1}\fi
\ifx \showISBNx    \undefined \def \showISBNx     #1{\unskip}     \fi
\ifx \showISBNxiii \undefined \def \showISBNxiii  #1{\unskip}     \fi
\ifx \showISSN     \undefined \def \showISSN      #1{\unskip}     \fi
\ifx \showLCCN     \undefined \def \showLCCN      #1{\unskip}     \fi
\ifx \shownote     \undefined \def \shownote      #1{#1}          \fi
\ifx \showarticletitle \undefined \def \showarticletitle #1{#1}   \fi
\ifx \showURL      \undefined \def \showURL       {\relax}        \fi
\providecommand\bibfield[2]{#2}
\providecommand\bibinfo[2]{#2}
\providecommand\natexlab[1]{#1}
\providecommand\showeprint[2][]{arXiv:#2}

\bibitem[Anandkumar et~al\mbox{.}(2011)]%
        {anandkumar2011distributed}
\bibfield{author}{\bibinfo{person}{Animashree Anandkumar},
  \bibinfo{person}{Nithin Michael}, \bibinfo{person}{Ao~Kevin Tang}, {and}
  \bibinfo{person}{Ananthram Swami}.} \bibinfo{year}{2011}\natexlab{}.
\newblock \showarticletitle{Distributed algorithms for learning and cognitive
  medium access with logarithmic regret}.
\newblock \bibinfo{journal}{\emph{IEEE Journal on Selected Areas in
  Communications}} \bibinfo{volume}{29}, \bibinfo{number}{4}
  (\bibinfo{year}{2011}), \bibinfo{pages}{731--745}.
\newblock


\bibitem[Audibert and Bubeck(2010)]%
        {audibert2010best}
\bibfield{author}{\bibinfo{person}{Jean-Yves Audibert} {and}
  \bibinfo{person}{S{\'e}bastien Bubeck}.} \bibinfo{year}{2010}\natexlab{}.
\newblock \showarticletitle{Best Arm Identification in Multi-Armed Bandits}. In
  \bibinfo{booktitle}{\emph{COLT-23th Conference on Learning Theory-2010}}.
  \bibinfo{pages}{13--p}.
\newblock


\bibitem[Auer et~al\mbox{.}(2002)]%
        {auer2002finite}
\bibfield{author}{\bibinfo{person}{Peter Auer}, \bibinfo{person}{Nicolo
  Cesa-Bianchi}, {and} \bibinfo{person}{Paul Fischer}.}
  \bibinfo{year}{2002}\natexlab{}.
\newblock \showarticletitle{Finite-time analysis of the multiarmed bandit
  problem}.
\newblock \bibinfo{journal}{\emph{Machine learning}} \bibinfo{volume}{47},
  \bibinfo{number}{2-3} (\bibinfo{year}{2002}), \bibinfo{pages}{235--256}.
\newblock


\bibitem[Auer et~al\mbox{.}(1995)]%
        {auer1995gambling}
\bibfield{author}{\bibinfo{person}{Peter Auer}, \bibinfo{person}{Nicolo
  Cesa-Bianchi}, \bibinfo{person}{Yoav Freund}, {and} \bibinfo{person}{Robert~E
  Schapire}.} \bibinfo{year}{1995}\natexlab{}.
\newblock \showarticletitle{Gambling in a rigged casino: The adversarial
  multi-armed bandit problem}. In \bibinfo{booktitle}{\emph{Proceedings of IEEE
  36th Annual Foundations of Computer Science}}. IEEE,
  \bibinfo{pages}{322--331}.
\newblock


\bibitem[Avner and Mannor(2014)]%
        {avner2014concurrent}
\bibfield{author}{\bibinfo{person}{Orly Avner} {and} \bibinfo{person}{Shie
  Mannor}.} \bibinfo{year}{2014}\natexlab{}.
\newblock \showarticletitle{Concurrent bandits and cognitive radio networks}.
  In \bibinfo{booktitle}{\emph{Joint European Conference on Machine Learning
  and Knowledge Discovery in Databases}}. Springer, \bibinfo{pages}{66--81}.
\newblock


\bibitem[Awerbuch and Kleinberg(2008)]%
        {awerbuch2008competitive}
\bibfield{author}{\bibinfo{person}{Baruch Awerbuch} {and}
  \bibinfo{person}{Robert Kleinberg}.} \bibinfo{year}{2008}\natexlab{}.
\newblock \showarticletitle{Competitive collaborative learning}.
\newblock \bibinfo{journal}{\emph{J. Comput. System Sci.}}
  \bibinfo{volume}{74}, \bibinfo{number}{8} (\bibinfo{year}{2008}),
  \bibinfo{pages}{1271--1288}.
\newblock


\bibitem[Bar-On and Mansour(2019)]%
        {bar2019individual}
\bibfield{author}{\bibinfo{person}{Yogev Bar-On} {and} \bibinfo{person}{Yishay
  Mansour}.} \bibinfo{year}{2019}\natexlab{}.
\newblock \showarticletitle{Individual regret in cooperative nonstochastic
  multi-armed bandits}.
\newblock \bibinfo{journal}{\emph{Advances in Neural Information Processing
  Systems}}  \bibinfo{volume}{32} (\bibinfo{year}{2019}),
  \bibinfo{pages}{3116--3126}.
\newblock


\bibitem[Bargiacchi et~al\mbox{.}(2018)]%
        {bargiacchi2018learning}
\bibfield{author}{\bibinfo{person}{Eugenio Bargiacchi},
  \bibinfo{person}{Timothy Verstraeten}, \bibinfo{person}{Diederik Roijers},
  \bibinfo{person}{Ann Now{\'e}}, {and} \bibinfo{person}{Hado Hasselt}.}
  \bibinfo{year}{2018}\natexlab{}.
\newblock \showarticletitle{Learning to coordinate with coordination graphs in
  repeated single-stage multi-agent decision problems}. In
  \bibinfo{booktitle}{\emph{International conference on machine learning}}.
  PMLR, \bibinfo{pages}{482--490}.
\newblock


\bibitem[Bistritz and Bambos(2020)]%
        {bistritz2020cooperative}
\bibfield{author}{\bibinfo{person}{Ilai Bistritz} {and}
  \bibinfo{person}{Nicholas Bambos}.} \bibinfo{year}{2020}\natexlab{}.
\newblock \showarticletitle{Cooperative multi-player bandit optimization}.
\newblock \bibinfo{journal}{\emph{Advances in Neural Information Processing
  Systems}}  \bibinfo{volume}{33} (\bibinfo{year}{2020}).
\newblock


\bibitem[Bistritz and Leshem(2018)]%
        {bistritz2018distributed}
\bibfield{author}{\bibinfo{person}{Ilai Bistritz} {and} \bibinfo{person}{Amir
  Leshem}.} \bibinfo{year}{2018}\natexlab{}.
\newblock \showarticletitle{Distributed multi-player bandits-a game of thrones
  approach}. In \bibinfo{booktitle}{\emph{Advances in Neural Information
  Processing Systems}}. \bibinfo{pages}{7222--7232}.
\newblock


\bibitem[Bogunovic et~al\mbox{.}(2020)]%
        {bogunovic2020corruption}
\bibfield{author}{\bibinfo{person}{Ilija Bogunovic}, \bibinfo{person}{Andreas
  Krause}, {and} \bibinfo{person}{Jonathan Scarlett}.}
  \bibinfo{year}{2020}\natexlab{}.
\newblock \showarticletitle{Corruption-tolerant Gaussian process bandit
  optimization}. In \bibinfo{booktitle}{\emph{International Conference on
  Artificial Intelligence and Statistics}}. PMLR, \bibinfo{pages}{1071--1081}.
\newblock


\bibitem[Bogunovic et~al\mbox{.}(2021)]%
        {bogunovic2021stochastic}
\bibfield{author}{\bibinfo{person}{Ilija Bogunovic}, \bibinfo{person}{Arpan
  Losalka}, \bibinfo{person}{Andreas Krause}, {and} \bibinfo{person}{Jonathan
  Scarlett}.} \bibinfo{year}{2021}\natexlab{}.
\newblock \showarticletitle{Stochastic linear bandits robust to adversarial
  attacks}. In \bibinfo{booktitle}{\emph{International Conference on Artificial
  Intelligence and Statistics}}. PMLR, \bibinfo{pages}{991--999}.
\newblock


\bibitem[Boursier and Perchet(2019)]%
        {boursier2019sic}
\bibfield{author}{\bibinfo{person}{Etienne Boursier} {and}
  \bibinfo{person}{Vianney Perchet}.} \bibinfo{year}{2019}\natexlab{}.
\newblock \showarticletitle{SIC-MMAB: Synchronisation Involves Communication in
  Multiplayer Multi-Armed Bandits}.
\newblock \bibinfo{journal}{\emph{Advances in Neural Information Processing
  Systems}}  \bibinfo{volume}{32} (\bibinfo{year}{2019}),
  \bibinfo{pages}{12071--12080}.
\newblock


\bibitem[Bubeck et~al\mbox{.}(2011)]%
        {bubeck2011pure}
\bibfield{author}{\bibinfo{person}{S{\'e}bastien Bubeck},
  \bibinfo{person}{R{\'e}mi Munos}, {and} \bibinfo{person}{Gilles Stoltz}.}
  \bibinfo{year}{2011}\natexlab{}.
\newblock \showarticletitle{Pure exploration in finitely-armed and
  continuous-armed bandits}.
\newblock \bibinfo{journal}{\emph{Theoretical Computer Science}}
  \bibinfo{volume}{412}, \bibinfo{number}{19} (\bibinfo{year}{2011}),
  \bibinfo{pages}{1832--1852}.
\newblock


\bibitem[Buccapatnam et~al\mbox{.}(2015)]%
        {buccapatnam2015information}
\bibfield{author}{\bibinfo{person}{Swapna Buccapatnam}, \bibinfo{person}{Jian
  Tan}, {and} \bibinfo{person}{Li Zhang}.} \bibinfo{year}{2015}\natexlab{}.
\newblock \showarticletitle{Information sharing in distributed stochastic
  bandits}. In \bibinfo{booktitle}{\emph{2015 IEEE Conference on Computer
  Communications (INFOCOM)}}. IEEE, \bibinfo{pages}{2605--2613}.
\newblock


\bibitem[Cesa-Bianchi et~al\mbox{.}(2016)]%
        {cesa2016delay}
\bibfield{author}{\bibinfo{person}{Nicolo Cesa-Bianchi},
  \bibinfo{person}{Claudio Gentile}, \bibinfo{person}{Yishay Mansour}, {and}
  \bibinfo{person}{Alberto Minora}.} \bibinfo{year}{2016}\natexlab{}.
\newblock \showarticletitle{Delay and cooperation in nonstochastic bandits}. In
  \bibinfo{booktitle}{\emph{Conference on Learning Theory}},
  Vol.~\bibinfo{volume}{49}. \bibinfo{pages}{605--622}.
\newblock


\bibitem[Chakraborty et~al\mbox{.}(2017)]%
        {chakraborty2017coordinated}
\bibfield{author}{\bibinfo{person}{Mithun Chakraborty}, \bibinfo{person}{Kai
  Yee~Phoebe Chua}, \bibinfo{person}{Sanmay Das}, {and}
  \bibinfo{person}{Brendan Juba}.} \bibinfo{year}{2017}\natexlab{}.
\newblock \showarticletitle{Coordinated Versus Decentralized Exploration In
  Multi-Agent Multi-Armed Bandits.}. In \bibinfo{booktitle}{\emph{IJCAI}}.
  \bibinfo{pages}{164--170}.
\newblock


\bibitem[Chawla et~al\mbox{.}(2020)]%
        {chawla2020gossiping}
\bibfield{author}{\bibinfo{person}{Ronshee Chawla}, \bibinfo{person}{Abishek
  Sankararaman}, \bibinfo{person}{Ayalvadi Ganesh}, {and}
  \bibinfo{person}{Sanjay Shakkottai}.} \bibinfo{year}{2020}\natexlab{}.
\newblock \showarticletitle{The Gossiping Insert-Eliminate Algorithm for
  Multi-Agent Bandits}. In \bibinfo{booktitle}{\emph{Proceedings of the Twenty
  Third International Conference on Artificial Intelligence and Statistics}}.
  \bibinfo{pages}{3471--3481}.
\newblock


\bibitem[Chawla et~al\mbox{.}(2022)]%
        {chawla2020multi}
\bibfield{author}{\bibinfo{person}{Ronshee Chawla}, \bibinfo{person}{Abishek
  Sankararaman}, {and} \bibinfo{person}{Sanjay Shakkottai}.}
  \bibinfo{year}{2022}\natexlab{}.
\newblock \showarticletitle{Multi-agent low-dimensional linear bandits}.
\newblock \bibinfo{journal}{\emph{IEEE Trans. Automat. Control}}
  (\bibinfo{year}{2022}).
\newblock


\bibitem[Chierichetti et~al\mbox{.}(2010)]%
        {chierichetti2010almost}
\bibfield{author}{\bibinfo{person}{Flavio Chierichetti},
  \bibinfo{person}{Silvio Lattanzi}, {and} \bibinfo{person}{Alessandro
  Panconesi}.} \bibinfo{year}{2010}\natexlab{}.
\newblock \showarticletitle{Almost tight bounds for rumour spreading with
  conductance}. In \bibinfo{booktitle}{\emph{Proceedings of the forty-second
  ACM symposium on Theory of computing}}. \bibinfo{pages}{399--408}.
\newblock


\bibitem[Dakdouk et~al\mbox{.}(2021)]%
        {dakdouk2021collaborative}
\bibfield{author}{\bibinfo{person}{Hiba Dakdouk}, \bibinfo{person}{Rapha{\"e}l
  F{\'e}raud}, \bibinfo{person}{Romain Laroche}, \bibinfo{person}{Nad{\`e}ge
  Varsier}, {and} \bibinfo{person}{Patrick Maill{\'e}}.}
  \bibinfo{year}{2021}\natexlab{}.
\newblock \showarticletitle{Collaborative Exploration and Exploitation in
  massively Multi-Player Bandits}.
\newblock  (\bibinfo{year}{2021}).
\newblock


\bibitem[Dubey et~al\mbox{.}(2020a)]%
        {dubey2020cooperative}
\bibfield{author}{\bibinfo{person}{Abhimanyu Dubey} {et~al\mbox{.}}}
  \bibinfo{year}{2020}\natexlab{a}.
\newblock \showarticletitle{Cooperative multi-agent bandits with heavy tails}.
  In \bibinfo{booktitle}{\emph{International Conference on Machine Learning}}.
  PMLR, \bibinfo{pages}{2730--2739}.
\newblock


\bibitem[Dubey et~al\mbox{.}(2020b)]%
        {dubey2020kernel}
\bibfield{author}{\bibinfo{person}{Abhimanyu Dubey} {et~al\mbox{.}}}
  \bibinfo{year}{2020}\natexlab{b}.
\newblock \showarticletitle{Kernel methods for cooperative multi-agent
  contextual bandits}. In \bibinfo{booktitle}{\emph{International Conference on
  Machine Learning}}. PMLR, \bibinfo{pages}{2740--2750}.
\newblock


\bibitem[Dubey and Pentland(2020)]%
        {dubey2020differentially}
\bibfield{author}{\bibinfo{person}{Abhimanyu Dubey} {and}
  \bibinfo{person}{AlexSandy' Pentland}.} \bibinfo{year}{2020}\natexlab{}.
\newblock \showarticletitle{Differentially-Private Federated Linear Bandits}.
\newblock \bibinfo{journal}{\emph{Advances in Neural Information Processing
  Systems}}  \bibinfo{volume}{33} (\bibinfo{year}{2020}).
\newblock


\bibitem[Garcelon et~al\mbox{.}(2020)]%
        {garcelon2020adversarial}
\bibfield{author}{\bibinfo{person}{Evrard Garcelon}, \bibinfo{person}{Baptiste
  Roziere}, \bibinfo{person}{Laurent Meunier}, \bibinfo{person}{Jean
  Tarbouriech}, \bibinfo{person}{Olivier Teytaud}, \bibinfo{person}{Alessandro
  Lazaric}, {and} \bibinfo{person}{Matteo Pirotta}.}
  \bibinfo{year}{2020}\natexlab{}.
\newblock \showarticletitle{Adversarial Attacks on Linear Contextual Bandits}.
\newblock \bibinfo{journal}{\emph{Advances in Neural Information Processing
  Systems}}  \bibinfo{volume}{33} (\bibinfo{year}{2020}).
\newblock


\bibitem[Gupta et~al\mbox{.}(2019)]%
        {gupta2019better}
\bibfield{author}{\bibinfo{person}{Anupam Gupta}, \bibinfo{person}{Tomer
  Koren}, {and} \bibinfo{person}{Kunal Talwar}.}
  \bibinfo{year}{2019}\natexlab{}.
\newblock \showarticletitle{Better Algorithms for Stochastic Bandits with
  Adversarial Corruptions}. In \bibinfo{booktitle}{\emph{Conference on Learning
  Theory}}. \bibinfo{pages}{1562--1578}.
\newblock


\bibitem[Harper and Konstan(2015)]%
        {harper2015movielens}
\bibfield{author}{\bibinfo{person}{F~Maxwell Harper} {and}
  \bibinfo{person}{Joseph~A Konstan}.} \bibinfo{year}{2015}\natexlab{}.
\newblock \showarticletitle{The movielens datasets: History and context}.
\newblock \bibinfo{journal}{\emph{Acm transactions on interactive intelligent
  systems (tiis)}} \bibinfo{volume}{5}, \bibinfo{number}{4}
  (\bibinfo{year}{2015}), \bibinfo{pages}{1--19}.
\newblock


\bibitem[Hillel et~al\mbox{.}(2013)]%
        {hillel2013distributed}
\bibfield{author}{\bibinfo{person}{Eshcar Hillel}, \bibinfo{person}{Zohar~S
  Karnin}, \bibinfo{person}{Tomer Koren}, \bibinfo{person}{Ronny Lempel}, {and}
  \bibinfo{person}{Oren Somekh}.} \bibinfo{year}{2013}\natexlab{}.
\newblock \showarticletitle{Distributed exploration in multi-armed bandits}. In
  \bibinfo{booktitle}{\emph{Advances in Neural Information Processing
  Systems}}. \bibinfo{pages}{854--862}.
\newblock


\bibitem[Jun et~al\mbox{.}(2018)]%
        {jun2018adversarial}
\bibfield{author}{\bibinfo{person}{Kwang-Sung Jun}, \bibinfo{person}{Lihong
  Li}, \bibinfo{person}{Yuzhe Ma}, {and} \bibinfo{person}{Jerry Zhu}.}
  \bibinfo{year}{2018}\natexlab{}.
\newblock \showarticletitle{Adversarial attacks on stochastic bandits}.
\newblock \bibinfo{journal}{\emph{Advances in Neural Information Processing
  Systems}}  \bibinfo{volume}{31} (\bibinfo{year}{2018}),
  \bibinfo{pages}{3640--3649}.
\newblock


\bibitem[Kalathil et~al\mbox{.}(2014)]%
        {kalathil2014decentralized}
\bibfield{author}{\bibinfo{person}{Dileep Kalathil}, \bibinfo{person}{Naumaan
  Nayyar}, {and} \bibinfo{person}{Rahul Jain}.}
  \bibinfo{year}{2014}\natexlab{}.
\newblock \showarticletitle{Decentralized learning for multiplayer multiarmed
  bandits}.
\newblock \bibinfo{journal}{\emph{IEEE Transactions on Information Theory}}
  \bibinfo{volume}{60}, \bibinfo{number}{4} (\bibinfo{year}{2014}),
  \bibinfo{pages}{2331--2345}.
\newblock


\bibitem[Kanade et~al\mbox{.}(2012)]%
        {kanade2012distributed}
\bibfield{author}{\bibinfo{person}{Varun Kanade}, \bibinfo{person}{Zhenming
  Liu}, {and} \bibinfo{person}{Bozidar Radunovic}.}
  \bibinfo{year}{2012}\natexlab{}.
\newblock \showarticletitle{Distributed non-stochastic experts}. In
  \bibinfo{booktitle}{\emph{Advances in Neural Information Processing
  Systems}}. \bibinfo{pages}{260--268}.
\newblock


\bibitem[Kao et~al\mbox{.}(2022)]%
        {kao2021decentralized}
\bibfield{author}{\bibinfo{person}{Hsu Kao}, \bibinfo{person}{Chen-Yu Wei},
  {and} \bibinfo{person}{Vijay Subramanian}.} \bibinfo{year}{2022}\natexlab{}.
\newblock \showarticletitle{Decentralized cooperative reinforcement learning
  with hierarchical information structure}. In
  \bibinfo{booktitle}{\emph{International Conference on Algorithmic Learning
  Theory}}. PMLR, \bibinfo{pages}{573--605}.
\newblock


\bibitem[Kapoor et~al\mbox{.}(2019)]%
        {kapoor2019corruption}
\bibfield{author}{\bibinfo{person}{Sayash Kapoor},
  \bibinfo{person}{Kumar~Kshitij Patel}, {and} \bibinfo{person}{Purushottam
  Kar}.} \bibinfo{year}{2019}\natexlab{}.
\newblock \showarticletitle{Corruption-tolerant bandit learning}.
\newblock \bibinfo{journal}{\emph{Machine Learning}} \bibinfo{volume}{108},
  \bibinfo{number}{4} (\bibinfo{year}{2019}), \bibinfo{pages}{687--715}.
\newblock


\bibitem[Kolla et~al\mbox{.}(2018)]%
        {kolla2018collaborative}
\bibfield{author}{\bibinfo{person}{Ravi~Kumar Kolla}, \bibinfo{person}{Krishna
  Jagannathan}, {and} \bibinfo{person}{Aditya Gopalan}.}
  \bibinfo{year}{2018}\natexlab{}.
\newblock \showarticletitle{Collaborative learning of stochastic bandits over a
  social network}.
\newblock \bibinfo{journal}{\emph{IEEE/ACM Transactions on Networking}}
  \bibinfo{volume}{26}, \bibinfo{number}{4} (\bibinfo{year}{2018}),
  \bibinfo{pages}{1782--1795}.
\newblock


\bibitem[Korda et~al\mbox{.}(2016)]%
        {korda2016distributed}
\bibfield{author}{\bibinfo{person}{Nathan Korda}, \bibinfo{person}{Bal{\'a}zs
  Sz{\"o}r{\'e}nyi}, {and} \bibinfo{person}{Li Shuai}.}
  \bibinfo{year}{2016}\natexlab{}.
\newblock \showarticletitle{Distributed clustering of linear bandits in peer to
  peer networks}. In \bibinfo{booktitle}{\emph{Journal of machine learning
  research workshop and conference proceedings}}, Vol.~\bibinfo{volume}{48}.
  International Machine Learning Societ, \bibinfo{pages}{1301--1309}.
\newblock


\bibitem[Lalitha and Goldsmith(2021)]%
        {lalitha2021bayesian}
\bibfield{author}{\bibinfo{person}{Anusha Lalitha} {and}
  \bibinfo{person}{Andrea Goldsmith}.} \bibinfo{year}{2021}\natexlab{}.
\newblock \showarticletitle{Bayesian Algorithms for Decentralized Stochastic
  Bandits}.
\newblock \bibinfo{journal}{\emph{IEEE Journal on Selected Areas in Information
  Theory}} \bibinfo{volume}{2}, \bibinfo{number}{2} (\bibinfo{year}{2021}),
  \bibinfo{pages}{564--583}.
\newblock


\bibitem[Landgren et~al\mbox{.}(2016)]%
        {landgren2016distributed}
\bibfield{author}{\bibinfo{person}{Peter Landgren}, \bibinfo{person}{Vaibhav
  Srivastava}, {and} \bibinfo{person}{Naomi~Ehrich Leonard}.}
  \bibinfo{year}{2016}\natexlab{}.
\newblock \showarticletitle{Distributed cooperative decision-making in
  multiarmed bandits: Frequentist and Bayesian algorithms}. In
  \bibinfo{booktitle}{\emph{2016 IEEE 55th Conference on Decision and Control
  (CDC)}}. IEEE, \bibinfo{pages}{167--172}.
\newblock


\bibitem[Lattimore and Szepesv{\'a}ri(2020)]%
        {lattimore2020bandit}
\bibfield{author}{\bibinfo{person}{Tor Lattimore} {and} \bibinfo{person}{Csaba
  Szepesv{\'a}ri}.} \bibinfo{year}{2020}\natexlab{}.
\newblock \bibinfo{booktitle}{\emph{Bandit algorithms}}.
\newblock \bibinfo{publisher}{Cambridge University Press}.
\newblock


\bibitem[LeBlanc et~al\mbox{.}(2013)]%
        {leblanc2013resilient}
\bibfield{author}{\bibinfo{person}{Heath~J LeBlanc}, \bibinfo{person}{Haotian
  Zhang}, \bibinfo{person}{Xenofon Koutsoukos}, {and} \bibinfo{person}{Shreyas
  Sundaram}.} \bibinfo{year}{2013}\natexlab{}.
\newblock \showarticletitle{Resilient asymptotic consensus in robust networks}.
\newblock \bibinfo{journal}{\emph{IEEE Journal on Selected Areas in
  Communications}} \bibinfo{volume}{31}, \bibinfo{number}{4}
  (\bibinfo{year}{2013}), \bibinfo{pages}{766--781}.
\newblock


\bibitem[Liu and Shroff(2019)]%
        {liu2019data}
\bibfield{author}{\bibinfo{person}{Fang Liu} {and} \bibinfo{person}{Ness
  Shroff}.} \bibinfo{year}{2019}\natexlab{}.
\newblock \showarticletitle{Data Poisoning Attacks on Stochastic Bandits}. In
  \bibinfo{booktitle}{\emph{International Conference on Machine Learning}}.
  \bibinfo{pages}{4042--4050}.
\newblock


\bibitem[Liu et~al\mbox{.}(2021)]%
        {liu2021cooperative}
\bibfield{author}{\bibinfo{person}{Junyan Liu}, \bibinfo{person}{Shuai Li},
  {and} \bibinfo{person}{Dapeng Li}.} \bibinfo{year}{2021}\natexlab{}.
\newblock \showarticletitle{Cooperative Stochastic Multi-agent Multi-armed
  Bandits Robust to Adversarial Corruptions}.
\newblock \bibinfo{journal}{\emph{arXiv preprint arXiv:2106.04207}}
  (\bibinfo{year}{2021}).
\newblock


\bibitem[Liu and Zhao(2010)]%
        {liu2010distributed}
\bibfield{author}{\bibinfo{person}{Keqin Liu} {and} \bibinfo{person}{Qing
  Zhao}.} \bibinfo{year}{2010}\natexlab{}.
\newblock \showarticletitle{Distributed learning in multi-armed bandit with
  multiple players}.
\newblock \bibinfo{journal}{\emph{IEEE Transactions on Signal Processing}}
  \bibinfo{volume}{58}, \bibinfo{number}{11} (\bibinfo{year}{2010}),
  \bibinfo{pages}{5667--5681}.
\newblock


\bibitem[Liu et~al\mbox{.}(2020)]%
        {liu2020competing}
\bibfield{author}{\bibinfo{person}{Lydia~T Liu}, \bibinfo{person}{Horia Mania},
  {and} \bibinfo{person}{Michael Jordan}.} \bibinfo{year}{2020}\natexlab{}.
\newblock \showarticletitle{Competing bandits in matching markets}. In
  \bibinfo{booktitle}{\emph{International Conference on Artificial Intelligence
  and Statistics}}. PMLR, \bibinfo{pages}{1618--1628}.
\newblock


\bibitem[Lykouris et~al\mbox{.}(2018)]%
        {lykouris2018stochastic}
\bibfield{author}{\bibinfo{person}{Thodoris Lykouris}, \bibinfo{person}{Vahab
  Mirrokni}, {and} \bibinfo{person}{Renato Paes~Leme}.}
  \bibinfo{year}{2018}\natexlab{}.
\newblock \showarticletitle{Stochastic bandits robust to adversarial
  corruptions}. In \bibinfo{booktitle}{\emph{Proceedings of the 50th Annual ACM
  SIGACT Symposium on Theory of Computing}}. \bibinfo{pages}{114--122}.
\newblock


\bibitem[Madhushani and Leonard(2021)]%
        {madhushani2021call}
\bibfield{author}{\bibinfo{person}{Udari Madhushani} {and}
  \bibinfo{person}{Naomi Leonard}.} \bibinfo{year}{2021}\natexlab{}.
\newblock \showarticletitle{When to call your neighbor? strategic communication
  in cooperative stochastic bandits}.
\newblock \bibinfo{journal}{\emph{arXiv preprint arXiv:2110.04396}}
  (\bibinfo{year}{2021}).
\newblock


\bibitem[Mansour et~al\mbox{.}(2018)]%
        {mansour2018competing}
\bibfield{author}{\bibinfo{person}{Yishay Mansour}, \bibinfo{person}{Aleksandrs
  Slivkins}, {and} \bibinfo{person}{Steven Wu}.}
  \bibinfo{year}{2018}\natexlab{}.
\newblock \showarticletitle{Competing bandits: Learning under competition}. In
  \bibinfo{booktitle}{\emph{9th Innovations in Theoretical Computer Science,
  ITCS 2018}}. Schloss Dagstuhl-Leibniz-Zentrum fur Informatik GmbH, Dagstuhl
  Publishing, \bibinfo{pages}{48}.
\newblock


\bibitem[Mart{\'\i}nez-Rubio et~al\mbox{.}(2019)]%
        {martinez2019decentralized}
\bibfield{author}{\bibinfo{person}{David Mart{\'\i}nez-Rubio},
  \bibinfo{person}{Varun Kanade}, {and} \bibinfo{person}{Patrick Rebeschini}.}
  \bibinfo{year}{2019}\natexlab{}.
\newblock \showarticletitle{Decentralized cooperative stochastic multi-armed
  bandits}.
\newblock \bibinfo{journal}{\emph{Advances in Neural Information Processing
  Systems}} (\bibinfo{year}{2019}).
\newblock


\bibitem[Mitra et~al\mbox{.}(2021)]%
        {mitra2021exploiting}
\bibfield{author}{\bibinfo{person}{Aritra Mitra}, \bibinfo{person}{Hamed
  Hassani}, {and} \bibinfo{person}{George Pappas}.}
  \bibinfo{year}{2021}\natexlab{}.
\newblock \showarticletitle{Exploiting Heterogeneity in Robust Federated
  Best-Arm Identification}.
\newblock \bibinfo{journal}{\emph{arXiv preprint arXiv:2109.05700}}
  (\bibinfo{year}{2021}).
\newblock


\bibitem[Newton et~al\mbox{.}(2021)]%
        {newton2021asymptotic}
\bibfield{author}{\bibinfo{person}{Conor Newton}, \bibinfo{person}{AJ Ganesh},
  {and} \bibinfo{person}{Henry Reeve}.} \bibinfo{year}{2021}\natexlab{}.
\newblock \showarticletitle{Asymptotic Optimality for Decentralised Bandits}.
  In \bibinfo{booktitle}{\emph{Reinforcement Learning in Networks and Queues,
  Sigmetrics 2021}}.
\newblock


\bibitem[Pittel(1987)]%
        {pittel1987spreading}
\bibfield{author}{\bibinfo{person}{Boris Pittel}.}
  \bibinfo{year}{1987}\natexlab{}.
\newblock \showarticletitle{On spreading a rumor}.
\newblock \bibinfo{journal}{\emph{SIAM J. Appl. Math.}} \bibinfo{volume}{47},
  \bibinfo{number}{1} (\bibinfo{year}{1987}), \bibinfo{pages}{213--223}.
\newblock


\bibitem[Rosenski et~al\mbox{.}(2016)]%
        {rosenski2016multi}
\bibfield{author}{\bibinfo{person}{Jonathan Rosenski}, \bibinfo{person}{Ohad
  Shamir}, {and} \bibinfo{person}{Liran Szlak}.}
  \bibinfo{year}{2016}\natexlab{}.
\newblock \showarticletitle{Multi-player bandits--a musical chairs approach}.
  In \bibinfo{booktitle}{\emph{International Conference on Machine Learning}}.
  \bibinfo{pages}{155--163}.
\newblock


\bibitem[Sankararaman et~al\mbox{.}(2019)]%
        {sankararaman2019social}
\bibfield{author}{\bibinfo{person}{Abishek Sankararaman},
  \bibinfo{person}{Ayalvadi Ganesh}, {and} \bibinfo{person}{Sanjay
  Shakkottai}.} \bibinfo{year}{2019}\natexlab{}.
\newblock \showarticletitle{Social learning in multi agent multi armed
  bandits}.
\newblock \bibinfo{journal}{\emph{Proceedings of the ACM on Measurement and
  Analysis of Computing Systems}} \bibinfo{volume}{3}, \bibinfo{number}{3}
  (\bibinfo{year}{2019}), \bibinfo{pages}{1--35}.
\newblock


\bibitem[Shahrampour et~al\mbox{.}(2017)]%
        {shahrampour2017multi}
\bibfield{author}{\bibinfo{person}{Shahin Shahrampour},
  \bibinfo{person}{Alexander Rakhlin}, {and} \bibinfo{person}{Ali Jadbabaie}.}
  \bibinfo{year}{2017}\natexlab{}.
\newblock \showarticletitle{Multi-armed bandits in multi-agent networks}. In
  \bibinfo{booktitle}{\emph{2017 IEEE International Conference on Acoustics,
  Speech and Signal Processing (ICASSP)}}. IEEE, \bibinfo{pages}{2786--2790}.
\newblock


\bibitem[Sz{\"o}r{\'e}nyi et~al\mbox{.}(2013)]%
        {szorenyi2013gossip}
\bibfield{author}{\bibinfo{person}{Bal{\'a}zs Sz{\"o}r{\'e}nyi},
  \bibinfo{person}{R{\'o}bert Busa-Fekete}, \bibinfo{person}{Istv{\'a}n
  Heged{\H{u}}s}, \bibinfo{person}{R{\'o}bert Orm{\'a}ndi},
  \bibinfo{person}{M{\'a}rk Jelasity}, {and} \bibinfo{person}{Bal{\'a}zs
  K{\'e}gl}.} \bibinfo{year}{2013}\natexlab{}.
\newblock \showarticletitle{Gossip-based distributed stochastic bandit
  algorithms}. In \bibinfo{booktitle}{\emph{Journal of Machine Learning
  Research Workshop and Conference Proceedings}}, Vol.~\bibinfo{volume}{2}.
  International Machine Learning Societ, \bibinfo{pages}{1056--1064}.
\newblock


\bibitem[Tekin and Van Der~Schaar(2015)]%
        {tekin2015distributed}
\bibfield{author}{\bibinfo{person}{Cem Tekin} {and} \bibinfo{person}{Mihaela
  Van Der~Schaar}.} \bibinfo{year}{2015}\natexlab{}.
\newblock \showarticletitle{Distributed online learning via cooperative
  contextual bandits}.
\newblock \bibinfo{journal}{\emph{IEEE transactions on signal processing}}
  \bibinfo{volume}{63}, \bibinfo{number}{14} (\bibinfo{year}{2015}),
  \bibinfo{pages}{3700--3714}.
\newblock


\bibitem[Vial et~al\mbox{.}(2021)]%
        {vial2021robust}
\bibfield{author}{\bibinfo{person}{Daniel Vial}, \bibinfo{person}{Sanjay
  Shakkottai}, {and} \bibinfo{person}{R Srikant}.}
  \bibinfo{year}{2021}\natexlab{}.
\newblock \showarticletitle{Robust multi-agent multi-armed bandits}. In
  \bibinfo{booktitle}{\emph{Proceedings of the Twenty-second International
  Symposium on Theory, Algorithmic Foundations, and Protocol Design for Mobile
  Networks and Mobile Computing}}. \bibinfo{pages}{161--170}.
\newblock


\bibitem[Wang et~al\mbox{.}(2020)]%
        {wang2020optimal}
\bibfield{author}{\bibinfo{person}{Po-An Wang}, \bibinfo{person}{Alexandre
  Proutiere}, \bibinfo{person}{Kaito Ariu}, \bibinfo{person}{Yassir Jedra},
  {and} \bibinfo{person}{Alessio Russo}.} \bibinfo{year}{2020}\natexlab{}.
\newblock \showarticletitle{Optimal algorithms for multiplayer multi-armed
  bandits}. In \bibinfo{booktitle}{\emph{International Conference on Artificial
  Intelligence and Statistics}}. PMLR, \bibinfo{pages}{4120--4129}.
\newblock


\bibitem[Zhu et~al\mbox{.}(2021)]%
        {zhu2021federated}
\bibfield{author}{\bibinfo{person}{Zhaowei Zhu}, \bibinfo{person}{Jingxuan
  Zhu}, \bibinfo{person}{Ji Liu}, {and} \bibinfo{person}{Yang Liu}.}
  \bibinfo{year}{2021}\natexlab{}.
\newblock \showarticletitle{Federated bandit: A gossiping approach}. In
  \bibinfo{booktitle}{\emph{Abstract Proceedings of the 2021 ACM
  SIGMETRICS/International Conference on Measurement and Modeling of Computer
  Systems}}. \bibinfo{pages}{3--4}.
\newblock


\end{thebibliography}
